\documentclass[twoside,11pt]{article}

\usepackage{amsmath,amsfonts,amssymb}
\usepackage{listings} 
\usepackage{jmlr2e}
\usepackage{pslatex}
\usepackage[small]{caption}

\usepackage{amsthm}

\usepackage[disable]{todonotes}

\usepackage{color} 
\usepackage{caption}
\usepackage{subcaption}
\usepackage{listings}
\usepackage{enumerate}
\usepackage{natbib}
\usepackage{graphicx} 
\usepackage{clrscode3e}

\chardef\bslash=`\\ 

\newtheorem{thm}{Theorem}[section]

\newtheorem{claim}[thm]{Claim}
\newtheorem{prop}[thm]{Proposition}

\theoremstyle{definition}
\newtheorem{defn}{Definition}[section]

\theoremstyle{remark}

\newcommand{\addr}{\mathbf{addr}}

\newcommand{\args}{\mathbf{args}}

\newcommand{\sample}{\mathbf{sample}}
\newcommand{\unsample}{\mathbf{unsample}}

\newcommand{\expr}{\mathbf{expr}}

\newcommand{\env}{\mathbf{env}}

\newcommand{\mem}{\mathbf{mem}}

\newcommand{\true}{\mathbf{true}}

\newcommand{\DG}{\mathbf{D}}
\newcommand{\AG}{\mathbf{A}}

\newcommand{\IG}{\mathbf{I}}
\newcommand{\PG}{\mathbf{P}}
\newcommand{\SG}{\mathbf{S}}
\newcommand{\OG}{\mathbf{O}}
\newcommand{\RG}{\mathbf{R}}
\newcommand{\Scope}{\mathbf{Scope}}
\newcommand{\Envelope}{\mathbf{Envelope}}

\newcommand{\out}{\mathbf{output}}

\newcommand{\torus}{\mathbf{torus}}

\newcommand{\request}{\mathbf{request}}
\newcommand{\on}{\mathbf{on}}
\newcommand{\llia}{\mathbf{requests}}
\newcommand{\dash}{\mbox{-}}
\newcommand{\prog}{\mathbf{P}}

\newcommand{\nds}{\mathbf{nds}}
\newcommand{\pns}{\mathbf{principals}}

\newcommand{\MH}{\mathbf{MH}}
\newcommand{\Q}{\mathbf{Q}}
\newcommand{\K}{\mathbf{K}}
\newcommand{\Mixmh}{\mathbf{MixMH}}
\newcommand{\MHn}{\mathbf{MH}_n}
\newcommand{\PS}{\Upsilon}
\newcommand{\unique}{\mathbf{unique}}

\newcommand{\MB}{\mathbf{MB}}

\newcommand{\brush}{\mathbf{brush}}

\newcommand{\regen}{\mathbf{regen}}
\newcommand{\scaffold}{\mathbf{scaffold}}

\newcommand{\detach}{\mathbf{detach}}
\newcommand{\regenAndAttach}{\mathbf{regenAndAttach}}

\newcommand{\detachAndRegen}{\mathbf{detachAndRegen}}

\begin{document} 

\title{Venture: a higher-order probabilistic
  programming platform with programmable inference}

\author{\name Vikash Mansinghka \email vkm@mit.edu \\
\name Daniel Selsam \email dselsam@mit.edu \\
\name Yura Perov \email perov@mit.edu}

\todo[color=green]{[small] Reinsert editor for JMLR}
 
\maketitle 
 
\begin{abstract}We describe Venture, an interactive virtual machine for probabilistic programming that aims to be sufficiently expressive, extensible, and efficient for general-purpose use. Like Church, probabilistic models and inference problems in Venture are specified via a Turing-complete, higher-order probabilistic language descended from Lisp. Unlike Church, Venture also provides a compositional language for custom inference strategies, assembled from scalable implementations of several exact and approximate techniques. Venture is thus applicable to problems involving widely varying model families, dataset sizes and runtime/accuracy constraints. We also describe four key aspects of Venture's implementation that build on ideas from probabilistic graphical models. First, we describe the {\em stochastic procedure interface} (SPI) that specifies and encapsulates primitive random variables, analogously to conditional probability tables in a Bayesian network. The SPI supports custom control flow, higher-order probabilistic procedures, partially exchangeable sequences and ``likelihood-free'' stochastic simulators, all with custom proposals. It also supports the integration of external models that dynamically create, destroy and perform inference over latent variables hidden from Venture. Second, we describe {\em probabilistic execution traces} (PETs), which represent execution histories of Venture programs. Like Bayesian networks, PETs capture conditional dependencies, but PETs also represent existential dependencies and exchangeable coupling. Third, we describe partitions of execution histories called {\em scaffolds} that can be efficiently constructed from PETs and that factor global inference problems into coherent sub-problems. Finally, we describe a family of {\em stochastic regeneration algorithms} for efficiently modifying PET fragments contained within scaffolds without visiting conditionally independent random choices. Stochastic regeneration insulates inference algorithms from the complexities introduced by changes in execution structure, with runtime that scales linearly in cases where previous approaches often scaled quadratically and were therefore impractical. We show how to use stochastic regeneration and the SPI to implement general-purpose inference strategies such as Metropolis-Hastings, Gibbs sampling, and blocked proposals based on hybrids with both particle Markov chain Monte Carlo and mean-field variational inference techniques.
\todo[color=green]{[big] add real expts and summarize alg\&model, including results on datasets with millions of observations}
\end{abstract}


\thanks{{\bf Acknowledgements:} The authors thank Vlad Firoiu and Alexey Radul for contributions to multiple Venture implementations, and Daniel Roy, Cameron Freer and Alexey Radul for helpful discussions, suggestions, and comments on early drafts. This work was supported by the DARPA PPAML
  program, grants from the ONR and ARO, and Google's ``Rethinking AI''
  project. Any opinions, findings, and conclusions or recommendations
  expressed in this work are those of the authors and do not
  \todo[color=red]{[small] final grant/contract info}
  necessarily reflect the views of any of the above sponsors.}


\tableofcontents 
\listoftodos

\section{Introduction}

Probabilistic modeling and approximate Bayesian inference have proven
to be powerful tools in multiple fields, from machine learning
\citep{Bishop06} and statistics \citep{green2003highly, Gelman1995} to robotics
\citep{thrun2005probabilistic}, artificial intelligence
\citep{Russell2002}, and cognitive science \citep{tenenbaum2011grow}. Unfortunately, even relatively simple probabilistic models and their associated inference
schemes can be difficult and time-consuming to design, specify,
analyze, implement, and debug. Applications in different fields, such as robotics and statistics,
involve differing modeling idioms, inference techniques and dataset
sizes. Different fields also often impose varying speed and accuracy requirements that
interact with modeling and algorithmic choices. Small changes in the
modeling assumptions, data, performance/accuracy requirements or
compute budget frequently necessitate end-to-end redesign of the
probabilistic model and inference strategy, in turn necessitating
reimplementation of the underlying software. 

These difficulties impose a high cost on practitioners, making state-of-the-art modeling and inference approaches impractical for many problems. Minor variations on standard templates can be out of reach for non-specialists. The cost is also high for experts: the development time and failure rate make it difficult to innovate on methodology except in simple settings. This limits the richness of probabilistic models of cognition and artificial intelligence systems, as these kinds of models push the boundaries of what is possible with current knowledge representation and inference techniques.

Probabilistic programming languages could potentially mitigate these problems. They provide a formal representation for models --- often via executable code that makes a random choice for
every latent variable --- and attempt to encapsulate and automate
inference. Several languages and systems have been built along these
lines over the last decade \citep{Lunn2000,
  stan-manual:2013, milch20071, Pfeffer01ibal:a,
  mccallum09:factorie}. Each of these systems is promising in its own domain; some of the strengths of each are described below. However, none of the probabilistic
programming languages and systems that have been developed thus far
is suitable for general purpose use. Examples of drawbacks
include inadequate and unpredictable runtime performance, limited
expressiveness, batch-only operation, lack of extensibility, and
overly restrictive and/or opaque inference schemes. In this paper, we describe Venture, a
new probabilistic language and inference engine that attempts to address these limitations.

Several probabilistic programming tools have sought efficiency by
restricting expressiveness. For example, Microsoft's Infer.NET system
\citep{minka2010infer} leverages fast message passing techniques
originally developed for graphical models, but as a result restricts
the use of stochastic choice in the language so that it cannot
influence control flow. Such random choices would yield models over
sets of random variables with varying or even unbounded size, and
therefore preclude compilation to a graphical model. BUGS, arguably
the first (and still most widely used) probabilistic programming
language, has essentially the same restrictions \citep{Lunn2000}. Random compound data types, procedures and stochastic control flow constructs that could lead to a priori unbounded executions are all out of scope. STAN, a BUGS-like
language being developed in the Bayesian statistics community, has
limited support for discrete random variables, as these are
incompatible with the hybrid (gradient-based) Monte Carlo strategy it
uses to overcome convergence issues with Gibbs sampling \citep{stan-manual:2013}.
Other probabilistic programming tools that have seen real-world use include FACTORIE \citep{mccallum09:factorie} and Markov
Logic \citep{richardson2006markov}; applications of both have emphasized problems in information
extraction. The probabilistic models that can be defined using
FACTORIE and Markov Logic are finite and undirected, specified
imperatively (for FACTORIE) or declaratively (for Markov Logic). Both
systems make use of specialized, efficient approximation algorithms
for inference and parameter estimation. Infer.NET, STAN, BUGS,
FACTORIE and Markov Logic each capture important modeling and
approximate inference idioms, but there are also interesting models
that each cannot express.  Additionally, a number of probabilistic extensions of classical logic programming languages have also been developed \citep{poole1997independent, sato1997prism, de2008probabilistic}, motivated by problems in statistical relational learning. As with FACTORIE and Markov Logic, these languages have interesting and useful properties, but have thus far not yielded compact descriptions of many useful classes of probabilistic generative models from domains such as statistics and robotics.

In contrast, probabilistic programming languages such as BLOG \citep{milch20071}, IBAL \citep{Pfeffer01ibal:a}, Figaro \citep{pfeffer2009figaro} and Church \citep*{Goodman:2008tb, Mansinghka:2009}
emphasize expressiveness. Each language was designed around the needs of models
whose fine-grained structure cannot be represented using directed or
undirected graphical models, and where standard inference algorithms
for graphical models such as belief propagation do not directly
apply. Examples include probabilistic grammars \citep{jelinek1992basic}, nonparametric Bayesian
models \citep{rasmussen1999infinite, johnson2007adaptor, rasmussen2006gaussian, griffiths2005infinite}, probabilistic models over worlds with a priori unknown numbers
of objects, models for learning the structure of graphical models \citep{heckerman1998tutorial, Friedman2003, mansinghka2006structured}, models for inductive learning of symbolic expressions \citep{grosse2012exploiting, duvenaud2013structure} and models defined in terms of complex
stochastic simulation software that lack tractable likelihoods \citep{marjoram2003markov}. 

Each of these model classes is the basis of real-world applications, where inference over richly structured  models can address limitations of classic statistical modeling and pattern recognition techniques. Example domains include natural language processing \citep{manning1999foundations}, speech recognition \citep{baker2005trainable}, information extraction \citep{pasula2002identity},
multitarget tracking and sensor fusion \citep{oh2009markov, arora2010global}, ecology \citep{csillery2010approximate} and computational biology \citep{friedman2000using, yu2004advances, toni2009approximate, dowell2004evaluation}. However, the performance engineering needed to turn specialized inference algorithms for these models into viable implementations is challenging. Direct deployment of probabilistic program implementations in real-world applications is often infeasible. The elaborations on these models that expressive probabilistic languages enable can thus seem completely impractical.

Church makes the most extreme tradeoffs with respect to expressiveness and efficiency. It can represent models from all the
classes listed above, partly through its support for higher-order
probabilistic procedures, and it can also represent generative models
defined in terms of algorithms for simulation and inference in
arbitrary Church programs. This flexibility makes Church especially
suitable for nonparametric Bayesian modeling \citep{roy2008stochastic}, as well as artificial
intelligence and cognitive science problems that involve reasoning
about reasoning, such as sequential decision making, planning and
theory of mind \citep{probmods, Mansinghka:2009}. Additionally, probabilistic formulations of learning
Church programs from data --- including both program structure and
parameters --- can be formulated in terms of inference in an ordinary
Church program \citep{Mansinghka:2009}. But although various Church implementations provide automatic Metropolis-Hastings inference
mechanisms that in principle apply to all these problems, these
mechanisms have exhibited limitations in practice. It has not been
clear how to make general-purpose sampling-based inference in Church
sufficiently scalable for typical machine learning applications,
including problems for which standard techniques based on graphical
models have been applied successfully. It also is not easy for Church
programmers to override built in inference mechanisms or add
new higher-order stochastic primitives.

In this paper we describe Venture, an interactive, Turing-complete, higher order probabilistic programming platform that aims to be sufficiently expressive, extensible and efficient for general-purpose use. Venture includes a virtual machine, a language for specifying probabilistic models, and a language for specifying inference problems along with custom inference strategies for solving those problems. Venture's implementation of standard MCMC schemes scales linearly with dataset size on problems where many previous inference architectures scale quadratically and are therefore impractical. Venture also supports a larger class of primitives --- including ``likelihood-free'' primitives arising from complex stochastic simulators --- and enables programmers to incrementally migrate performance-critical portions of their probabilistic program to optimized external inference code. Venture thus improves over Church in terms of expressiveness, extensibility, and scalability. Although it remains to be seen if these improvements are sufficient for general-purpose use, unoptimized Venture prototypes have begun to be successfully applied in real-world system building,  Bayesian data analysis, cognitive modeling and machine intelligence research.

\todo[color=green]{[medium] Figure outlining key characteristics of applied modeling + inference across fields, with discussion of general purpose and robotics-statistics spectrum}

\todo[color=green]{[small] Provide short overview blurb summarizing performance measurements and bump up to key contributions, here and below}

\subsection{Contributions}

This paper makes two main contributions. First, it describes key aspects of Venture's design, including support for interactive modeling and programmable inference. Due to these and other innovations, Venture provides broad coverage in terms of models, approximation strategies for those models and overall applicability to inference problems with varying model/data complexities and time/accuracy requirements. Second, this paper describes key aspects of Venture's implementation: the {\em stochastic
  procedure interface} (SPI) for encapsulating primitives, the {\em probabilistic execution
  trace} (PET) data structure for efficient representation and updating of execution histories, and a suite of {\em stochastic regeneration algorithms} for scalable inference within trace fragments called {\em scaffolds}. Other important aspects of
Venture, including the VentureScript front-end syntax, formal language definitions, software architecture, standard library, and performance measurements of optimized implementations, are all beyond the scope of this paper.

\todo[color=green]{[medium] Figure outlining Venture's architecture and main ideas}

It is helpful to consider the relationships between PETs and the SPI.
The SPI generalizes the notion of elementary random procedure associated with many previous probabilistic programming languages. The SPI encapsulates Venture primitives
and enables interoperation with external modeling components, analogously to a foreign function interface in a traditional programming language. External modeling components can represent sets of latent variables hidden from Venture and that use specialized inference code. The SPI also supports custom control flow, higher-order probabilistic procedures, exchangeable
sequences, and the ``likelihood-free'' stochastic primitives that can
arise from complex simulations. Probabilistic execution traces are used to represent generative models written in Venture, along with particular realizations of these models and the data values they must explain. PETs generalize Bayesian networks to handle the patterns of conditional dependence, existential dependence and exchangeable coupling amongst invocations of stochastic procedures conforming to
the SPI. PETs thus must handle a priori unbounded sets of random
variables that can themselves be arbitrary probabilistic generative
processes written in Venture and that may lack tractable probability
densities.

Using these tools, we show how to define coherent local inference
steps over arbitrary sets of random choices within a PET, conditioned
on the surrounding trace. The core idea is the notion of a {\em
  scaffold}. A scaffold is a subset of a PET that contains those
random variables that must exist regardless of the values chosen for a
set of variables of interest, along with a set of random variables
whose values will be conditioned on. We show how to construct
scaffolds efficiently. Inference given a scaffold proceeds via {\em
  stochastic regeneration algorithms} that efficiently consume and
either restore or resample PET fragments without visiting
conditionally independent random choices. The proposal probabilities,
local priors, local likelihoods and gradients needed for several
approximate inference strategies can all be obtained via small
variations on stochastic regeneration.

We use stochastic regeneration to implement both single-site and composite Metropolis-Hastings, Gibbs sampling, and also blocked proposals based on hybrids with conditional sequential
Monte Carlo and variational techniques. The uniform implementation of
these approaches to incremental inference, along with analytical tools
for converting randomly chosen local transition operators into ergodic global transition operators on
PETs, constitutes another contribution.

\section{The Venture Language}

Consider the following example Venture program for determining if a coin is fair or tricky:

\begin{verbatim}
[ASSUME is_tricky_coin (bernoulli 0.1)]
[ASSUME coin_weight (if is_tricky_coin (uniform 0.0 1.0) 0.5)]
[OBSERVE (bernoulli coin_weight) True]
[OBSERVE (bernoulli coin_weight) True]
[INFER (mh default one 10)]
[PREDICT (bernoulli coin_weight)]
\end{verbatim}

We will informally discuss this program before defining the Venture language more precisely.

The \verb|ASSUME| instructions induce the hypothesis space
for the probabilistic model, including a random variable for whether
or not the coin is tricky, and either a deterministic coin weight or a
potential random variable corresponding to the unknown weight. The model selection problem is expressed via an \verb|if| with a stochastic predicate, with the alternative models on the consequent and alternate branches. After executing the \verb|ASSUME| instructions, particular values for \verb|is_tricky_coin| and \verb|coin_weight| will have been sampled, though the meaning of the program so far corresponds to a probability distribution over possible executions.

The \verb|OBSERVE| instructions describe a data generator that produces two flips of a coin with the generated weight, along with data that is assumed to be generated by the given generator. In this program, the \verb|OBSERVE|s encode constraints that both of these coin flips landed heads up. 

The \verb|INFER| instruction causes Venture to find a hypothesis
(execution trace) that is probable given the data using 10 iterations
of its default Markov chain for inference\footnote{This default Markov
  chain is a variant of the algorithm from Church
  \citep*{Goodman:2008tb, Mansinghka:2009}. This is a simple random
  scan single-site Metropolis-Hastings algorithm that chooses random
  choices uniformly at random from the current execution, resimulates
  them conditioned on the rest of the trace, and accepts or rejects
  the result.}. \verb|INFER| evolves the probability distribution over
execution traces inside Venture from whatever distribution is present
before the instruction --- in this case, the prior --- closer to its
conditional given any observations that have been added prior to that
\verb|INFER|, using a user-specified inference technique. In this
example program, the resulting marginal distribution on whether or not
the coin is tricky shifts from the prior to an approximation of the
posterior given the two observed flips, increasing the probability of
the coin being tricky ever so slightly. Increasing \verb|10| to
\verb|100| shifts the distribution closer to the true posterior; other
inference strategies, including exact sampling techniques, will be
covered later.  The execution trace inside Venture after the
instruction is sampled from this new distribution.

\todo[color=green]{[medium] tricky coin example: schematically illustrate samples vs distributions}

Once inference has finished, the \verb|PREDICT| instruction causes Venture to report a sampled prediction for the outcome of another flip of the coin. The weight used to generate this sample comes from the current execution trace.

\todo[color=green]{[medium] tricky coin example: figure for changing evidence as number of coin flips changes}

\subsection{Modeling and Inference Instructions}

Venture programs consist of sequences of modeling instructions and inference instructions, each given a positive integer index by a global instruction counter\footnote{Venture implementations also support labels for the instruction language. However, the recurrence of line numbers is reflective of the ways the current instruction language is primitive as compared to the modeling language. For example, it currently lacks control flow constructs, procedural abstraction, and recursion, let alone runtime generation and execution of statements in the instruction language. Current research is focused on addressing these limitations.}. Interactive Venture sessions have the same structure. The modeling instructions are used to specify the probabilistic model of interest, any conditions on the model that the inference engine needs to enforce, and any requests for prediction of values against the conditioned distribution. 

The core modeling instructions in Venture are:

\begin{enumerate}

\item {\tt [ASSUME <name> <expr>]}: binds the result of simulating the model expression {\tt <expr>} to the symbol {\tt <name>} in the global environment. This is used to build up the model that will be used to interpret data. Returns the value taken on by {\tt <name>}, along with the {\tt index} of the instruction.

\item {\tt [OBSERVE <expr> <literal-value>]}: adds the constraint that the model expression {\tt <expr>} must yield {\tt <literal-value>} in every execution. Note that this constraint is not enforced until inference is performed.

\item {\tt [PREDICT <expr>]}: samples a value for the model expression {\tt <expr>} from the current distribution on executions in the engine and returns the value. As the amount of inference done since the last {\tt OBSERVE} approaches infinity, this distribution converges to the conditioned distribution that reconciles the {\tt OBSERVE}s.

\item {\tt [FORGET <instruction-index-or-label>]}: This instruction causes the engine to undo and then forget the given instruction, which must be either an {\tt OBSERVE} or {\tt PREDICT}. Forgetting an observation removes the constraint it represents from the inference problem. Note that the effect may not be visible until an {\tt INFER} is performed.

\end{enumerate}

Venture supports additional instructions for inference and read-out, including:

\todo[color=green]{[code] Patch Venture to add a controllable computation time bound for rejection}

\begin{enumerate}

\item {\tt [INFER <inference-expr>]}: This instruction first incorporates any observations that have occurred after the last {\tt INFER}, then evolves the probability distribution on executions according to the inference strategy described by {\tt <inference-expr>}. Two inference expressions, corresponding to general-purpose exact and approximate sampling schemes, are useful to consider here:

\begin{enumerate}

\item {\tt (rejection default all)} corresponds to the use of rejection sampling to generate an exact sample from the conditioned distribution on traces. The runtime requirements may be substantial, and exact sampling applies to a smaller class of programs than approximate sampling. However, rejection is crucial for understanding the meaning of a probabilistic model and for debugging models without simultaneously debugging inference strategies.

\item {\tt (mh default one 1)} corresponds to one transition of the standard uniform mixture of single-site Metropolis-Hastings transition operators used as a general-purpose ``automatic'' inference scheme in many probabilistic programming systems. As the number of transitions is increased from $1$ towards $\infty$, the semantics of the instruction approach an exact implementation of conditioning via rejection.

\end{enumerate}

\item {\tt [SAMPLE <expr>]}: This instruction simulates the model expression {\tt <expr>} against the current trace, returns the value, and then forgets the trace associated with the simulation. It is equivalent to {\tt [PREDICT <expr>]} followed by {\tt [FORGET <index-of-predict>]}, but is provided as a single instruction for convenience.

\item {\tt [FORCE <expr> <literal-value>]}: Modify the current trace so that the simulation of {\tt <expr>} takes on the value {\tt <literal-value>}. Its implementation can be roughly thought of as an {\tt OBSERVE} immediately followed by an {\tt INFER} and then a {\tt FORGET}. This instruction can be used for controlling initialization and for debugging.

\end{enumerate}

\subsection{Modeling Expressions}

Venture modeling expressions describe stochastic generative processes. The space of all possible executions of all the modeling expressions in a Venture program constitute the hypothesis space that the program represents. Each Venture program thus represents a probabilistic model by defining a stochastic generative process that samples from it. 

At the expression level, Venture is similar to Scheme and to Church, though there are several differences\footnote{We sometimes refer to the s-expression syntax (including syntactic sugar) as Venchurch, and the desugared language (represented as JSON objects corresponding to parse trees) as Venture.}. For example, branching and procedure construction can both be desugared into applications of stochastic procedures --- that is, ordinary combinations --- and do not need to be treated as special forms. Additionally, Venture supports a dynamic scoping construct called {\tt scope\_include} for tagging portions of an execution history such that they can be referred to by inference instructions; to the best of our knowledge, analogous constructs have not yet been introduced in other probabilistic programming languages.

Venture modeling expressions can be broken down into a few simple cases:

\begin{enumerate}
\item {\bf Self-evaluating or ``literal'' values:} These describe constant values in the language, and are discussed below.

\item {\bf Combinations:} {\tt (<operator-expr> <operand0-expr> ... <operandk-expr>)} first evaluates all of its expressions in arbitrary order, then applies the value of {\tt <operator-expr>} (which must be a stochastic procedure) to the values of all the {\tt <operand-expr>}s. It returns the value of the application as its own result.

\item {\bf Quoted expressions:} {\tt (quote <expr>)} returns the expression value {\tt <expr>}. As compared to combinations, {\tt quote} suppresses evaluation.

\item {\bf Lambda expressions:} {\tt (lambda <args> <body-expr>)} returns a stochastic procedure with the given formal parameters and procedure body. {\tt <args>} is a specific list of argument names {\tt (<arg0> ... <argk>)}.

\item {\bf Conditionals:} {\tt (if <predicate-expr> <consequent-expr> <alternate-expr>)} evaluates the {\tt <predicate-expr>}, and then if the resulting value is {\tt true}, evaluates and returns value of the {\tt <consequent-expr>}, and if not, evaluates and returns the value of the {\tt <alternate-expr>}.

\item {\bf Inference scope annotations:} {\tt (scope\_include <scope-expr> <block-expr> <expr>)} provides a mechanism for naming random choices in a probabilistic model, so that they can be referred to during inference programming. The {\tt scope\_include} form simulates {\tt <scope-expr>} and {\tt <block-expr>} to obtain a scope value and a block value, and then simulates {\tt <expr>}, tagging all the random choices in that process with the given scope and block. More details on inference scopes are given below.

\end{enumerate}

\subsection{Inference Scopes}

Venture programs may attach metadata to fragments of execution traces via a dynamic scoping construct called {\em inference scopes}. Scopes are defined in modeling expressions, via the special form {\tt (scope\_include <scope-expr> <block-expr> <expr>)}, that assigns all random choices required to simulate {\tt <expr>} to a scope that is named by the value resulting from simulating {\tt <scope-expr>} and a {\em block} within that scope that is named by the value that results from simulating {\tt <block-expr>}. A single random choice can be in multiple inference scopes, but can only be in one block within each scope. Also, a random choice gets annotated with a scope each time the choice is simulated within the context of a {\tt scope\_include} form, not just the first time it is simulated.

Inference scopes can be referred to in inference expressions, thus providing a mechanism for associating custom inference strategies with different model fragments. For example, in a parameter estimation problem for hidden Markov models, it might be natural to have one scope for the hidden states, another for the hyperparameters, and a third for the parameters, where the blocks for the hidden state scope correspond to indexes into the hidden state sequence. We will see later how to write cycle hybrid kernels that use Metropolis-Hastings to make proposals for the hyperparameters and either single-site or particle Gibbs over the hidden states. Inference scopes also provide a means of controlling the allocation of computational effort across different aspects of a hypothesis, e.g. by only performing inference over scopes whose random choices are conditionally dependent on the choices made by a given {\tt PREDICT} instruction of interest.

Random choices are currently tagged with {\tt (scope, block)} pairs. Blocks can be thought of as subdivisions of scopes into meaningful (and potentially ordered) subsets. We will see later how inference expressions can make use of block structure to provide fine-grained control over inference and enable novel inference strategies. For example, the order in which a set of random choices is traversed by conditional sequential Monte Carlo can be controlled via blocks, regardless of the order in which they were constructed during initial simulation.

Scopes and blocks can be produced by random choices; {\tt <scope-expr>}s and {\tt <block-expr>}s are ordinary Venture modeling expressions\footnote{Our current implementations restrict the values of scope and block names to symbols and integers for simplicity, but this restriction is not intrinsic to the Venture specification.}. This enables the use of random choices in one scope to control the scope or block allocation of random choices in another scope. The random choices used to construct scopes and blocks may be auxiliary variables independent of the rest of the model, or latent variables whose distributions depend on the interaction of modeling assumptions, data and inference. At present, the only restriction is that inference on the random choices in a set of blocks cannot add or remove random choices from that set of blocks, though the probability of membership can be affected. Applications of randomized, inference-driven scope and block assignments include variants of cluster sampling techniques, beyond the spin glass \citep{swendsen1986replica} and regression \citep{nott2004bayesian} settings where they have typically been deployed.

\todo[color=green]{[medium] Add support for random scopes/blocks; cite stochastic digital circuits as motivation}
The implementation details needed to handle random scopes and blocks are beyond the scope of this paper. However, we will later see analytical machinery that is sufficient for justifying the correctness of complex transition operators involving randomly chosen scopes and blocks.

Venture provides two built-in scopes\footnote{Some implementations so far have merged the default and latent scopes, or triggered inference over latents automatically after every transition over the default scope.
}:

\begin{enumerate}

\item {\tt default} --- This scope contains every random choice. Previously proposed inference schemes for Church as well as concurrently developed generic inference schemes for variants of Venture correspond to single-line inference instructions acting on this default, global scope.

\item {\tt latents} --- This scope contains all the latent random choices made by stochastic procedures but hidden from Venture. Using this scope, programmers can control the frequency with which any external inference systems are invoked, and interleave inference over external variables with inference over the latent variables managed by Venture.

\end{enumerate}

\subsection{Inference Expressions}

{\em Inference expressions} specify transition operators that evolve the probability distribution on traces inside the Venture virtual machine. This is in contrast to {\em instructions}, which extend a model, add data, or initiate inference using a valid transition operator. 

Venture provides several primitive forms for constructing transition operators that leave the conditioned distribution invariant, each of which is a valid inference expression. In each of these forms, {\tt scope} must be a literal scope, and {\tt block} must either be a literal block within that scope, or the keyword {\tt one} or {\tt all}. The ``selected'' set of random choices on which each inference expression acts is given by the specified scope and block. If the block specification is {\tt all}, then the union of all blocks within the scope is taken. If the block specification is {\tt one}, then one block is chosen uniformly at random from the set of all blocks within the given scope.

The core set of inference expressions in Venture are as follows:

\begin{enumerate}

\item {\tt (mh <scope> <block> <\#-transitions>)} --- Propose new values for the selected choices either by resimulating them or by invoking a custom local proposal kernel if one has been provided. Accept or reject the results via the Metropolis-Hastings rule, accounting for changes to the mapping between random choices and scopes/blocks using the machinery provided later in this paper. Repeat the whole process {\tt \#-transitions} times.

\item {\tt (rejection <scope> <block> <\#-transitions>)} --- Use rejection sampling to generate an exact sample from the conditioned distribution on all the selected random choices. Repeat the whole process {\tt \#-transitions} times, potentially improving convergence if the selected set is randomly chosen, i.e. {\tt block} is {\tt one}. This transition operator is often computationally intractable, but is optimal, in the sense that it makes the most progress per completed transition towards the conditioned distribution on traces. All the other transition operators exposed by the Venture inference language can be viewed as asymptotically convergent approximations to it.

\item {\tt (pgibbs <scope> <block> <\#-particles> <\#-transitions>)} --- Use conditional sequential Monte Carlo to propose from an approximation to the conditioned distribution over the selected set of random choices. If {\tt block} is {\tt ordered}, all the blocks in the scope are sorted, and each distribution in the sequence of distributions includes all the random choices from the next block. Otherwise, each distribution in the sequence includes a single random choice drawn from the selected set, and the ordering is arbitrary.

\item {\tt (meanfield <scope> <block> <iters> <\#-transitions>)} --- Use {\tt iters} steps of stochastic gradient to optimize the parameters of a partial mean-field approximation to the conditioned distribution over the random choices in the given {\tt scope} and {\tt block} (with block interpreted as with {\tt mh}). Make a single Metropolis-Hastings proposal using this approximation. Repeat the process {\tt \#-transitions} times.

\item {\tt (enumerative\_gibbs <scope> <block> <\#-transitions>)} --- Use exhaustive enumeration to perform a transition over all the selected random choices from a proposal corresponding to the optimal conditional proposal (conditioned on the values of any newly created random variables). Random choices whose domains cannot be enumerated are resimulated from their prior unless they have been equipped with custom simulation kernels. If all selected random choices are discrete and no new random choices are created, this is equivalent to the {\tt rejection} transition operator, and corresponds to a discrete, enumerative implementation of Gibbs sampling, hence the name. The computational cost scales exponentially with the number of random choices, as opposed to the KL divergence between the prior and the conditional \citep{Anonymous:T2b7vHXI}.

\end{enumerate}

Both {\tt mh} and {\tt pgibbs} are implemented by in-place mutation. However, versions of each that use simultaneous particles to represent alternative possibilities are given by {\tt func-mh} and {\tt func-pgibbs}. These are prepended with {\tt func} to signal an aspect of their implementation: these simultaneously accessible sets of particles are implemented using persistent data structure techniques typically associated with pure functional programming. {\tt func-pgibbs} can yield improvements in order of growth of runtime as compared to {\tt pgibbs}, but it imposes restrictions on the selected random choices\footnote{To support multiple simultaneous particles, all stochastic procedures within the given {\tt scope} and {\tt block} must support a clone operation for their auxiliary state storage (or have the ability to emulate it). This is feasible for standard exponential family models, but may be not be feasible for external inference systems hosted on distributed hardware.}.

There are also currently two composition rules for transition operators, enabling the creation of cycle and mixture hybrids:

\begin{enumerate}

\item {\tt (cycle (<inference-expr-1> <inference-expr-2>) ... <\#-transitions>)} --- This produces a cycle hybrid of the transition operators represented by the given inference expression: each transition operator is run in sequence, and the whole sequence is repeated {\tt \#-transitions} times.

\item {\tt (mixture ((<w1> <inference-expr-1>) (<w2> <inference-expr-2>) ...) <\#-transitions>)} --- This produces a mixture hybrid of the given transition operators, using the given mixing weights, that is invoked {\tt \#-transitions} times.

\end{enumerate}

This language is flexible enough to express a broad class of standard approximate inference strategies as well as novel combinations of standard inference algorithm templates such as conditional sequential Monte Carlo, Metropolis-Hastings, Gibbs sampling and mean-field variational inference. Additionally, the ability to use random variables to map random choices to inference strategies and to perform inference over these variables may enable new cluster sampling techniques. That said, from an aesthetic standpoint, the current inference language also has many limitations, some of which seem straightforward to relax. For example, it seems natural to expand inference expressions to support arbitrary modeling expressions, and thereby also support arbitrary computation to produce inference schemes. The machinery needed to support these and other natural extensions is discussed later in this paper.

\subsection{Values}

Venture values include the usual scalar and symbolic data types from Scheme, along with extended support for collections and additional datatypes corresponding to primitive objects from probability theory and statistics. Venture also supports the {\em stochastic procedure} datatype, used for built-in and user-added primitive procedures as well as compound procedures returned by \verb|lambda|. A full treatment of the value hierarchy is out of scope, but we provide a brief list of the most important values here:

\begin{enumerate}
\item {\bf Numbers}: roughly analogous to floating point numbers, e.g. 1, 2.4, -23, and so on.

\item {\bf Atoms}: discrete items with no internal structure or ordering. These are generated by categorical draws, but also Dirichlet and Pitman-Yor processes.

\item {\bf Symbols}: symbol values, such as the name of a formal argument being passed to lambda, the name associated with an {\tt ASSUME} instruction, or the result of evaluating a {\tt quote} special form.

\item {\bf Collections:} vectors, which map numbers to values and support O(1) random access, and maps, which map values to values and support O(1) amortized random access (via a hash table that relies on the built-in hash function associated with each kind of value).

\item {\bf Stochastic procedures:} these include the components of the standard library, and can also be created by \verb|lambda| and other stochastic procedures.

\end{enumerate}

\subsection{Automatic inference versus inference programming}

Although Venture programs can incorporate custom inference strategies, it does not {\em require} them. Interfaces that are as automatic as existing probabilistic programming systems are straightforward to implement. Single-site Metropolis-Hastings and Gibbs sampling algorithms --- the sole automatic inference option in many probabilistic programming systems --- can be invoked with a single instruction. We have also seen that global sequential Monte Carlo and mean field algorithms are similarly straightforward to describe. Support for programmable inference does not necessarily increase the education burden on would-be probabilistic programmers, although it does provide a way to avoid limiting probabilistic programmers to a potentially inadequate set of inference strategies.

The idea that inference strategies can be formalized as structured, compositionally specified inference programs operating on model programs is, to the best of our knowledge, new to Venture. Under this view, standard inference algorithms actually correspond to primitive inference programming operations or program templates, some of which depend on specific features of the model program being acted upon. This perspective suggests that far more complex inference strategies should be possible if the right primitives, means of combination and means of abstraction can be identified. Considerations of modularity, analyzability, soundness, completeness, and reuse will become central. and will be complicated by the interaction between inference programs and model programs. For example, inference programmers will need to be able to predict the asymptotic scaling of inference instructions, factoring out the contribution of the computational complexity of the model expressions to which a given inference instruction is being applied. Another example comes from considering abstraction and reuse. It should be possible to write compound inference procedures that can be reused across different models, and perhaps even use inference to learn these procedures via an appropriate hierarchical Bayesian formulation. 

Another view, arguably closer to the mainstream view in machine learning, is that inference algorithms are better thought of by analogy to mathematical programming and operations research, with each algorithm corresponding to a ``solver'' for a well-defined class of problems with certain structure. This perspective suggests that there is likely to be a small set of monolithic, opaque mechanisms that are sufficient for most important problems. In this setting, one might hope that inference mechanisms can be matched to models and problems via simple heuristics, and that the problem of automatically generating high-quality inference strategies will prove easier than query planning for databases, and will be vastly easier than automatic programming. 
\todo[color=green]{[medium] Add in tactics languages for theorem provers: coq, isabelle, etc. and term rewriting logics}

It remains to be seen whether the traditional view is sufficient in practice or if it underestimates the richness of inference and its interaction with modeling and problem specification.

\subsection{Procedural and Declarative Interpretations}

\todo[color=yellow]{[medium] Include an example that illustrates equivalences between procedural and declarative}

We briefly consider the relationship between procedural and declarative interpretations of Venture programs.

Venture code has a direct procedural reading: it defines a probabilistic generative process that
samples hypotheses, checks constraints, and invokes inference instructions that trigger specific algorithms for reconciling the hypotheses with the constraints. Re-orderings of the instructions can significantly impact runtime and change the distribution on outputs. The divergence between the true conditioned distribution on execution traces and the distribution encoded by the program may depend strongly on what inference instructions are chosen and how they are interleaved with the incorporation of data.

Venture code also has declarative readings that are unaffected by some of these procedural details.
One way to formalize the of meaning of a Venture program is as a probability distribution over execution traces. A second approach is to ignore the details of execution and restrict attention to the joint probability distribution of the values of all {\tt PREDICT}s so far. A third approach, consistent with Venture's interactive interface, is to equate the meaning of a program with the probability distribution of the values of all the {\tt PREDICT}s in all possible sequences of instructions that could be executed in the future. Under the second and third readings, many programs are equivalent, in that they induce the same distribution albeit with different scaling behavior.

As the amount of inference performed at each {\tt INFER} instruction increases, these interpretations coalesce, recovering a simple semantics based on sequential Bayesian reasoning. Consider replacing all inference instructions are replaced with exact sampling --- {\tt [INFER (rejection default one 1)]} --- or a sufficiently large number of transitions of a generic inference operator, such as {\tt [INFER (mh default one 1000000)]}. In this case, each {\tt INFER} implements a single step of sequential Bayesian reasoning, conditioning the distribution on traces with all the {\tt OBSERVE}s since the last {\tt INFER}. 
The distribution after each {\tt INFER} becomes equivalent to the distribution represented by all programs with the same {\tt ASSUME}s, all the {\tt OBSERVE}s before the infer (in any order), and a single {\tt INFER}. The computational complexity varies based on the ordering and interleaving of {\tt INFER}s and {\tt OBSERVE}s, but the declarative meaning is unchanged. Although correspondence with these declarative, fully Bayesian semantics may require an unrealistic amount of computation in real-world applications, close approximations can be useful tools for debugging, and the presence of the limit may prove useful for probabilistic program analysis and transformation.

Venture programs represent distributions by combining modeling operations that sample values for expressions, constraint specification operations that build up a conditioner, and inference operations that evolve the distribution closer to the conditional distribution induced by a conditioner. Later in this paper we will see how to evaluate the partial probability densities of probabilistic execution traces under these distributions, as well as the ratios and gradients of these partial densities that are needed for a wide range of inference schemes.

\subsection{Markov chain and sequential Monte Carlo architectures}

The current Venture implementation maintains a single probabilistic execution trace per virtual machine instance. This trace is initialized by simulating the code from the {\tt ASSUME} and {\tt OBSERVE} instructions, and stochastically modified during inference via transition operators that leave the current conditioned distribution on traces invariant. The prior and posterior probability distributions on traces are implicit, but can be probed by repeatedly re-executing the program and forming Monte Carlo estimates.

This Markov chain architecture has been chosen for simplicity, but sequential Monte Carlo architectures based on weighted collections of traces are also possible and indeed straightforward. The number of initial traces could be specified via an {\tt INFER} instruction at the beginning of the program. Forward simulation would be nearly unchanged. {\tt OBSERVE} instructions would attach weights to traces based on the ``likelihood'' probability density corresponding to the constrained random choice in each observation, and {\tt PREDICT} instructions would read out their values from a single, arbitrarily chosen ``active'' trace. An {\tt [INFER (resample <k>)]} instruction would then implement multinomial resampling, and also change the active trace to ensure that {\tt PREDICT}s are always mutually consistent. Venture's other inference programming instructions could be treated as rejuvenation kernels \citep{del2006sequential}, and would not need to modify the weights\footnote{The only subtlety is that transition operators must not change which random choice is being constrained, as this would require changing the weight of the trace.}. This kind of sequential Monte Carlo implementation would have the advantage that the weights could be used to estimate marginal probability densities of given {\tt OBSERVE}s, and that another source of non-embarrassing parallelism would be exposed. Integrating sophisticated coupling strategies from the $\alpha$SMC framework \citep{2013arXiv1309.2918W} into the inference language could also prove fruitful.

Running separate Venture virtual machines is guaranteed to produce independent samples from the distribution represented by the Venture program. This distribution will typically only approximate some desired conditional. If there is no need to quantify uncertainty precisely, then Venture programmers can 
append repetitions of a sequence of {\tt INFER} and {\tt PREDICT} instructions to their program. Unless the {\tt INFER} instructions use rejection sampling, this choice yields {\tt PREDICT} outputs that are dependent under both Markov chain and sequential Monte Carlo architectures. It only approximates the behavior of independent runs of the program. Application constraints will determine what approximation strategies are most appropriate for each problem.

\todo[color=yellow]{[big] Figure using the multiripl illustrating convergence}
\todo[color=yellow]{[big] Careful treatment using some new math notation; tease apart degrees of procedural/declarative}

\subsection{Examples}

Here we give simple illustrations of the Venture language, including some standard modeling idioms as well as the use of custom inference instructions. Venture has been also used to implement applications of several advanced modeling and inference techniques; examples include generative probabilistic graphics programming \citep*{mansinghka2013approximate} and topic modeling \citep{Blei2003}. A description of these and other applications is beyond the scope of this paper.

\todo[color=green]{[medium] include LDA w variational and asymptotics w real results on NIPS}

\todo[color=green]{[medium] include CrossCat w real results on DHA and asymptotics}

\todo[color=green]{[medium] include some filtering result with asymptotics}

\todo[color=green]{[small] revise experimental section text to explain coverage}

\subsubsection{Hidden Markov Models}

To represent a Hidden Markov model in Venture, one can use a stochastic recursion to capture the hidden state sequence, and index into it by a stochastic observation procedure. Here we give a variant with continuous observations, a binary latent state, and an a priori unbounded number of observation sequences to model:

\vspace{0.05in}\begin{lstlisting}[frame=single,showstringspaces=false]
[ASSUME observation_noise (scope_include 'hypers 'unique (gamma 1.0 1.0))]

[ASSUME get_state
  (mem (lambda (seq t)
      (scope_include 'state t
         (if (= t 0) 
             (bernoulli 0.3)
             (transition_fn (get_state seq (- t 1)))))))]

[ASSUME get_observation
  (mem (lambda (seq t)
         (observation_fn (get_state seq t))))]

[ASSUME transition_fn
  (lambda (state)
      (scope_include 'state t (bernoulli (if state 0.7 0.3))))]

[ASSUME observation_fn
  (lambda (state)
    (normal (if state 3 -3) observation_noise))]

[OBSERVE (get_observation 1 1) 3.6]
[INFER (mh default one 10)]
[OBSERVE (get_observation 1 2) -2.8]
[INFER (mh default one 10)]
<...>
\end{lstlisting}

This is a sequentialized variant of the "default" resimulation-based Metropolis-Hastings inference scheme. If all but the last {\tt INFER} statement were removed, the program would yield the same stochastic transitions as several Church implementations, but with linear (rather than quadratic) scaling in the length of the sequence. Interleaving inference with the addition of observations improves over bulk incorporation of observations by mitigating some of the strong conditional dependencies in the posterior.

Another inference strategy is particle Markov chain Monte Carlo. For example, one could combine Metropolis-Hastings moves on the hyperparameters, given the latent states, with a conditional sequential Monte Carlo approximation to Gibbs over the hidden states given the hyper parameters and observations. Here is one implementation of this scheme, where 10 transitions of Metropolis-Hastings are done on the hyper parameters for every 5 transitions of approximate Gibbs based on 30 particles, all repeated 50 times:

\vspace{0.05in}\begin{lstlisting}[frame=single,showstringspaces=false]
[INFER (cycle ((mh hypers one 10) (pgibbs state ordered 30 5)) 50)]
\end{lstlisting}

The global particle Gibbs algorithm from \citep{wjwmpgibbs2014} with 30 particles would be expressed as follows:

\vspace{0.05in}\begin{lstlisting}[frame=single,showstringspaces=false]
[INFER (pgibbs default ordered 30 100)]
\end{lstlisting}

Note that Metropolis-Hastings transitions can be more effective than pure conditional sequential Monte Carlo for handling global parameters \citep{andrieu2010particle}. This is because MH moves allow hyperparameter inference to be constrained by all latent states.

In real-time applications, hyperparameter inference is sometimes skipped. Here is one representation of a close relative\footnote{The only difference is that a particle filter exposes all its weighted particles for forming Monte Carlo expectations or for rapidly obtaining a set of approximate samples. Only straightforward modifications are needed for Venture to expose a set of weighted traces instead of a single trace and literally recover particle filtering.} of a standard 30-particle particle filter that uses randomly chosen hyper parameters and yields a single latent trajectory:

\vspace{0.05in}\begin{lstlisting}[frame=single,showstringspaces=false]
[INFER (pgibbs state ordered 30 1)]
\end{lstlisting}

\subsubsection{Hierarchical Nonparametric Bayesian Models}

Here we show how to implement one version of a multidimensional Dirichlet process mixture of Gaussians \citep{rasmussen1999infinite}:

\vspace{0.05in}\begin{lstlisting}[frame=single,showstringspaces=false]
[ASSUME alpha (scope_include 'hypers 0 (gamma 1.0 1.0))]
[ASSUME scale (scope_include 'hypers 1 (gamma 1.0 1.0))]

[ASSUME crp (make_crp alpha)]

[ASSUME get_cluster (mem (lambda (id)
  (scope_include 'clustering id (crp))))]

[ASSUME get_mean (mem (lambda (cluster dim)
  (scope_include 'parameters cluster (normal 0 10))))]

[ASSUME get_variance (mem (lambda (cluster dim)
  (scope_include 'parameters cluster (gamma 1 scale))))]

[ASSUME get_component_model (lambda (cluster dim)
  (lambda () (normal (get_mean cluster dim) (get_variance cluster dim))))]

[ASSUME get_datapoint (mem (lambda (id dim)
  ((get_component_model (get_cluster id dim)))))]

[OBSERVE (get_datapoint 0 0) 0.2]
<...>

; default resimulation-based Metropolis-Hastings scheme
[INFER (mh default one 100)]
\end{lstlisting}

The parameters are explicitly represented, i.e. ``uncollapsed'', rather than integrated out as they often are in practice. While the default resimulation-based Metropolis-Hastings scheme can be effective on this problem, it is also straightforward to balance the computational effort differently:

\begin{lstlisting}[frame=single]
[INFER (cycle ((mh hypers one 1)
                        (mh parameters one 5)
                        (mh clustering one 5))
                        1000)]
\end{lstlisting}

On each execution of this {\tt cycle}, one hyperparameter transition, five parameter transitions (each to both parameters of a randomly chosen cluster), and five cluster reassignments are made. Which hyperparameter, parameters and cluster assignments are chosen is random. As the number of data points grows, the ratio of computational effort devoted to inference over the hyperparameters and cluster parameters to the cluster assignments is higher than it is for the default scheme. Note that the complexity of this inference instruction, as well as the computational effort ratio, depend on the number of clusters in the current trace.

It is also straightforward to use a structured particle Markov chain Monte Carlo scheme over the cluster assignments:

\begin{lstlisting}[frame=single]
[INFER (mixture ((0.2 (mh hypers one 10))
                            (0.5 (mh parameters one 5))
                            (0.3 (pgibbs clustering ordered 2 2)))
                            100)]
\end{lstlisting}

Due to the choice of only 2 particles for the {\tt pgibbs} inference strategy, this scheme closely resembles an approximation to blocked Gibbs over the indicators based on a sequential initialization of the complete model. Also note that despite the low mixing weight on the {\tt clustering} scope in the {\tt mixture}, this inference program allocates asymptotically greater computational effort to inference over the cluster assignments than the previous strategy. This is because the {\tt pgibbs} transition operator is guaranteed to reconsider every single cluster assignment.

%
\todo[color=green]{[medium] Include road scene example and CAPTCHA example}
\subsubsection{Inverse Interpretation}

We now describe {\em inverse interpretation}, a modeling idiom that is only possible in Turing-complete languages. Recall that Venture modeling expressions are easy to represent as Venture data objects, and Venture models can invoke the evaluation and application of arbitrary Venture stochastic procedures. These Scheme-like features make it straightforward to write an evaluator --- perhaps better termed a simulator --- for a Turing-complete, higher-order probabilistic programming language.

This application highlights Turing-completeness and also embodies a new potentially appealing path for solving problems of probabilistic program synthesis. In less expressive languages, learning programs (or structure) requires custom machinery that goes beyond what is provided by the language itself. In Venture, the same inference machinery used for state estimation or causal inference can be brought to bear on problems of probabilistic program synthesis. The dependency tracking and inference programming machinery that is general to Venture can be brought to bear on the problem of approximately Bayesian learning of probabilistic programs in a Venture-like language\footnote{Although we are still far from a study of the expressiveness of probabilistic languages via definitional interpretation \citep{reynolds1972definitional, abelson1983structure}, it seems likely that probabilistic programming formulations of probabilistic program synthesis --- inference over a space of probabilistic programs, possibly including inference instructions, and an interpreter for those programs --- will be revealing.}.

We first define some utility procedures for manipulating references, symbols and environments:

\vspace{0.05in}\begin{lstlisting}[frame=single,showstringspaces=false]
[ASSUME make_ref (lambda (x) (lambda () x))]
[ASSUME deref (lambda (x) (x))]

[ASSUME incremental_initial_environment
  (lambda ()
    (list
      (dict 
        (list (quote bernoulli)
              (quote normal)
              (quote plus)
              (quote times)
              (quote branch))
        (list (make_ref bernoulli)
              (make_ref normal)
              (make_ref plus)
              (make_ref times)
              (make_ref branch)))))]

[ASSUME extend_env
  (lambda (outer_env syms vals) 
    (pair (dict syms vals) outer_env))]

[ASSUME find_symbol
  (lambda (sym env)
    (if (contains (first env) sym)
     	(lookup (first env) sym)
    	(find_symbol sym (rest env))))]
\end{lstlisting}

The most interesting of these are {\tt make\_ref} and {\tt deref}. These use closures to pass references around the trace, using an idiom that avoids unnecessary growth of scaffolds. Consider an execution trace in which a value that is the argument to {\tt make\_ref} becomes the principal node of a transition. Only those uses of the reference to the value that have been passed to {\tt deref} will become resampling nodes. The value of the reference is unchanged, though the value the reference refers to is not. This permits dependence tracking through the construction of complex data structures.

Given this machinery, it is straightforward to write an evaluator for a simple function language that has access to arbitrary Venture primitives:

\vspace{0.05in}\begin{lstlisting}[frame=single,showstringspaces=false]
[ASSUME incremental_venture_apply
  (lambda (op args) (eval (pair op (map_list deref args)) (get_empty_environment)))]

[ASSUME incremental_apply
  (lambda (operator operands)
    (incremental_eval (deref (lookup operator 2))
		      (extend_env (deref (lookup operator 0))
					  (deref (lookup operator 1))
					  operands)))]

[ASSUME incremental_eval
  (lambda (expr env)
    (if (is_symbol expr)
    	(deref (find_symbol expr env))
	    (if (not (is_pair expr))
	        expr
     	    (if (= (deref (lookup expr 0)) (quote lambda))
	        	  (pair (make_ref env) (rest expr))
        		  ((lambda (operator operands)
		             (if (is_pair operator)
            		       (incremental_apply operator operands)
		                 (incremental_venture_apply operator operands)))
           		   (incremental_eval (deref (lookup expr 0)) env)
		           (map_list (lambda (x)
		                             (make_ref (incremental_eval (deref x) env)))
		                     (rest expr)))))))]
\end{lstlisting}

It is also possible to generate the input {\tt expr}s using another Venture program, and use general-purpose, Turing-complete inference mechanisms to explore a hypothesis space of expressions given constraints on the values that result. We call this the  {\em inverse interpretation} approach to probabilistic program synthesis. As in Church --- and contrary to \citep{liang10programs} --- inverse interpretation algorithms are not limited to rejection sampling. But while Church was limited to a single-site Metropolis-Hastings scheme, Venture programmers have more options. In Venture it is possible to associate portions of the program source (and portions of the induced program's executions) with custom inference strategies.

Here is an example expression grammar that can be used with the incremental evaluator:

\vspace{0.05in}\begin{lstlisting}[frame=single,showstringspaces=false]
[ASSUME genBinaryOp (lambda () (if (flip) (quote plus) (quote times)))]
[ASSUME genLeaf (lambda () (normal 4 3))]
[ASSUME genVar (lambda (x) x)]

[ASSUME genExpr
  (lambda (x)
    (if (flip 0.4) 
      (genLeaf)
      (if (flip 0.8)
         (genVar x)
         (list (make_ref (genBinaryOp)) (make_ref (genExpr x)) (make_ref (genExpr x))))))]

[ASSUME noise (gamma 5 .2)]
[ASSUME expr (genExpr (quote x))]

[ASSUME f
  (mem 
    (lambda (y) 
      (incremental_eval expr 
                        (extend_env (incremental_initial_environment) 
                                            (list (quote x)) 
                                            (list (make_ref y))))))]
                                            
[ASSUME g (lambda (z) (normal (f z) noise))]

[OBSERVE (g 1) 10] ;f(x) = 5x + 5
[OBSERVE (g 3) 20]
<...>
\end{lstlisting}

Indirection via references substantially improves the asymptotic scaling of programs like these. When a given production rule in the grammar is resimulated, only those portions of the execution of the program that depend on the changed source code are resimulated. A naive implementation of an evaluator would not have this property.

Scaling up this approach to larger symbolic expressions and small programs will require multiple advances. Overall system efficiency improvements will be necessary. Inverse interpretation also may benefit from additional inference operators, such as Hamiltonian Monte Carlo for the continuous parameters. Expression priors with inference-friendly structures would also help. For example, a prior where resimulation recovers some of the search moves from \citep{duvenaud2013structure} may be expressible by separately generating symbolic expressions and the contents of the environment into which they are evaluated. Longer term, it may be fruitful to explore formalizations of some of the knowledge taught to programmers using probabilistic programming.

\todo[color=green]{[big] Redo expression learning with better inference and target example, after fixing type issues; discuss connection to Grosse via carefully chosen prior}

\section{Stochastic Procedures}

Random choices in Venture programs arise due to the invocation of {\em
  stochastic procedures} (SPs). These stochastic procedures accept input
arguments that are values in Venture and sample output values given those
inputs. Venture includes a built-in stochastic procedure library, which includes
SPs that construct other SPs, such as the SP \verb|make_csp| which the special
form \verb|lambda| gets desugared to. 
Stochastic procedures can also be added as extensions to Venture, and provide a mechanism
for incremental optimization of Venture programs. Model fragments for which Venture delivers inadequate performance can be migrated to native inference code that interoperates with the enclosing Venture program.

\subsection{Expressiveness and extensibility}

Many typical random variables, such as draws from a Bernoulli or Gaussian distribution, are straightforward to represent computationally. One common approach has been to use a pair of functions: a {\em simulator} that maps from an input space of values $\mathcal{X}$ and a stream of random bits $\{0,1\}^*$ to an output space of values $\mathcal{Y}$, and a {\em marginal density} that maps from $(x,y)$ pairs to $[0, \infty)$. This representation corresponds to the ``elementary random procedures'' supported by early Church implementations. Repeated invocations of such procedures correspond to IID sequences of random variables whose densities are known. While simple and intuitive, this simple interface does not naturally handle many useful classes of random objects. In fact, many objects that are easy to express as compound procedures in Church and in Venture cannot be made to fit in this form.

Stochastic procedures in Venture support a broader class of random objects:

\begin{enumerate}

\item {\bf Higher-order stochastic procedures, such as \verb|mem| (including stochastic memorization), \verb|apply|
    and \verb|map|.} Higher-order procedures may accept procedures as arguments,
  apply these procedures internally, and produce procedures as return
  values. Stochastic procedures in Venture are equipped with a simple mechanism
  for handling these cases. In fact, it turns out that all structural changes to
  execution traces --- including those arising from the execution of constructs
  that affect control flow, such as \verb|IF| --- can be handled by this
  mechanism. This simplifies the development of inference
  algorithms, and permits users to extend Venture by adding new primitives that
  affect the flow of control.

\item {\bf Stochastic procedures whose applications are exchangeably
  coupled.} Examples include collapsed representations of conjugate
  models from Bayesian statistics, combinatorial objects from Bayesian
  nonparametrics such as the Chinese Restaurant and Pitman-Yor
  processes, and probabilistic sequence models (such as HMMs) whose
  hidden state sequences can be efficiently marginalized out. Support
  for these primitives whose applications are coupled is important for
  recovering the efficiency of manually optimized samplers, which
  frequently make use of collapsed representations. Whereas
  exchangeable primitives in Church are thunks, which prohibits
  collapsing many important models such as HMMs, Venture supports
  primitives whose applications are row-wise partially exchangeable
  across different sets of arguments. The formal requirement is that
  the cumulative log probability density of any sequence of input-output pairs is
  invariant under permutation.
  
\item {\bf Likelihood-free stochastic procedures that lack tractable
  marginal densities.} Complex stochastic simulations can be
  incorporated into Venture even if the marginal probability density
  of the outputs of these simulations given the inputs cannot be
  efficiently calculated. Models from the literature on Approximate
  Bayesian Computation (ABC), where priors are defined over the
  outcome of forward simulation code, can thus be naturally supported
  in Venture. Additionally, a range of ``doubly intractable''
  inference problems, including applications of Venture to reasoning
  about the behavior of approximately Bayesian reasoning agents, can
  be included using this mechanism.
  
\item {\bf Stochastic procedures with external latent variables that
  are invisible to Venture.} There will always be models that admit
  specialized inference strategies whose efficiency cannot be
  recovered by performing generic inference on execution traces. One
  of the principal design decisions in Venture is to allow these
  strategies to be exploited whenever possible by supporting a broad
  class of stochastic procedures with custom inference over internal
  latent variables, hidden from the rest of Venture. The stochastic
  procedure interface thus serves as a flexible bridge between Venture
  and foreign inference code, analogous to the role that foreign
  function interfaces (FFIs) play in traditional programming
  languages.
\end{enumerate}

\subsection{Primitive stochastic procedures}

Informally, a primitive stochastic prodecure (PSP) is an object that
can simulate from a family of distributions indexed by some arguments. In addition to simulating, PSPs may be
able to report the logdensity of an output given an input, and may
incorporate and unincorporate information about the samples drawn from
it using mutation, e.g. in the case of a conjugate prior. This
mutation cannot be aribitrary: the draws from the PSP must remain
row-wise partially exchangeable as discussed above. A PSP may also
have custom proposal kernels, in which case it must be able to return
its contribution to the Metropolis-Hastings acceptance rate. For
example, the PSP that simulates Gaussian random variables may provide
a drift kernel that proposes small steps around its previous location,
rather than resampling from the prior distribution.

Primitive stochastic procedures are parameterized by the following
properties and behaviors:

\begin{enumerate}
\item {\tt isStochastic()} --- does this PSP consume randomness when
  it is invoked?

\item {\tt canAbsorbArgumentChanges()} --- can this PSP absorb changes
  to its input arguments? If {\tt true}, then this PSP must correctly
  implement {\tt logdensity()} as described below.

\item {\tt childrenCanAbsorbAtApplications()} --- does this PSP return
  an SP that implements the short-cut ``absorbing at applications''
  optimization, needed to integrate optimized expressions for the log marginal probability of sufficient statistics in standard conjugate models.\todo[color=yellow]{[small, end]Clarify when and where AAA Is defined}

\item {\tt value = simulate(args)} --- samples a {\tt value} for an
  application of the PSP, given the arguments {\tt args}

\item {\tt logp = logdensity(value, args)} --- an optional procedure
  that evaluates the log probability density\footnote{This density is
    implicitly defined with respect to a PSP-specific (but argument independent)
    choice of dominating measure. For PSPs that are guaranteed to produce
    discrete outputs, the measure is assumed to be the counting
    measure, so {\tt logdensity} is equivalent to the log probability
    mass function. A careful measure-theoretic treatment of Venture is left for future work.} of an output given the input arguments
  $p_{psp}({\tt simulate(args) = value} | {\tt args})$.

\item {\tt incorporate(aux, value, args)} --- incorporate the value stored in the variable named
  {\tt value} into the auxiliary storage {\tt aux} associated with the SP that contains the PSP. {\tt incorporate()} is used to implement SPs
  whose applications are exchangeably coupled. While it is always
  sufficient to store and update the full set of {\tt value}s returned for
  each observed {\tt args}, often only the counts (or some other
  sufficient statistics) are necessary.

\item {\tt unincorporate(aux, value, args)} --- remove {\tt
  value} from the auxiliary storage, restoring it to a state
  consistent with all other values that have been {\tt incorporate}d
  but not {\tt unincorporate}d; this is done when an application is
  unevaluated.
  
 \todo[color=yellow]{[medium] Check lite to finish SP extras: add custom var/sim/delt kernels, opt enumeration, and good treatment of latents}
\end{enumerate}

\subsection{The Stochastic Procedure Interface}

The stochastic procedure interface specifies the contract that Venture
primitives must satisfy to preserve the coherence of Venture's
inference mechanisms. It also serves as the vehicle by which external
inference systems can be integrated into Venture. This interface preserves the ability of primitives to  dynamically create and destroy internal latent variables hidden from Venture and to perform custom inference over these latent variables.

\subsubsection{Definition}

\begin{defn}[Stochastic procedure]
A stochastic prodecure is a pair \(  (\request, \out) \) of PSPs, along with a
latent variable simulator, where

\begin{enumerate}
\item \( \request \) returns:

\begin{enumerate}
\item A list of tuples \( (\addr, \expr, \env) \) that represent expressions
  whose values must be available to \( \out \) before it can start its simulation. We refer to
  requests of this form as \emph{exposed simulation requests} (ESRs),

\item A list of opaque tokens that can be interpreted by the SP as the latent
  variables that \( \out \) will need in order to simulate its output, along with the values of the exposed simulation requests. We refer
  to requests of this form as \emph{latent simulation requests} (LSRs).
\end{enumerate}

\item The latent variable simulator responds to LSRs by simulating any latent
  variables requested.

\item \( \out \) returns the final output of the procedure, conditioned on the
  inputs, the results of any of the exposed simulation requests, and the results of
  any latent simulation requests.

\end{enumerate}

\end{defn}
 
\subsubsection{Exposed Simulation Requests}

We want our procedures to be able to pass \( (\expr,\env) \) pairs to Venture
for evaluation, and make use of the results in some way. A procedure may also
have multiple applications all make use of a shared evaluation, e.g. \( \mem \),
and in these cases the procedure must take care to request the same \( \addr \)
each time, and the \( (\expr,\env) \) will only be evaluated the first time and
then reused thereafter. 

Specifically, an ESR request of the  \( (\addr,\expr,\env) \) is handled by
Venture as follows. First Venture checks the requesting SP's namespace to see if
it already has an entry with address \( \addr \). If it does not, then Venture
evaluates \( \expr \) in \( \env \), and adds the mapping \( \addr \to
\mathbf{root} \) to the SP's namespace, where \textbf{root} is the root of the
evaluation tree. If the SP's namespace does contain \( \addr \), then Venture
can look up \textbf{root}. Either way, Venture wires in \textbf{root} as an
extra argument to the output node. 

\subsubsection{Latent Simulation Requests and the Foreign Inference Interface}

Some procedures may want to simulate and perform inference over latent variables that are hidden from Venture. Consider an optimized implementation of a hidden Markov model integrated into the following Venture program:

\begin{lstlisting}
[ASSUME my_hmm (make_hmm 10 0.1 5 0.2)]
[OBSERVE (my_hmm 0 0) 0]
[OBSERVE (my_hmm 0 1) 1]
[OBSERVE (my_hmm 0 2) 0]
[INFER (mh default one 20)]
[PREDICT (my_hmm 0 3)]
[PREDICT (my_hmm 1 0)]
\end{lstlisting}

Here we have a constructor SP {\tt (make\_hmm <num-states> <transition-hyper> <num-output-symbols> <observation-hyper)} that generates an (uncollapsed) hidden Markov model by generating the rows of the transition and observation matrices at random. The assumption is that {\tt make\_hmm} is a primitive, although it would be straightforward to implement {\tt make\_hmm} as a compound procedure in Venture, using variations of the examples presented earlier. The constructor returns a procedure bound to the symbol {\tt my\_hmm} that permits observations from this process to be queried via {\tt (my\_hmm <sequence-id> <index>)}. The program adds a sequence fragment of length three then requests predictions for the next observation in the sequence, as well as the initial observation from an entirely new sequence.

This probabilistic program captures a common pattern: integrating Venture with an foreign probabilistic model fragment that can be dynamically queried and contains latent variables hidden from Venture. It is useful to partition the random choices in this program as follows. The transition and observation matrices of the HMM could be viewed as part of the value of the {\tt my\_hmm} SP and therefore returned by the output PSP of the {\tt make\_hmm} SP. The observations are managed by Venture as the applications of the {\tt my\_hmm} SP. The hidden states, however, are fully latent from the standpoint of Venture, yet need to be created, updated and destroyed as invocations of {\tt my\_hmm} are created or destroyed and as their arguments (or the arguments to {\tt make\_hmm}) change.

Venture makes it possible for procedures to instantiate latent variables only as necessary to simulate a given program. The mechanism is similar to that for exposed simulation requests, except in this
case the requests--which we call latent simulation requests (LSRs)--are opaque
to Venture. Venture only calls appropriate methods on the SP at appropriate
times to ensure that all the bookkeeping is handled correctly. 

This framework is straightforward to apply to {\tt make\_hmm} and the stochastic procedure(s) that it returns. If {\tt my\_hmm} is queried for an observation
at time \( t \), the requestPSP can return the time \( t \) as an
LSR. When Venture tells the HMM to simulate that LSR, the HMM will
either do nothing if \( x_t \) already exists in its internal store of
simulations, or else continue simulating from its current position up
until \( x_t \). The outputPSP then samples \( o_t \) conditioned on
the latent \( x_t \). If the application at time \( t \) is ever
unevaluated, Venture will tell the HMM to detach the LSR \( t \),
which will cause the HMM to place the latents that are no longer
necessary to simulate the program into a ``latentDB'', which it
returns to Venture. Later on, Venture may tell the HMM to restore
latents from a latentDB, for example if a proposal is rejected and the
starting trace is being restored.

The main reason to encapsulate latent variables in this way, as
opposed to requesting them as ESRs, is so that the SP can use optimized implementations of inference over their values, potentially utilizing special-purpose inference methods. For example, the
uncollapsed HMM can implement forwards-filtering backwards-sampling to
efficiently sample the latent variables conditioned on all
observations. Such procedures are integrated into Venture by defining an ``Arbitrary Ergodic Kernel'' (AEKernel) which Venture may call during inference, and which is simply a black-box to
Venture. Note that this same mechanism may be used by SPs that do not make
latent simulation requests at all, but which have latent variables instantiated
upon creation, such as a finite-time HMM or an uncollapsed
Dirichlet-Multinomial.

\todo[color=yellow]{[big] Introduce and illustrate MakeSIVM for inference over inference, in all its flexibility}

\todo[color=yellow]{[medium] Based on lite, introduce AEAAA and latent kernels}

To implement this functionality, stochastic procedures must implement three procedures in addition to the procedures needed for their ESR requestor, LSR requestor and output PSPs:

\begin{enumerate}
\item {\tt simulateLatents(aux, LSR, shouldRestore, latentDB)} ---
  simulate the latents corresponding to {\tt LSR}, using the tokens in
  {\tt latentDB} (indexed by {\tt LSR}) to find a previous value if
  {\tt shouldRestore} is true.

\item {\tt detachLatents(aux, LSR, latentDB)} --- signal that the
  latents corresponding to the request {\tt LSR} are no longer needed,
  and store enough information in {\tt latentDB} so that the value can
  be recovered later.
  
\item {\tt AEInfer(aux)} --- Trigger the external implementation to perform inference over all latent variables using the contents of {\tt aux}. It is often convenient for {\em simulateLatents} to store latent variables in the {\tt aux} and for {\tt incorporate} to store the return values of applications in {tt aux}, along with the arguments that produced them.

\end{enumerate}

Examples of this use of the stochastic procedure interface can be found in current releases of the Venture system.

\subsubsection{Optimizations for higher-order SPs}

Venture provides a special mechanism that allows certain SPs to exploit the
ability to quickly compute the logdensity of its applications. Consider
the following program:
\begin{verbatim}
[ASSUME alpha (gamma 1 1)]
[ASSUME collapsed_coin (make_beta_bernoulli alpha alpha)]
[OBSERVE (collapsed_coin) False]
[OBSERVE (collapsed_coin) True]
<repeat 10^9 times>
[OBSERVE (collapsed_coin) False]
[INFER]
\end{verbatim}

A hand-written inference scheme would only keep track of the counts of the observations, and could perform rapid Metropolis-Hastings proposals on alpha by exploiting conjugacy. On the other hand, a naive generic inference scheme might visit all one billion observation nodes to compute the acceptance ratio for each proposal. We can achieve this efficient inference scheme by letting {\tt make\_beta\_bernoulli} be responsible for tracking the sufficient statistics from the applications of {\tt collapsed\_coin}, and for evaluating the log density of all those applications as a block. Stochastic procedures that return other stochastic procedures and implement this optimization are said to be {\em absorbing at applications}, often abbreviated AAA. We will discuss techniques for implementing this mechanism in a later section.

\subsubsection{Auxiliary State}

An SP is itself stateless, but may have an associated auxiliary store,
called \emph{SPAux}, that carries any mutable information. SPAuxs have
several uses:

\begin{enumerate}
\item If an SP makes exposed simulation requests, Venture uses the
  SPAux to store mappings from the addresses of the ESRs to the node
  that stores the result of that simulation.

\item If a PSP keeps track of its sample counts or other sufficient
  statistics, such as the collapsed-beta-bernoulli which stores the
  number of \textbf{true}s and \textbf{false}s, it will store this
  information in the SPAux.

\item If an SP makes latent simulation requests, then all latent
  variables it simulates to respond to those requests are stored in
  the SPAux.

\item Some SPs may optionally store part of the value of the SP
  directly in the SPAux. This is necessary if an SP cannot easily
  store its value, its latents, and its outputs separately.
\end{enumerate}

\todo[color=green]{[big] Examples of SPs: flip, ccoin, crp, mem, eval, apply, dpmem, lazyhmm, makeSIVM}

%
%
%
%
%
%
%
%
%
%
%
%
%
%


\section{Probabilistic Execution Traces}

Bayesian networks decompose probabilistic models into nodes for random variables, equipped with conditional probability tables, and edges for conditional dependencies. They can be interpreted as probabilistic programs via the ancestral simulation algorithm \citep{frey1997bayesian}. The network is traversed in an order consistent with the topology, and each random variable is generated given its immediate parents. Bayesian networks can also be viewed as expressing a function for evaluating the joint probability density of all the nodes in terms of a factorization given by the graph structure. 

Here we describe {\em probabilistic execution traces} (PETs), which serve analogous functions for Venture programs and address the additional complications that arise from Turing-completeness and the presence of higher order probabilistic procedures. We also describe recursive procedures for constructing and destroying probabilistic execution traces as Venture modeling language expressions are evaluated and unevaluated.

\todo[color=yellow]{[medium] Write comparative PET intro based on structure in the comments}

\subsection{Definition of a probabilistic execution trace}

Probabilistic execution traces consist of a directed graphical model representing the dependencies in the current execution, along with the stateful auxiliary data for each stochastic procedure, the Venture program itself, and metadata for existential dependencies and exchangeable coupling. We will typically identify executions and PETs, and denote them via the symbols \( \rho \) or \( \xi \).

PETs contain the following nodes:

\begin{enumerate}
\item One constant node for the global environment.

\item One constant node for every value bound in the global environment, which includes
  all built-in SPs.

\item One constant node for every call to \textbf{eval} on a expression that is
  either self-evaluating or quoted.

\item One lookup node for every call to \textbf{eval} that triggers a symbol
  lookup.

\item One request node and one output node for every call to \textbf{eval} that
  triggers an SP application. We refer to the operator nodes and operand nodes
  of request nodes and output nodes, but note that these are not special node
  types.

\end{enumerate}

PETs also contain the following edges:

\begin{enumerate}
\item One lookup edge to each lookup node from the node it is looking up.

\item One operator edge to every request node from its operator node.

\item One operator edge to every output node from its operator node.

\item One operand edge to every request node from each of its operand nodes.

\item One operand edge to every output node from each of its operand nodes.

\item One requester edge to every output node from its corresponding request node.

\item One ESR edge from the root node of every SP family to each SP application that
  requests it.
\end{enumerate}

Every node represents a random variable and has a value that cannot
change during an execution. The PET also includes the SPAux for every SP that
needs one. Unlike the values in the nodes, the SPAuxs may be mutated during an
execution, for example to increment the number of \emph{true}s for a
beta-bernoulli.

\subsection{Families}

We divide our traces into families: one Venture family for every
\emph{assume}, \emph{predict}, \emph{observe} directive, and one SP family for every
unique ESR requested during forward simulation. Because of our uniform
treatment of conditional simulation, executions of programs satisfy 
the following property: the structure of every family is a function of the
expression only, and does not depend on the random choices made while evaluating
that expression. The only part of the topology of the graph that can change is
which ESRs are requested.

\subsection{Exchangeable coupling}

\todo[color=green]{[small] Better motivate exchangeable coupling, based on Cameron's feedback}

Given our exchangeability assumptions for SPs, we can cite
(generalized) de Finetti \citep{orbanzbayesian} to conclude that there
is, in addition to the observed random variables explicitly
represented in the PET, one latent random variable \( \theta_{f} \)
for every SP \( f \) in the PET corresponding to the unobserved de
Finetti measure, and one latent variable \( \theta_{f,\args} \) for
every set of arguments \( \args \) that \( f \) is called on, with an
edge from \( f \)'s maker-node to \( \theta_f \), edges from \(
\theta_f \) to every \( \theta_{f,\args} \), and an edge from
\(\theta_{f,\args} \) to every node corresponding to an application of
the form \( f(\args) \). However, each \( \theta_f \) and \(
\theta_{f,\args} \) is marginalized out by each \( f \) by way of
mutation on its SPAux, effectively introducing a hyperedge that
indirectly couples the application nodes of all applications of \( f
\).

We have chosen not to represent these dependencies in the graphical
structure of the PET for the following reasons. First, once we
integrate out \( \theta_{f} \) and \( \theta_{f,\args} \), we can only
encode these dependencies at all with directed edges once we fix a
specific ordering for the applications. Second, for the orderings we
will be interested in, the graph that combines both types of directed
edges would be cyclic. This would complicate future efforts to develop
a causal semantics for PETs and Venture programs.

\subsection{Existential dependence and contingent evaluation}

An SP family is existentially dependent on the nodes that request it as part of
an ESR, in the sense that if at any point it is not requested by
\emph{any} nodes, then the family would not have been computed while simulating
the PET. Existential dependence is then handled with garbage collection
semantics, whereby an SP family can no longer be part of the PET if it is not
selected by any \( W \) nodes, and thus should be unevaluated.

\subsection{Examples}

Here we briefly give example PETs for simple Venture programs.

\subsubsection{Trick coin}

This is a variant of our running example. It defines a model that can be used for
inferences about whether or not a coin is tricky, where a trick coin is allowed to have any weight between 0 and 1. The version we use includes one observation that a single flip of the coin came up heads.

To keep the PET as simple as possible, we give both the program and
the PET in the form where \verb|IF| has already been desugared into an
SP application: \verb|(IF <predicate> <consequent> <alternate>)| has
been replaced with
\verb|(branch <predicate> (quote <consequent>) (quote <alternate>))|,
where \verb|branch| is an ordinary stochastic procedure.

\begin{verbatim}
[ASSUME coin_is_tricky (bernoulli 0.1)]
[ASSUME weight (branch coin_is_tricky (quote (beta 1.0 1.0)) (quote 0.5))]
[OBSERVE (bernoulli weight) true]
\end{verbatim}

Figure~\ref{fig:trickcoin} shows the two PET structures that can arise from simulating this program, along with arbitrarily chosen values.

\begin{figure}
\hspace{-0.5in}
(a)
\includegraphics[width=3in]{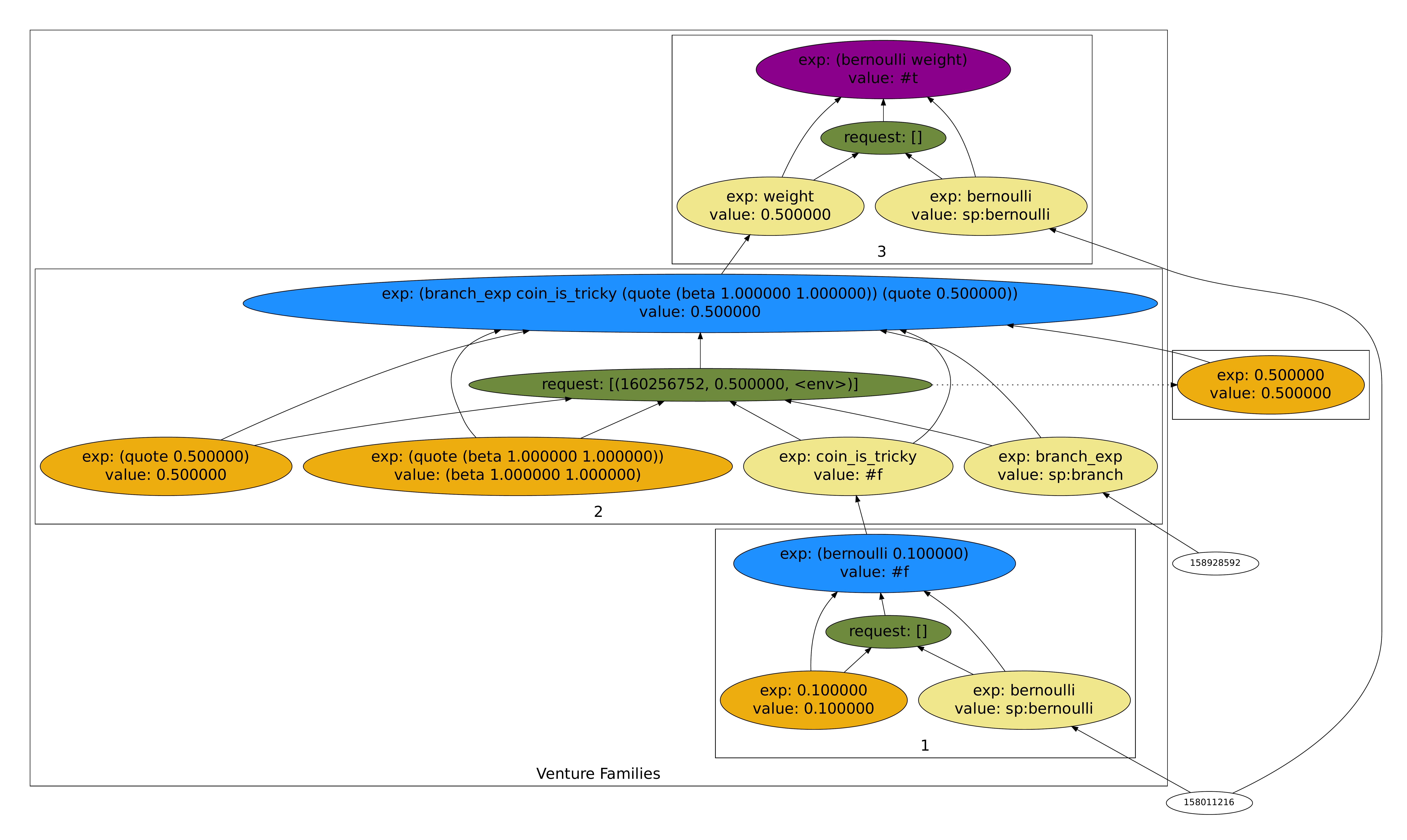}
(b) 
\includegraphics[width=4in]{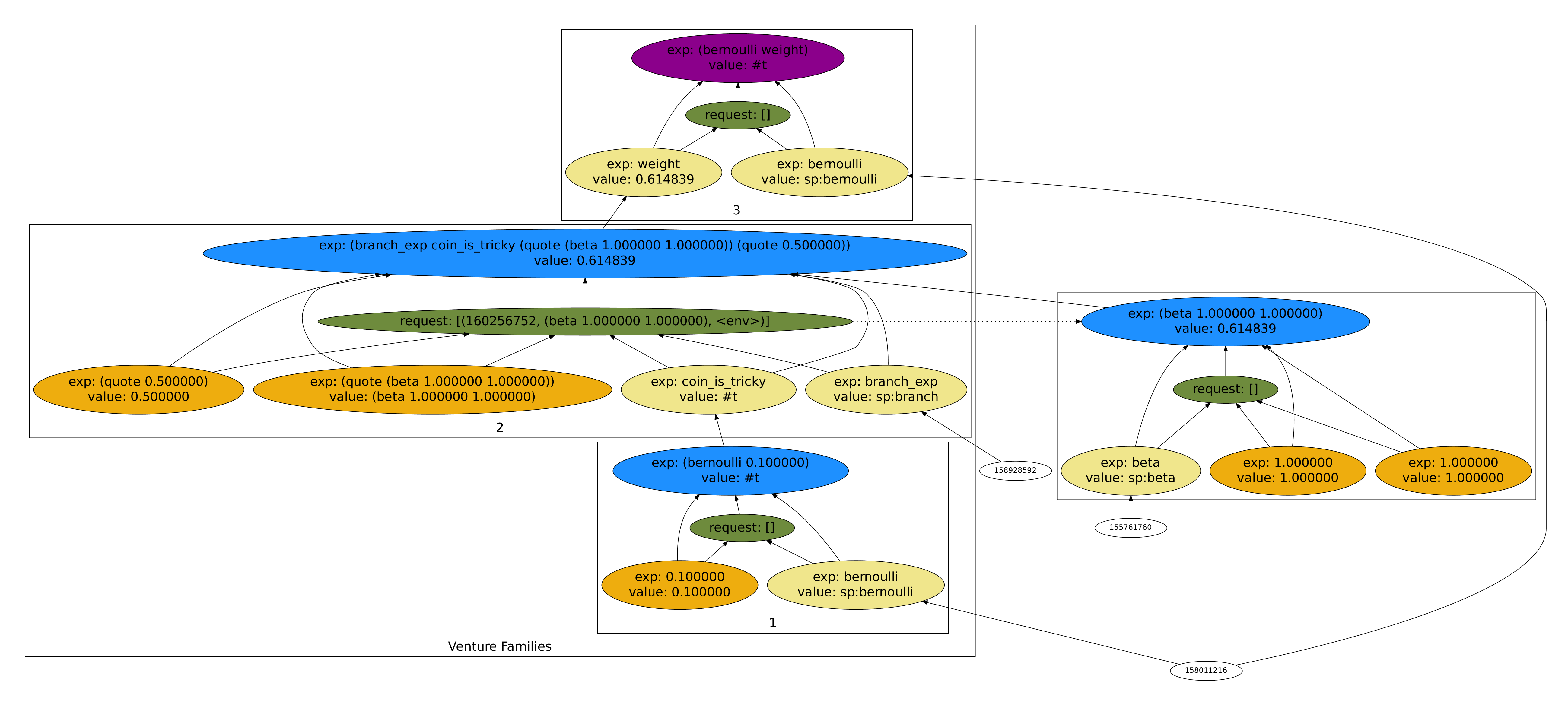}
\caption{{\bf The two different PET structures corresponding to the trick coin program.} (a) An execution trace where the coin is fair. (b) An execution trace where the coin is tricky, containing additional nodes that depend existentially on the coin flip.}
\label{fig:trickcoin}
\end{figure}

\todo[color=green]{[small] enlarge into pedagogical example discuss trick coin example PET}

\subsubsection{A simple Bayesian network}

Figure~\ref{fig:sprinkler} shows a PET for the following program, implementing a simple Bayesian network:

\begin{verbatim}
[ASSUME rain (bernoulli 0.2)]
[ASSUME sprinkler (bernoulli (branch rain 0.01 0.4))]
[ASSUME grassWet
    (bernoulli (branch rain
                       (branch sprinkler 0.99 0.8)
                       (branch sprinkler 0.9 0.00001)))]
[OBSERVE grassWet True]
\end{verbatim}

Note that each of the Venture families in the PET corresponds to a node in the Bayesian network. A coarsened version of the PET contains the same conditional dependence and independence information as the Bayesian network would.

\todo[color=green]{[medium] Explain alternative representations of a Bayesian network, with figures, and discuss tradeoffs; include the actual Bayes net in the figure}

\begin{figure}
\includegraphics[width=6in]{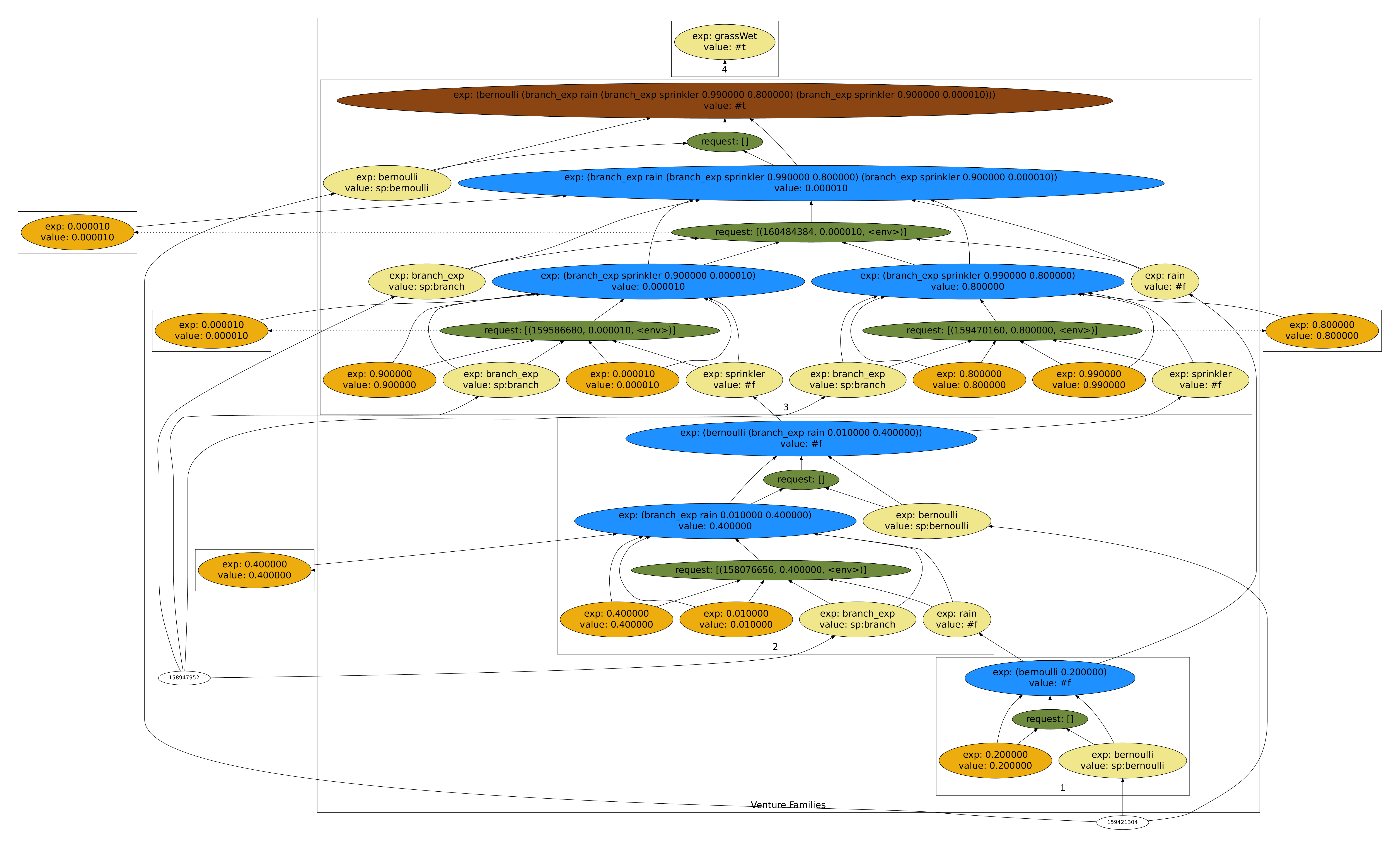}
\caption{{\bf The PET corresponding to a three-node Bayesian network.} Each of the three numbered families correspond to a node in the Bayesian network, capturing the execution history needed to simulate the node given its parents.}
\label{fig:sprinkler}
\end{figure}

\subsubsection{Stochastic memoization}

\todo[color=yellow]{[medium] Discuss+cite Church, mansinghka's thesis, nonparametrics; explain how this has been misinterpreted by Luc de Raedt (whom we should cite)}

\todo[color=red]{[medium] Discuss probabilistic logic programming and SRL}

We now illustrate a program that exhibits stochastic memoization. This
program constructs a stochastically memoized procedure and then
applies it three times. Unlike deterministic memoization, a
stochastically memoized procedure has a stochastic $\request$ PSP
which sometimes returns a previously sampled value, and sometimes
samples a fresh one. These random choices follow a Pitman-Yor
process. Figure~\ref{fig:pymem} shows a PET corresponding to a typical
simulation; note the overlapping requests.

\begin{verbatim}
[ASSUME f (pymem bernoulli 1.0 0.1)]
[PREDICT (f)]
[PREDICT (f)]
[PREDICT (f)]
\end{verbatim}

\begin{figure}
\includegraphics[width=6in]{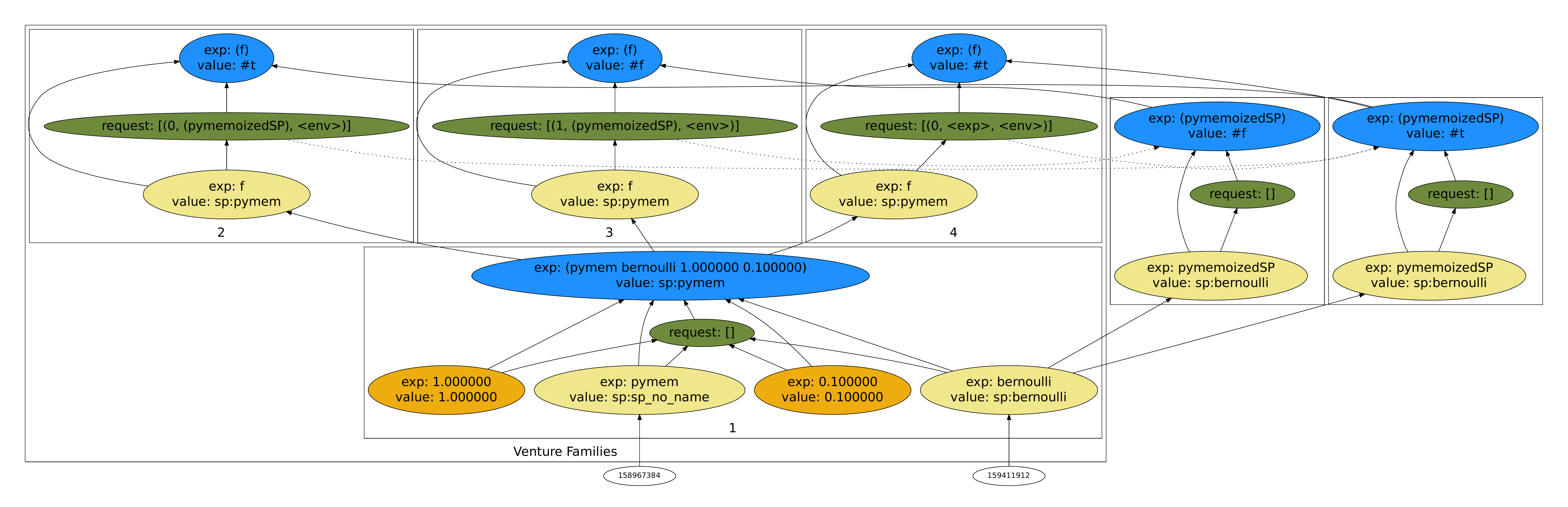}
\caption{{\bf A PET corresponding to an execution of a program with stochastic memoization.} The procedure being stochastically memoized, {\tt bernoulli}, is applied twice, based on the value of the requests arising in invocations of {\tt (f)}.}
\label{fig:pymem}
\end{figure}

\todo[color=green]{[small] de Raedt arxiv note on stochastic memorization: straighten out the record}

\subsection{Constructing PETs via forward simulation}

Let \( \prog \) be a constrained program. Venture's primary inference
strategies require an execution of \( \prog \) with positive
probability before they can even begin. Therefore the first thing we
do is simply evaluate the program by interpreting it in a fairly
standard way, except with all conditional evaluation handled uniformly
through the ESR machinery presented above. For simplicity, we elide details related to inference scoping.

\subsubsection{Pseudocode for EVAL, APPLY and EVAL-REQUESTS }

Evaluation is generally similar to a pure Scheme, but there are a few noteworthy differences. First, evaluation creates nodes for every recursive call, and connects them together to form the directed graph structure of the probabilistic execution trace. Second, there are no distinctions between primitive procedures and compound procedures. We \todo[color=yellow]{Rename EvalFamily to eval.}call the top-level evaluation procedure {\sc EvalFamily} to emphasize the family block structure in PETs. Third, a {\em scaffold} and a database of random choices {\em db} are threaded through the recursions, to support their use in inference, including restoring trace fragments when a transition is rejected by reusing random choices from the {\em db}. Note that an environment model evaluator  \citep{abelson1983structure} is used, even though the underlying language is pure, with the PET storing both the environment structure and the recursive invocations of {\tt eval}. 

\begin{codebox}
\Procname{$\proc{EvalFamily}(\id{trace},\id{exp},\id{env},\id{scaffold},\id{db})$}
\li \If \( \func{isSelfEvaluating?}(\id{exp}) \) 
\li \Then \Return \( (0,\attrib{trace}{\func{createConstantNode}}(\id{exp})) \)
\li \ElseIf \( \func{isQuoted?}(\id{exp}) \)
\li \Then \Return \( (0,\attrib{trace}{\func{createConstantNode}}(\func{TextOfQuotation}(\id{exp}))) \)
\li \ElseIf \( \func{isVariable?}(\id{exp}) \) 
\li \Then \( \id{sourceNode} = \attrib{env}{\func{findSymbol}}(\id{exp}) \)
\li \(\proc{Regenerate}(\id{trace},\id{exp},\id{scaffold},\const{False},\id{db})\)
\li \Return \( (0, \attrib{trace}{\func{createLookupNode}}(\id{sourceNode}) \)
\li \ElseNoIf
\li \( \id{weight}, \id{operatorNode} \gets \proc{EvalFamily}(\id{trace},\attrib{exp}{operator},\id{env},\id{scaffold},\id{db}) \)
\li \( \id{operandNodes} \gets [] \)
\li \For \( \id{operand} \) \In \( \attrib{exp}{operands} \)
\li \Do \( (w,\id{operandNode}) \gets \proc{EvalFamily}(\id{trace},\id{operand},\id{env},\id{scaffold},\id{db}) \)
\li \( \id{weight} \gets \id{weight} + w \)
\li \( \attrib{operandNodes}{\func{append}}(\id{operandNode}) \)
\li \End
\li \( (\id{requestNode},\id{outputNode}) \gets \attrib{trace}{\func{createApplicationNodes}}(\id{operatorNode},\id{operandNodes},\id{env}) \)
\li \( \id{weight} \gets \id{weight} + \proc{ApplySP}(\id{trace},\id{requestNode},\id{outputNode},\id{scaffold},\const{False},\id{db}) \)
\li \Return \( (\id{weight},\id{outputNode}) \)
 \End
\end{codebox}

The call to {\sc Regenerate} ensures that if {\sc EvalFamily} is called in the context of a pre-existing PET, it is traversed in an order that is compatible with the dependence structure of the program. We will sometimes refer to such orders as ``evaluation-consistent''. From the standpoint of forward simulation, however, {\sc Regenerate} can be safely ignored.

\begin{codebox}
\Procname{$\proc{ApplySP}(\id{trace},\id{requestNode},\id{outputNode},\id{scaffold},\id{restore?},\id{db})$}
\li \( \id{weight} \gets \proc{ApplyPSP}(\id{trace},\id{requestNode},\id{scaffold},\id{restore?},\id{db}) \)
\li \( \id{weight} \gets \id{weight} +
\proc{EvalRequests}(\id{trace},\id{requestNode},\id{scaffold},\id{restore?},\id{db})
\)
\li \( \id{weight} \gets \id{weight} + \proc{RegenerateESRParents}(\id{trace},\id{outputNode},\id{scaffold},\id{restore?},\id{db}) \)
\li \( \id{weight} \gets \id{weight} + \proc{ApplyPSP}(\id{trace},\id{outputNode},\id{scaffold},\id{restore?},\id{db}) \)
\li \Return \id{weight}
\end{codebox}

Stochastic procedures are allowed to request evaluations of expressions\footnote{To prevent this flexibility from introducing arbitrary dependencies, the expressions and environments are restricted to those constructible from the arguments to the procedure or the procedure that constructed it.}; this enables higher-order procedures as well as the encapsulation of custom control flow constructs. For example, compound procedures do this to evaluate their body in an environment extended with their formal parameters and argument values. 

\begin{codebox}
\Procname{\(\proc{EvalRequests}(\id{trace},\id{node},\id{scaffold},\id{restore?},\id{db})\)}
\li \( \id{sp},\id{spaux} \gets \attrib{trace}{\func{getSP}}(\id{node}),\attrib{trace}{\func{getSPAux}}(\id{node}) \)
\li \( \id{weight} \gets 0 \)
\li \( \id{request} \gets \attrib{trace}{\func{getValue}}(\id{node}) \)
\li \For \( \id{id},\id{exp},\id{env},\id{block},\id{subblock} \) \In \( \attrib{request}{esrs} \)
\li \Do \If \( \Not \attrib{spaux}{\func{containsFamily}}(\id{id}) \)
\li \Then 
\If \( \id{restore?} \)
\li \Then \( \id{esrParent} \gets \attrib{db}{\func{getESRParent}}(\id{sp},\id{id}) \)
\li \( \id{weight} \gets \id{weight} + \proc{RestoreFamily}(\id{trace},\id{esrParent},\id{scaffold},\id{db}) \)
\li \ElseNoIf
\li \( (\id{w},\id{esrParent}) \gets \proc{EvalFamily}(\id{trace},\id{exp},\id{env},\id{scaffold},\id{db}) \)
\li \( \id{weight} \gets \id{weight} + \id{w} \)
\End
\li \( \attrib{trace}{\func{registerFamily}}(\id{node},\id{id},\id{esrParent}) \)
\li \ElseNoIf
\li \( \id{weight} \gets \id{weight} + \proc{Regenerate}(\id{trace},\attrib{spaux}{\func{getFamily}}(\id{id}),\id{scaffold},\id{restore?},\id{db})\)
\End
\li \( \id{esrParent} \gets \attrib{spaux}{\func{getFamily}}(\id{id}) \)
\li \If \( \Not \id{block} \isequal \const{None} \)
\li \Then \( \attrib{trace}{\func{registerBlock}}(\id{block},\id{subblock},\id{esrParent}) \)
 \End
\li \( \attrib{trace}{\func{addESREdge}}(\id{esrParent},\attrib{node}{outputNode}) \)
\End
\li \For \( \id{lsr} \) \In \( \attrib{request}{lsrs} \)
\li \Do 
\If \attrib{db}{\func{hasLatentDBFor}}(\id{sp})
\li \Then \( \id{latentDB} \gets \attrib{db}{\func{getLatentDB}}(\id{sp}) \)
\li \ElseNoIf
\li \( \id{latentDB} \gets \const{None} \)
\End
\li \( \id{weight} \gets \id{weight} + \attrib{sp}{\func{simulateLatents}}(\id{spaux},\id{lsr},\id{restore?},\id{latentDB}) \)
\End
\li \Return \id{weight}
\end{codebox}

Stochastic procedures are also allowed to perform opaque operations to produce outputs given the value of their arguments and any requested evaluations. Note that this may result in random choices being added to the record maintained by the trace.

\begin{codebox}
\Procname{$\proc{ApplyPSP}(\id{trace},\id{node},\id{scaffold},\id{restore?},\id{db})$}
\li \( \id{psp},\id{args} \gets \attrib{trace}{\func{getPSP}}(\id{node}), \attrib{trace}{\func{getArgs}}(\id{node}) \)
\li \If \( \attrib{db}{\func{hasValueFor}}(\id{node}) \)
\li \Then \( \id{oldValue} \gets \attrib{db}{\func{getValue}}(node) \)
\li \ElseNoIf
\li \( \id{oldValue} \gets \const{None} \)
\End
\li \Comment{Determine new value}
\li \If \( \id{shouldRestore} \)
\li \Then \( \id{newValue} \gets \id{oldValue} \)
\li \ElseIf \( \attrib{\id{scaffold}}{\func{hasKernelFor}}(\id{node}) \)
\li \Then \( \id{newValue} \gets \attrib{\attrib{\id{scaffold}}{\func{getKernel}}(\id{node})}{\func{simulate}}(\id{trace},\id{oldValue},\id{args}) \)
\li \ElseNoIf
\li \( \id{newValue} \gets \attrib{psp}{\func{simulate}}(\id{args}) \)
\End
\li \Comment{Determine weight}
\li \If \( \attrib{\id{scaffold}}{\func{hasKernelFor}}(\id{node}) \)
\li \Then \( \id{weight} \gets \attrib{\attrib{\id{scaffold}}{\func{getKernel}}(\id{node})}{\func{weight}}(\id{trace},\id{newValue},\id{oldValue},\id{args}) \)
\li \ElseNoIf 
\li \( \id{weight} \gets \const{0} \)
\End
\li \( \attrib{trace}{\func{setValue}}(\id{node},\id{newValue}) \)
\li \( \attrib{psp}{\func{incorporate}}(\id{newValue},\id{args}) \)

\li \If \( \func{isSP?}(\id{newValue}) \)
\li \Then \( \proc{processMadeSP}(\id{trace},\id{node},\attrib{\id{scaffold}}{\func{isAAA?}}(\id{node})) \)
\End
\li \If \( \attrib{psp}{\func{isRandom?}}() \)
\li \Then \( \attrib{trace}{\func{registerRandomChoice}}(\id{node}) \)
\End
\li \Return \( \id{weight} \)
\end{codebox}

This functionality is supported by book-keeping to link new stochastic procedures into the trace, and register any customizations they implement with the trace so that inference transitions can make use of them:

\begin{codebox}
\Procname{\( \proc{ProcessMadeSP}(\id{trace},\id{node},\id{isAAA?}) \)}
\li \( \id{sp} \gets \attrib{trace}{\func{getValue}}(\id{node}) \)
\li \( \attrib{trace}{\func{setMadeSP}}(\id{node},\id{sp}) \)
\li \( \attrib{trace}{\func{setValue}}(\id{node},\func{spref}(\id{sp})) \)
\li \If \Not \( \id{isAAA?} \)
\li \Then 
\( \attrib{trace}{\func{setMadeSPAux}}(\id{node},\attrib{sp}{\func{constructSPAux}}()) \)
\li \If \( \attrib{sp}{\func{hasAEKernel}}() \)
\li \Then 
\( \attrib{trace}{\func{registerAEKernel}}(\id{node}) \)
\End
\End
\end{codebox}

\subsection{Undoing simulation of PET fragments}

Venture also includes requires unevaluation procedures that are dual to the evaluation procedures described earlier. This is because PET fragments need to be removed from the trace in two situations. First, when \verb|FORGET| instructions are triggered, the expression corresponding to the directive is removed. Second, during inference, changes to the values of certain requestPSPs make cause some SP families to no longer be requested. In the first case, all of the random choices are permanently removed, whereas in the second case, the random choices may need to be restored if the proposal is rejected.

\subsubsection{Pseudocode for UNEVAL, UNAPPLY and UNEVAL-REQUESTS}

When a trace fragment is unevaluated, we must visit all application nodes in the trace fragment so that the PSPs have a chance to unincorporate the (input,output) pairs. The operations needed to do this are essentially inverses of the simulation procedures described above, designed to visit nodes in the reverse order that evaluation does, to ensure compatibility with exchangeable coupling. Here we give pseudocode for these operations, eliding the details of garbage collection\footnote{Previous work has anecdotally explored the possibility of preserving unused trace fragments and treating them as auxiliary variables, following the treatment of component model parameters from Algorithm 8 in \citep{Neal:1998wz}. In theory, this delay of garbage collection, where multiple copies of each trace fragment are maintained for each branch point, could support adaptation to the posterior. A detailed empirical evaluation of these strategies is pending a comprehensive benchmark suite for Venture as well as a high-performance implementation.}:

\begin{codebox}
\Procname{\(\proc{UnevalFamily}(\id{trace},\id{node},\id{scaffold},\id{db})\)}
\li \If \( \attrib{node}{isConstantNode?} \) 
\li \Then \( \id{weight} \gets \const{0} \)
\li \ElseIf \( \attrib{node}{isLookupNode?} \)
\li \Then 
\( \attrib{trace}{\func{disconnectLookup}}(\id{node}) \)
\li \(\id{weight} \gets \proc{Extract}(\id{trace},\attrib{node}{sourceNode},\id{scaffold},\id{db})\)
\li \ElseNoIf
\li \( \id{weight} \gets \proc{UnapplySP}(\id{trace},\id{node},\id{scaffold},\const{False},\id{db}) \)
\li \For \( \id{operandNode} \) \In \( \attrib{node}{operandNodes} \)
\li \Do \( \id{weight} \gets \id{weight} + \proc{UnevalFamily}(\id{trace},\id{operandNode},\id{scaffold},\id{db}) \)
 \End
\li \( \id{weight} \gets \id{weight} + \proc{UnevalFamily}(\id{trace},\attrib{node}{operatorNode},\id{scaffold},\id{db}) \)
\End
\li \Return \( \id{weight} \)
\end{codebox}

To unapply a stochastic procedure, Venture must undo its random choices and also unapply any requests it generated:

\begin{codebox}
\Procname{$\proc{UnapplySP}(\id{trace},\id{node},\id{scaffold},\id{db})$}
\li \( \id{weight} \gets
\proc{UnapplyPSP}(\id{trace},\id{node},\id{scaffold},\id{db}) \)
\li \( \id{weight} \gets \id{weight} + \proc{extractESRParents}(\id{trace},\id{outputNode},\id{scaffold},\id{db}) \)
\li \( \id{weight} \gets \id{weight} + \proc{UnevalRequests}(\id{trace},\attrib{node}{requestNode},\id{scaffold},\id{db}) \)
\li \( \id{weight} \gets \id{weight} + \proc{UnapplyPSP}(\id{trace},\attrib{node}{requestNode},\id{scaffold},\id{db}) \)
\li \Return \id{weight}
\end{codebox}

To unapply a primitive stochastic procedure, we remove its random choices from the trace, unincorporate them, and update the weight accordingly. Note that we store the value in the {\em db} so that we can restore it later if necessary.

\begin{codebox}
\Procname{$\proc{UnapplyPSP}(\id{trace},\id{node},\id{scaffold},\id{db})$}
\li \( \id{psp},\id{args} \gets \attrib{trace}{\func{getPSP}}(\id{node}), \attrib{trace}{\func{getArgs}}(\id{node}) \)
\li \If \( \attrib{psp}{\func{isRandom}}() \)
\li \Then \( \attrib{trace}{\func{unregisterRandomChoice}}(\id{node}) \)
\End
\li \If \( \func{isSP}(\attrib{trace}{\func{getValue}}(\id{node})) \)
\li \Then \( \proc{teardownMadeSP}(\id{trace},\id{node},\attrib{\id{scaffold}}{\func{isAAA}}(\id{node})) \)
\End
\li \( \id{oldValue} \gets \attrib{trace}{\func{getValue}}(\id{node}) \)
\li \( \attrib{psp}{\func{unincorporate}}(\id{oldValue},\id{args}) \)
\li \If \( \attrib{\id{scaffold}}{\func{hasKernelFor}}(\id{node}) \)
\li \Then
\( \id{weight} \gets \attrib{\attrib{\id{scaffold}}{\func{getKernel}}(\id{node})}{\func{reverseWeight}}(\id{trace},\id{oldValue},\id{args}) \)
\li \ElseNoIf 
\li  \( \id{weight} \gets \const{0} \)
\End

\li \( \attrib{db}{\func{extractValue}}(\id{node},\id{oldValue}) \)
\li \( \attrib{trace}{\func{clearValue}}(\id{node}) \)
\li \Return \( \id{weight} \)
\end{codebox}

Handling of requests during unevaluation involves two additional subtleties. First, latent random choices must be handled appropriately. Second, requests must be unevaluated only if no other application of the SP refers to them.

\begin{codebox}
\Procname{\(\proc{UnevalRequests}(\id{trace},\id{node},\id{scaffold},\id{db})\)}
\li \( \id{sp},\id{spaux} \gets \attrib{trace}{\func{getSP}}(\id{node}),\attrib{trace}{\func{getSPAux}}(\id{node}) \)
\li \( \id{weight} \gets 0 \)
\li \( \id{request} \gets \attrib{trace}{\func{getValue}}(\id{node}) \)

\li \If \( \attrib{request}{lsrs} \) \AndNot \( \attrib{db}{\func{hasLatentDB}}(\attrib{trace}{\func{getSP}}(\id{node})) \)
\li \Then \( \attrib{db}{\func{registerLatentDB}}(\id{sp},\attrib{sp}{\func{constructLatentDB}}()) \)
\End

\li \For \( \id{lsr} \) \In \( \func{reversed}(\attrib{request}{lsrs}) \)
\li \Do 
\( \id{weight} \gets \id{weight} + \attrib{sp}{\func{detachLatents}}(\id{spaux},\id{lsr},\attrib{db}{\func{getLatentDB}}(\id{sp})) \)
\End

\li \For \( \id{id},\id{exp},\id{env} \) \In \( \func{reversed}(\attrib{request}{esrs}) \)
\li \Do 
\( \id{esrParent} \gets \attrib{trace}{\func{popLastESRParent}}(\attrib{node}{outputNode}) \)
\li \If \( \attrib{trace}{getNumberOfRequests}(\id{esrParent}) \isequal \const{0} \)
\li \Then
 \( \attrib{trace}{\func{unregisterFamily}}(\id{node},\id{id}) \)
\li \( \attrib{db}{\func{registerFamily}}(\id{sp},\id{id},\id{esrParent}) \)
\li \( \id{weight} \gets \id{weight} +  \proc{UnevalFamily}(\id{trace},\id{esrParent},\id{scaffold},\id{db}) \)
\End
\End

\li \Return \id{weight}
\end{codebox}

When an SP is finally destroyed --- i.e. when the output PSP of its maker SP application is unapplied --- then the SP can be destroyed and its auxiliary storage can be garbage collected. Because unapplication of the maker must happen after all applications of the made SP have been unapplied --- due to the constraint that unevaluation visits nodes in the reverse order of evaluation and regeneration --- there is no information in the auxiliary storage that needs to be preserved in the latent DB.

\begin{codebox}
\Procname{\( \proc{TeardownMadeSP}(\id{trace},\id{node},\id{isAAA?}) \)}
\li \( \id{sp} \gets \attrib{trace}{\func{getMadeSP}}(\id{node}) \)
\li \( \attrib{trace}{\func{setValue}}(\id{node},\id{sp}) \)
\li \( \attrib{trace}{\func{clearMadeSP}}(\id{node}) \)
\li \If \Not \( \id{isAAA?} \)
\li \Then 
 \If \( \attrib{sp}{\func{hasAEKernel}}() \)
\li \Then 
\( \attrib{trace}{\func{unregisterAEKernel}}(\id{node}) \)
\End
\li \( \attrib{trace}{\func{clearMadeSPAux}}(\id{node}) \)

\End
\end{codebox}

\subsection{Enforcing constraints via CONSTRAIN and UNCONSTRAIN}

Unfortunately, the observations may not have positive probability
given the program execution. Venture may not happen to sample the
right data, and but if noise is added to each data generating
expression to ensure positive probability, it may take a long time for
Venture to incorporate the observed data at all. On the other hand,
the general problem of finding a program execution for which some
output has positive probability is intractable in general -- it can
encode \textbf{SAT} without even making use of contingent evaluation
or special primitives -- and Venture is not designed to attempt to
invert such programs.

Venture implements a middle ground. It is designed with sufficiently
stochastic probabilistic programs in mind, and in particular is not a
\textbf{SAT}-solver. However, in order to support common cases from
Bayesian statistics, we do support inverting a simple kind of
determinism. We introduce a method
\verb|constrain(<directive>,<value>)| that recursively walks backwards
along \( id \dash \on \dash \llia \) edges to find the outermost
application of a different kind of PSP. If the constrained value has
positive probability at this PSP application, we constrain the value
and we are done. Otherwise, we unevaluate the entire program and
evaluate it again, in hopes that this execution will be constrainable.

\todo[color=yellow]{[medium] Write up final version of constrain, with rejection subtleties, discussing ergodicity via global scaffold vs rejection, and include in CLRS}

Simple versions of constrain and unconstrain can be implemented as follows:

\begin{codebox}
\Procname{\(\proc{Constrain}(\id{trace},\id{node},\id{value})\)}
\li \If \( \attrib{node}{isLookupNode?} \)
\li \Then \Return  \( \proc{Constrain}(\id{trace},\attrib{node}{sourceNode},\id{value}) \)
\li \ElseIf \( \attribb{trace}{\func{getPSP}(\id{node})}{isESRReference?} \)
\li \Then \Return  \( \proc{Constrain}(\id{trace},\attrib{trace}{\func{getESRParents}}(\id{node})[0],\id{value}) \)
\li \ElseNoIf
\li \( \id{psp},\id{args} \gets \attrib{trace}{\func{getPSP}}(\id{node}), \attrib{trace}{\func{getArgs}}(\id{node}) \)
\li \( \attrib{psp}{\func{unincorporate}}(\attrib{trace}{\func{getValue}}(\id{node}),\id{args}) \)
\li \( \id{weight} \gets \attrib{psp}{\func{logDensity}}(\id{value},\id{args}) \)
\li \( \attrib{trace}{\func{setValue}}(\id{node},\id{value}) \)
\li \( \attrib{psp}{\func{incorporate}}(\id{value},\id{args}) \)
\li \If \( \attrib{psp}{isRandom?} \) 
\li \Then \( \attrib{trace}{\func{registerConstrainedChoice}}(\id{node}) \)
\End
\li \Return \( \id{weight} \)
\End
\end{codebox}

\begin{codebox}
\Procname{\(\proc{Unconstrain}(\id{trace},\id{node})\)}
\li \If \( \attrib{node}{isLookupNode?} \)
\li \Then \Return  \( \proc{Unconstrain}(\id{trace},\attrib{node}{sourceNode}) \)
\li \ElseIf \( \attribb{trace}{\func{getPSP}(\id{node})}{isESRReference?} \)
\li \Then \Return  \( \proc{Unconstrain}(\id{trace},\attrib{trace}{\func{getESRParents}}(\id{node})[0]) \)
\li \ElseNoIf
\li \( \id{oldValue} \gets \attrib{trace}{\func{getValue}}(\id{node}) \)
\li \( \id{psp},\id{args} \gets \attrib{trace}{\func{getPSP}}(\id{node}), \attrib{trace}{\func{getArgs}}(\id{node}) \)
\li \If \( \attrib{psp}{isRandom?} \) 
\li \Then \( \attrib{trace}{\func{unregisterConstrainedChoice}}(\id{node}) \)
\End
\li \( \attrib{psp}{\func{unincorporate}}(\id{oldValue},\id{args}) \)
\li \( \id{weight} \gets \attrib{psp}{\func{logDensity}}(\id{oldValue},\id{args}) \)
\li \( \attrib{psp}{\func{incorporate}}(\id{oldValue},\id{args}) \)
\li \Return \( \id{weight} \)
\End
\end{codebox}

It is instructive to consider the following program:

\begin{verbatim}
[ASSUME x_1 (bernoulli 0.5)]
[ASSUME x_2 (bernoulli 0.5)]
...
[ASSUME x_N (bernoulli 0.5)]
[OBSERVE (xor x_1 ... x_N) True]
\end{verbatim}

If \verb|xor()| happens to be true by chance ---
i.e. \verb|constrain()| did not have to adjust any random choices ---
then the PET that was sampled is already drawn from the true
conditioned distribution. However, the probability of this decreases
rapidly in \verb|N| and because \verb|xor()| is deterministic,
repeated rejection may appear to be the only option. To avoid this
problem, we frequently restrict ourselves to \verb|OBSERVE|s whose
expressions have the form \verb|(<some-stochastic-SP> ...)|, i.e. they
are applications of some fixed stochastic procedure. Then we can
always successfully initialize and begin inference. Venture was
designed with such ``sufficiently stochastic'' probabilistic programs
in mind; see \citep{Anonymous:T2b7vHXI} for some
rationale. We also note that to ensure ergodicity, periodic
whole-trace proposals that attempt to restart from the prior (via the
\verb|INFER| directive) are sufficient, although this kind of
independence sampler is unlikely to be efficient.

We note that without additional restrictions on observations, the
presence of \verb|constrain()| could introduce significant
complexity. Constraining a node \( X \) that is used in multiple
places would not necessarily give us an execution with positive
probability, since the probability of all children of \( X \) will now
have changed.

To resolve this and to simplify our algorithms, we make the following
strong assumption on Venture programs:

\emph{For a Venture program to be valid, it must guarantee that if a node gets
  constrained during inference (as opposed to from a call to observe), every
  node along every outgoing directed path must be either a deterministic output
  PSP, a null-request PSP, or a lookup node. Moreover, any observed nodes on
  this path must agree on the observed value.}

In addition to the code given above, our implementation of constrain also propagates the
constrained value along these outgoing deterministic paths, but the details of
the implementation are beyond the scope of this paper.

\section{Partitioning Traces into Scaffolds for Scalable, Incremental Inference}

Most scalable inference algorithms for Bayesian networks take advantage of the network's representation of conditional dependence. We have seen that probabilistic execution traces make analogous scalability gains possible for probabilistic programs written in Turing-complete, higher-order languages, by representing not just conditional dependences but also existential dependencies and exchangeable coupling. 

However, PETs and the SPI also introduce complexities that can be avoided in the setting of Bayesian networks. For example, in a Bayesian network, the random variables that could possibly directly depend on a given random variable are easy to identify: they are the immediate children of the node according to the direct graph structure of the network. In a PET, the value of a random choice may change which children exist. Additionally, in Bayesian networks, all nodes are equipped with conditional probability densities, but SPs that lack conditional density functions are permitted in Venture. Joint densities for chains of random choices in Venture may not be available, let alone the conditional density of a choice given its immediate descendants.

Here we describe {\em scaffolds}, the mechanism used in Venture to handle these complexities. Scaffolds carve out coherent subproblems of inference in the context of a global inference problem analogously to the inference subproblems in a Bayesian network that arise by conditioning on a subset of the nodes and querying the nodes contained within it.

\subsection{Motivation and Notation}

Let \( \rho \) be an execution and let \( \mathbf{X} =  \{ X \} \) denote a set of nodes that
we wish to propose new values for as a block. If we only propose new values to these
nodes, we may end up with a trace that has probability 0. For example, consider the
following program:

\begin{verbatim}
[ASSUME x (normal 0 1)]
[PREDICT (+ x 1)]
\end{verbatim}

If we only propose a change to the value of \( x \) from \( 1 \) to \( 2 \), and do not propagate the
change through the application of \( + \), then \( + \) will report probability
0 of sampling its output \( 2 \) given its inputs \( 1 \) and \( 2
\). 

Similarly, in the program

\begin{verbatim}
[ASSUME x (normal 0 1)]
[PREDICT (if (< x 0) (normal -10 1) (normal 10 1))]
\end{verbatim}

if we only propose a new value for \( x \) and its sign flips, then the requester for IF will
report probability 0 of its old request given its new inputs. 

SPs that cannot report their likelihoods introduce a similar problem. In the program

\begin{verbatim}
[ASSUME x (normal 0 1)]
[PREDICT (likelihood-free-sp x)]
\end{verbatim}

if we only propose a new value for \( x \), then we cannot compute the
Metropolis-Hastings acceptance ratio since we have no way to account for the
probability of the likelihood-free-sp's old value given its new arguments.

In all three cases, the problem can be solved by proposing new values to other nodes
downstream of the nodes in \( \mathbf{X} \). However, we do not want to
resimulate the rest of the program just to make a single proposal. We
want to find a middleground between rejecting at \( + \) and running the entire
program. As a default, we propose new values to downstream nodes until we reach
applications of stochastic PSPs for which we can definitely compute the
logdensity of its original output given any new values for its inputs. We say
that these nodes ``absorb the flow of change'' and so we call them
\emph{absorbing nodes}.

Note that in the case of the IF statement that may switch from one branch to
another, the nodes in the old branch may no longer exist in the proposed trace,
and other nodes may come into existence.

\todo[color=yellow]{[small] Describe mapping from scaffold to overall Venture program}

To simplify notation, we will not explicitly condition on the parents of a
random variable in our derivations. For example, in the Bayesian network \( A
\to B \), we will write \(P(A) \) for \( P(A) \) and \( P(B) \) for \( P(B | A)
\). Although this would be unacceptable in the case of Bayesian networks where
we marginalize and condition arbitrarily, in our discussion of probabilistic
programming we will only ever compute probabilities of nodes given their
parents, so that this notation will not be problematic. 

\subsection{Partitioning traces and defining scaffolds}

As before, let \( \rho \) be an execution and let \( \mathbf{X} =  \{ X \} \) denote a set of nodes that
we wish to propose new values for as a block. We call \( \mathbf{X} \) the
\emph{principal nodes} of the proposal. Let \( \xi \) denote the proposed
trace. We introduce the following definitions:

\begin{enumerate}
\item \( \DG = \DG(\rho,\mathbf{X}) \): the nodes that will definitely still exist in
  \( \xi \), whose values might change. This includes \( \mathbf{X} \), as well
  as any nodes at which we do not want to absorb at (e.g. deterministic nodes)
  or are unable to absorb at (e.g. applications of likelihood-free SPs). Note
  that we have \( \DG(\rho) = \DG(\xi) = \DG \) by symmetry.

\item \( \brush(\rho) = \brush(\rho,\mathbf{X}) \): the nodes that may no longer be requested if \( \DG \)
  is resampled. This includes any nodes in branches that may be abandoned, or
  any requests whose operator may change.

\item 
\( \RG(\rho) = \RG(\rho,\mathbf{X}) =  \DG(\rho,\mathbf{X}) \cup \brush(\rho,\mathbf{X}) \): the regenerated subset of \(
\rho \). These are all the nodes that may either change value or no longer exist
in \( \xi \). Conversely \( \RG(\xi) \) consists of precisely the
set of new values we are proposing for \( \xi \).

\item
\( \AG = \AG(\rho,\mathbf{X}) \): the children of \( \DG \), or equivalently, the ``absorbing''
nodes. For each absorbing, the differences between the logdensity of the output
given the new value and the logdensity of the output of the old value appears in
the Metropolis-Hastings acceptance ratio.

\item \( \torus(\rho,\mathbf{X}) = \rho \setminus \RG(\rho) \): the nodes in \( \rho
  \) that are guaranteed to still exist in \( \xi \), and whose values are
  guaranteed not to change. By construction we have \( \torus(\rho) =
  \torus(\xi) = \torus \).

\item 
\( \PG(\rho) = \PG(\rho,\mathbf{X}) \): all parents of all nodes in \( \RG(\rho) \cup \AG \), excluding any
nodes in those sets. These are the nodes whose values would have been required to
simulate \( \RG(\rho) \) from \( \torus \) and to calculate the new log
densities at \( \AG \), but which definitely exist and whose values cannot change.

\item
\( \IG(\rho)  \): the set of all nodes that would never be referenced at all
while regenerating \( \RG(\rho) \).

\item
\( \OG = \PG(\rho) \cup \IG(\rho) = \PG(\xi) \cup \IG(\xi) \). Note that we may
not have \( \PG(\rho) = \PG(\xi) \), since \( \RG(\rho) \) and \( \RG(\xi) \)
may lookup different symbols in different environments, and may also request
different simulations. 

\end{enumerate}

These definitions give us the following partitions:

\begin{align}
\rho &= \DG \cup \RG(\rho) \cup \AG \cup \OG \\
\xi &= \DG \cup \RG(\xi) \cup \AG \cup \OG 
\end{align}

We refer to the pair \( ( \DG, \AG) \) as the \( \scaffold \) of the
set \( \mathbf{X} \) of principal nodes. As we will
see, the scaffold is the fundamental entity in the inference methods we will
consider. It induces a set
\( \Xi = \Xi[\OG,\scaffold] \) of executions that can be reached by resimulating
along the scaffold, consulting parents in \( \OG \) as needed. Moreover, for
any \( \xi \in \Xi \), the scaffold \( \scaffold \) will yield the same set \(
\OG \) in the factorization \( \xi = \scaffold \cup \brush(\xi) \cup \OG \).

\begin{figure}
\begin{center}
\includegraphics[width=200pt,height=200pt]{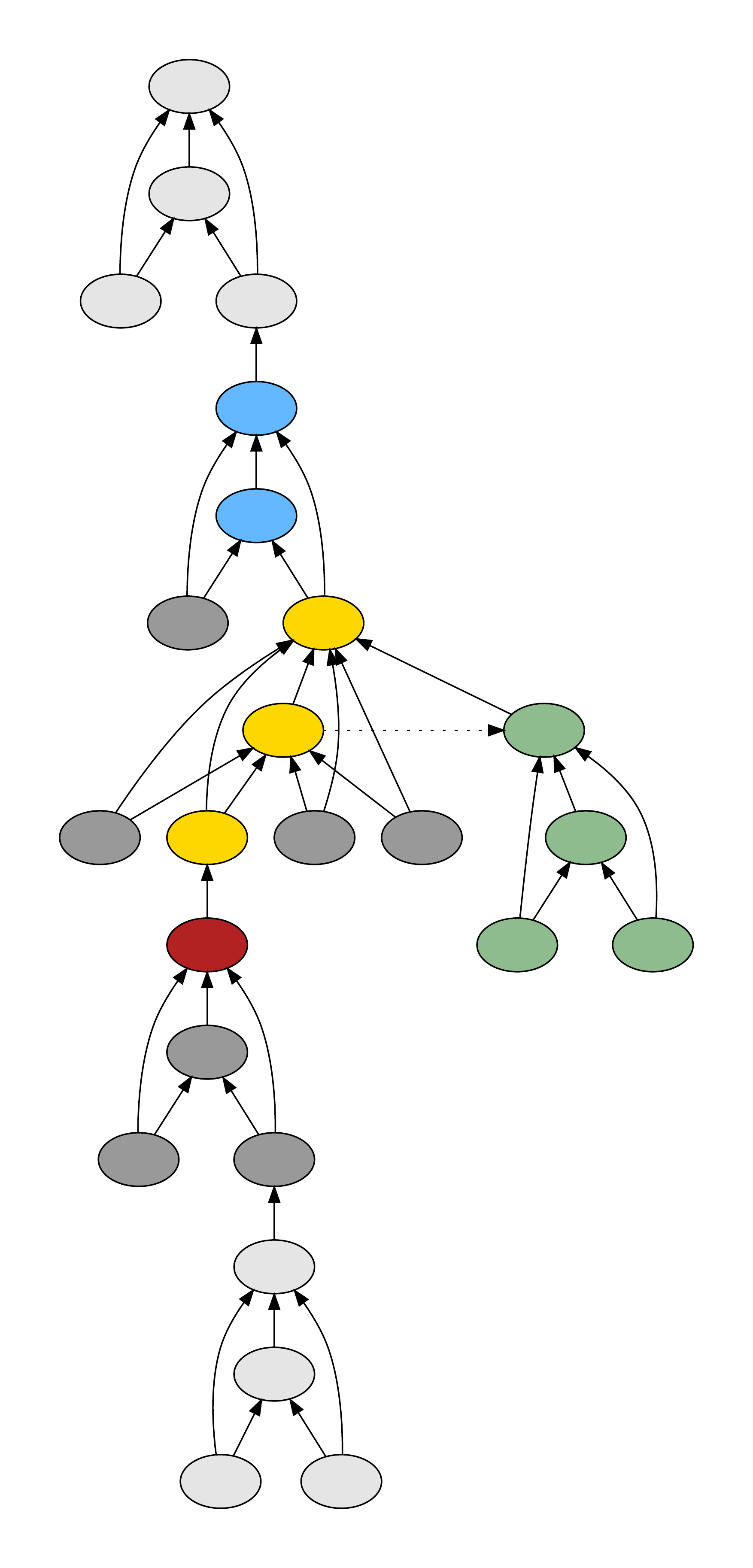}
\end{center}
\caption{The scaffold partitions the trace into five groups: the gold nodes that
  will definitely still exist in the proposal trace but whose values may change
  (drg), the blue nodes at which we will definitely compute likelihoods
  (absorbing), the green nodes that may no longer exist (the brush), the dark
  grey nodes that are parents of nodes in these three groups (parents), and all
  other light grey nodes that need never be visited at all (ignored). Note
  that future graphs will not distinguish between these last two groups.} 
\label{fig:partition_scaffold}
\end{figure}

\subsection{Constructing the scaffold}

\todo[color=yellow]{[small] Revise the voice in the scaffold construction section to match the tone of the rest of the paper.}

Given a trace \( \rho \) and a set of principal nodes \( \mathbf{X}
\), we construct the scaffold \( \DG(\rho,\mathbf{X}) \) as
follows. First, we walk downstream from each of the principal nodes
depth-first until every path either terminates or reaches the
application of a stochastic PSP for which we are certain we can
compute a logdensity. As we go along, we mark nodes as \( \DG \) or \(
\AG \) as appropriate. However, some nodes that we marked during this
process may actually be in the brush. Therefore our second step is to
recursively ``disable'' all requests that may change value or that may
no longer exist, and to mark as brush every node that is only
requested by these disabled requests. We then remove from \( \DG \)
and \( \AG \) every node in the brush.
Figure~\ref{fig:scaffold_construction} gives an overview of this
process, while Figure~\ref{fig:prior_scaffold} illustrates the need
for removing the brush of the source trace $\rho$ before making a
proposal.

\begin{figure}
(a)
\includegraphics[width=200pt,height=200pt]{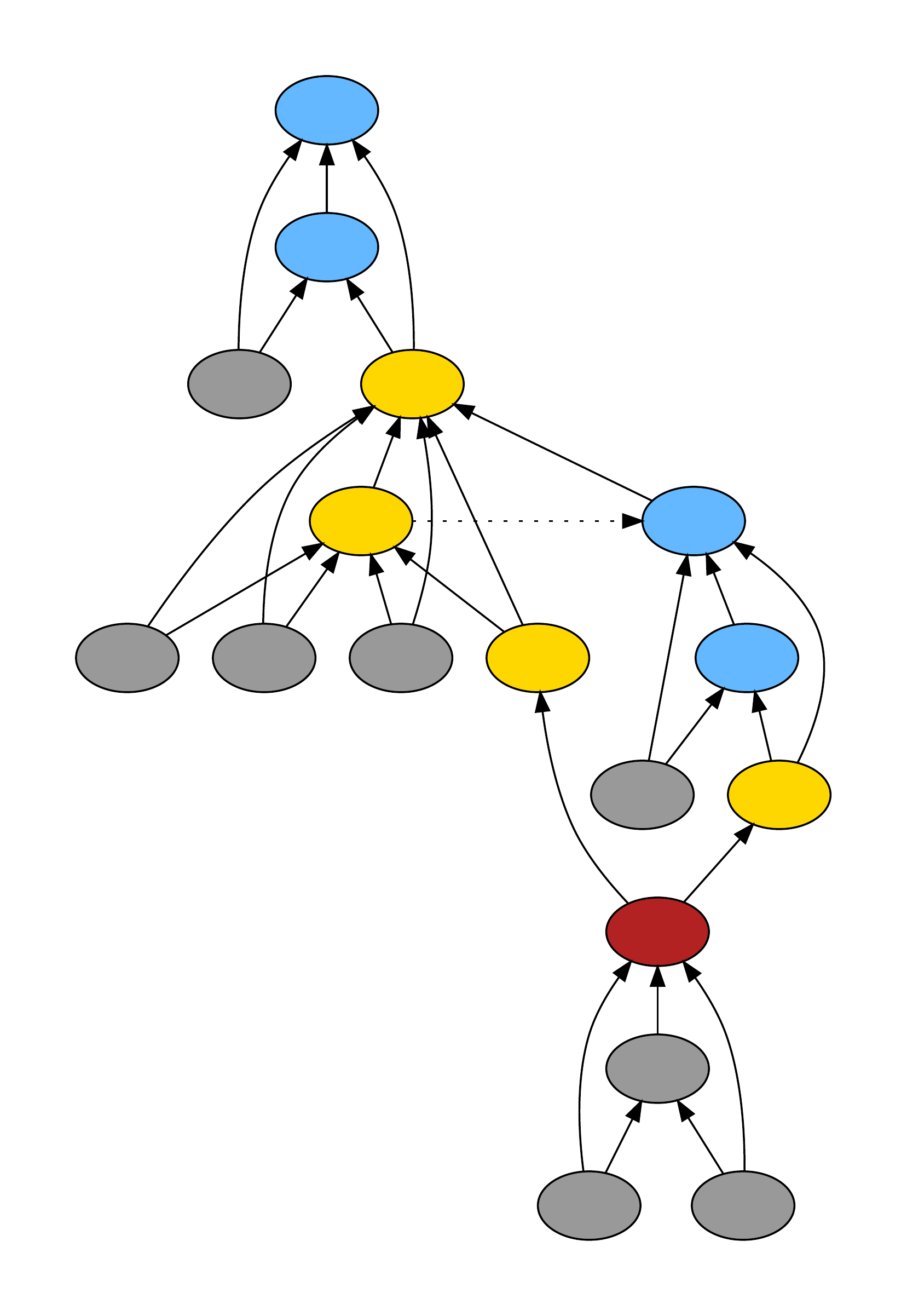}
(b)
\includegraphics[width=200pt,height=200pt]{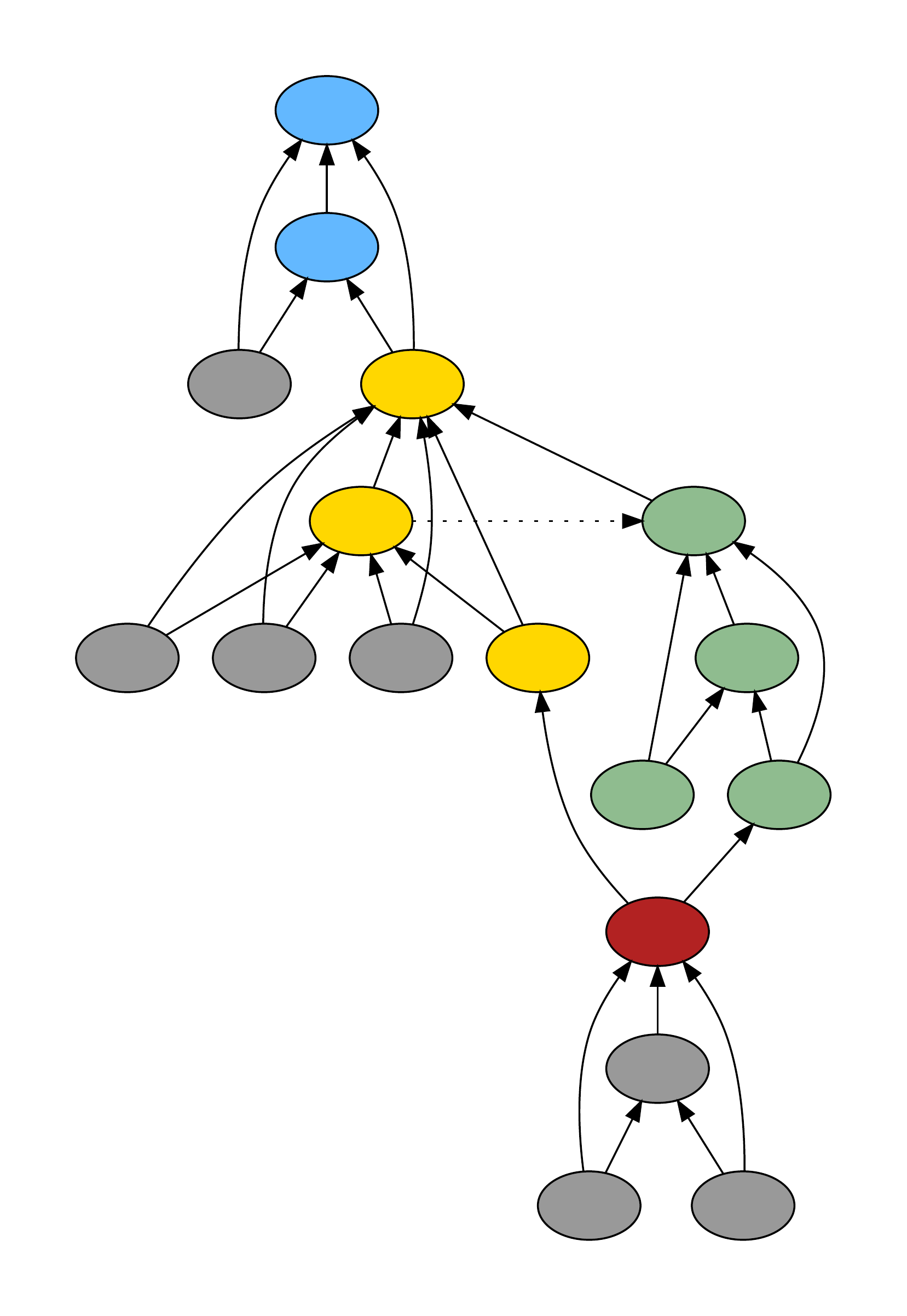}
\caption{{\bf Scaffold construction.} The principal node is shown in red, resampling nodes are shown in gold, absorbing nodes are shown in blue, and the brush is shown in green. (a) To construct a scaffold, Venture first walks downstream from the (red) principal node, identifying all the nodes whose values could possibly change, stopping at nodes that are guaranteed to be able to absorb the change (albeit perhaps with a probability density of 0). (b) Venture then identifies the brush --- those nodes that may no longer exist depending on the values chosen for the nodes being resampled --- using a separate recursion. Once the brush has been identified, the definite regeneration graph (gold) and the border (blue) are straightforward to identify. See the main text for additional details.}
\label{fig:scaffold_construction}
\end{figure}

\begin{figure}[ht]
\includegraphics[width=\textwidth]{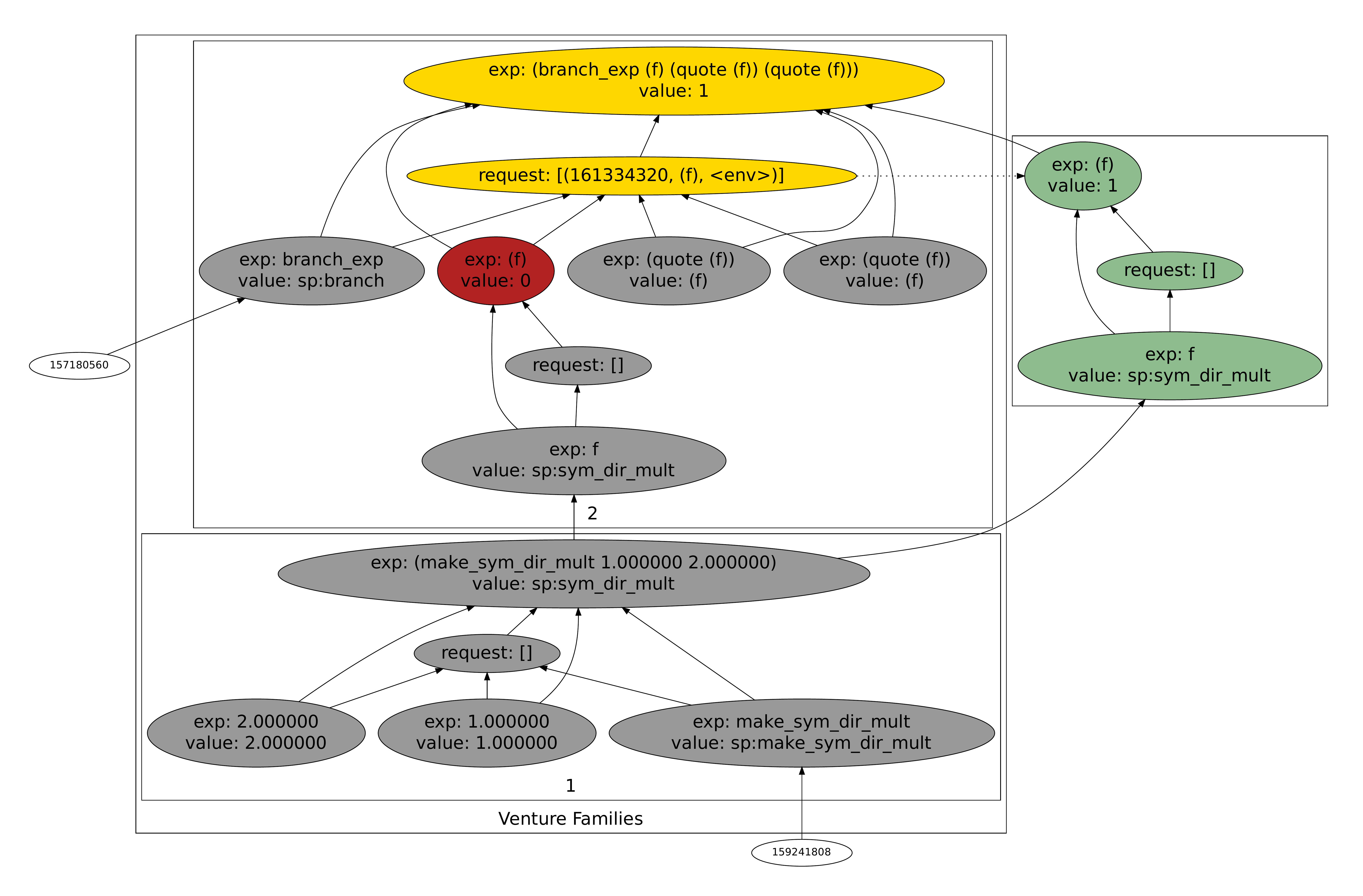}
\caption{{\bf The need for detaching the brush before making a proposal.} This trace includes the invocation of {\tt make\_symmetric\_dirichlet\_discrete}, an SP that produces SPs whose applications are exchangeably coupled. In this trace, the result of one application determines whether or not another application is made. Thus the second application is in the brush of the scaffold generated by the first application. If we do not detach the brush first, we may propose to the first application conditioned on an application that will not exist in the new trace.}
\label{fig:prior_scaffold}
\end{figure}

\subsection{Pseudocode for constructing scaffolds}

The overall construction procedure for scaffolds involves several passes, all linear in the size of the scaffold:

\begin{codebox}
\Procname{\(\proc{ConstructScaffold}(\id{trace},\id{principalNodes},\id{kernelInfo})\)}
\li \( \id{cDRG},\id{cAbsorbing},\id{cAAA} \gets \proc{FindCandidateScaffold}(\id{trace},\id{principalNodes}) \)
\li \( \id{brush} \gets \proc{FindBrush}(\id{trace},\id{cDRG}, \id{cAbsorbing},\id{cAAA}) \)
\li \( \id{drg},\id{absorbing},\id{aaa} \gets \func{RemoveBrush}(\id{cDRG},\id{cAbsorbing},\id{cAAA},\id{brush}) \)
\li \( \id{border} \gets \proc{FindBorder}(\id{drg},\id{absorbing},\id{aaa}) \)
\li \( \id{regenCounts} \gets \proc{ComputeRegenCounts}(\id{trace},\id{drg},\id{absorbing},\id{aaa},\id{border},\id{brush}) \)
\li \( \id{lkernels} \gets \func{loadKernels}(\id{trace},\id{drg},\id{kernelInfo}) \)
\li \Return \( \func{Scaffold}(\id{drg},\id{absorbing},\id{aaa},\id{border},\id{regenCounts},\id{lkernels}) \)
\end{codebox}

Finding a candidate scaffold given a set of principal nodes involves walking downstream from each principal node until it is possible to absorb. To find the brush, Venture must count the number of times a given requested family is reachable from the DRG. If this equals the number of times the family is requested, then all references to the family originate in the DRG, and it is therefore possible for the family's existence to depend on the values of the DRG nodes.

\todo[color=yellow]{[medium] Retypeset scaffold utils from lite; check scope/block details}

\begin{codebox}
\Procname{\(\proc{FindBrush}(\id{trace},\id{cDRG},\id{cAbsorbing},\id{cAAA})\)}
\li \( \id{disableCounts} \gets \func{map}(0) \)
\li \( \id{disabledRequestNodes} \gets \func{set}() \)
\li \( \id{brush} \gets \func{set}() \)
\li \For \id{node} \In \id{cDRG}
\li \Do 
\If \( \attrib{node}{isRequestNode?} \)
\li \Then \( \proc{DisableRequests}(\id{trace},\id{node},\id{disableCounts},\id{disabledRequestNodes},\id{brush}) \)
\End
\End
\li \Return \id{brush}
\end{codebox}

\begin{codebox}
\Procname{\(\proc{DisableRequests}(\id{trace},\id{node},\id{disableCounts},\id{disabledRequestNodes},\id{brush}\))}
\li \If \id{node} \In \id{disabledRequests}
\li \Then \Return
\End
\li \For \id{esrParent} \In \( \attrib{trace}{\func{getESRParents}}(\attrib{node}{outputNode}) \)
\li \Do
\( \id{disableCounts}[\id{esrParent}] \gets \id{disableCounts}[\id{esrParent}] + 1 \)
\li \If \( \id{disableCounts}[\id{esrParent}] \isequal \attrib{trace}{\func{getNumberOfRequests}}(\id{esrParent}) \)
\li \Then \( \proc{DisableFamily}(\id{trace},\id{esrParent},\id{disableCounts},\id{disabledRequestNodes},\id{brush}) \)
\End
\End
\end{codebox}

\begin{codebox}
\Procname{\(\proc{DisableFamily}(\id{trace},\id{node},\id{disableCounts},\id{disabledRequestNodes},\id{brush}\))}
\li \( \attrib{brush}{\func{insert}}(\id{node}) \)
\li \If \( \attrib{node}{isOutputNode?} \)
\li \Then
 \( \attrib{brush}{\func{insert}}(\attrib{node}{requestNode}) \)
\li 
\( \proc{DisableRequests}(\id{trace},\attrib{node}{requestNode},\id{disableCounts},\id{disabledRequestNodes},\id{brush}) \)
\li \( \proc{DisableFamily}(\id{trace},\attrib{node}{operatorNode},\id{disableCounts},\id{disabledRequestNodes},\id{brush}) \)
\li \For \( \id{operatorNode} \) \In \( \attrib{node}{operatorNodes} \)
\li \Do
\( \proc{DisableFamily}(\id{trace},\id{operatorNode},\id{disableCounts},\id{disabledRequestNodes},\id{brush}) \)
\End
\End
\end{codebox}

The border consists of nodes where resampling stops. This includes absorbing nodes, AAA nodes, and also resampling nodes with no children, e.g. from top-level \verb|PREDICT| directives:

\begin{codebox}
\Procname{\(\proc{FindBorder}(\id{trace},\id{drg},\id{absorbing},\id{aaa})\)}
\li \( \id{border} \gets \func{union}(\id{absorbing},\id{aaa}) \)
\li \For \id{node} \In \id{drg}
\li \Do \If \Not \( \func{hasChildInAorD}(\id{trace},\id{node},\id{drg},\id{absorbing}) \)
\li \Then \( \attrib{border}{\func{insert}}(\id{node}) \)
\End
\End
\li \Return \id{border}
\end{codebox}

For stochastic regeneration to proceed correctly, the nodes in the scaffold must be annotated with counts that allow determination of when a given node is no longer referenced and can therefore be removed.

\begin{codebox}
\Procname{\(\proc{ComputeRegenCounts}(\id{trace},\id{drg},\id{absorbing},\id{aaa},\id{border},\id{brush})\)}
\li \( \id{regenCounts} \gets \func{map}(\const{0}) \)
\li \For \id{node} \In \id{drg}
\li \Do \If \id{node} \In \id{aaa}
\li \Then 
\( \id{regenCounts}[\id{node}] \gets \const{1} \) \Comment{will be added to shortly}
\li \ElseIf \id{node} \In \id{border}
\li \Then  \( \id{regenCounts}[\id{node}] \gets \attrib{\attrib{trace}{\func{getChildren}}(\id{node})}{size} + \const{1} \)
\li \ElseNoIf
\li \( \id{regenCounts}[\id{node}] \gets \attrib{\attrib{trace}{\func{getChildren}}(\id{node})}{size} \)
\End
\End
\li \Comment{Now determine the number of times each reference to an AAA node will be regenerated}
\li \For \id{node} \In \( \func{union}(\id{drg},\id{absorbing}) \)
\li \Do
\For \id{parent} \In \( \attrib{trace}{\func{getParents}}(\id{node}) \)
\li \Do \( \proc{MaybeIncrementAAARegenCount}(\id{parent}) \)
\End
\End
\li \For \id{node} \In \( \id{brush} \)
\li \Do 
\If \( \attrib{node}{isLookupNode?} \)
\li \Then \( \proc{MaybeIncrementAAARegenCount}(\attrib{node}{sourceNode}) \)
\li \ElseIf \( \attrib{node}{isOutputNode?} \)
\li \Then \( \proc{MaybeIncrementAAARegenCount}(\attrib{trace}{\func{getESRParents}}(\id{node})[0]) \)
\End
\End
\li \Return \id{regenCounts}
\end{codebox}

\begin{codebox}
\Procname{\(\proc{MaybeIncrementAAARegenCount}(\id{trace},\id{node},\id{regenCounts})\)}
\li \( \id{value} \gets \attrib{trace}{\func{getValue}}(\id{node}) \)
\li \If \( \attrib{value}{\id{isSPRef?}} \) \AndNot \( \attrib{value}{makerNode} \isequal \id{node} \) \Andd \( \attrib{aaa}{\func{contains?}}(\attrib{value}{makerNode}) \)
\li \Then \( \id{regenCounts}[\attrib{value}{makerNode}] = \id{regenCounts}[\attrib{value}{makerNode}] + 1 \)
\End
\end{codebox}

\subsubsection{Absorbing at Applications}

As we discussed earlier, there are cases in which we do not need to
visit all the absorbing nodes explicitly to account for their
logdensities because the SP can keep track of its outputs and report
the joint logdensity of all of its applications. If this is the case,
we say that the application of the SP is absorbing at applications
(AAA), and our preliminary walk to the find the scaffold can stop upon
reaching an AAA node. Figure~\ref{fig:scaffold_AAA} gives a graphical
illustration. Note that the AAA node itself is in \( \DG \)--not \(
\AG \)--since its value may change. 

Handling AAA introduces many subtle bookkeeping challenges, and the code for
computing the regenCounts for AAA nodes will seem mysterious until inspecting
the special cases for AAA nodes in regen and extract in the next section.

\begin{figure}
(a)
\includegraphics[width=200pt,height=200pt]{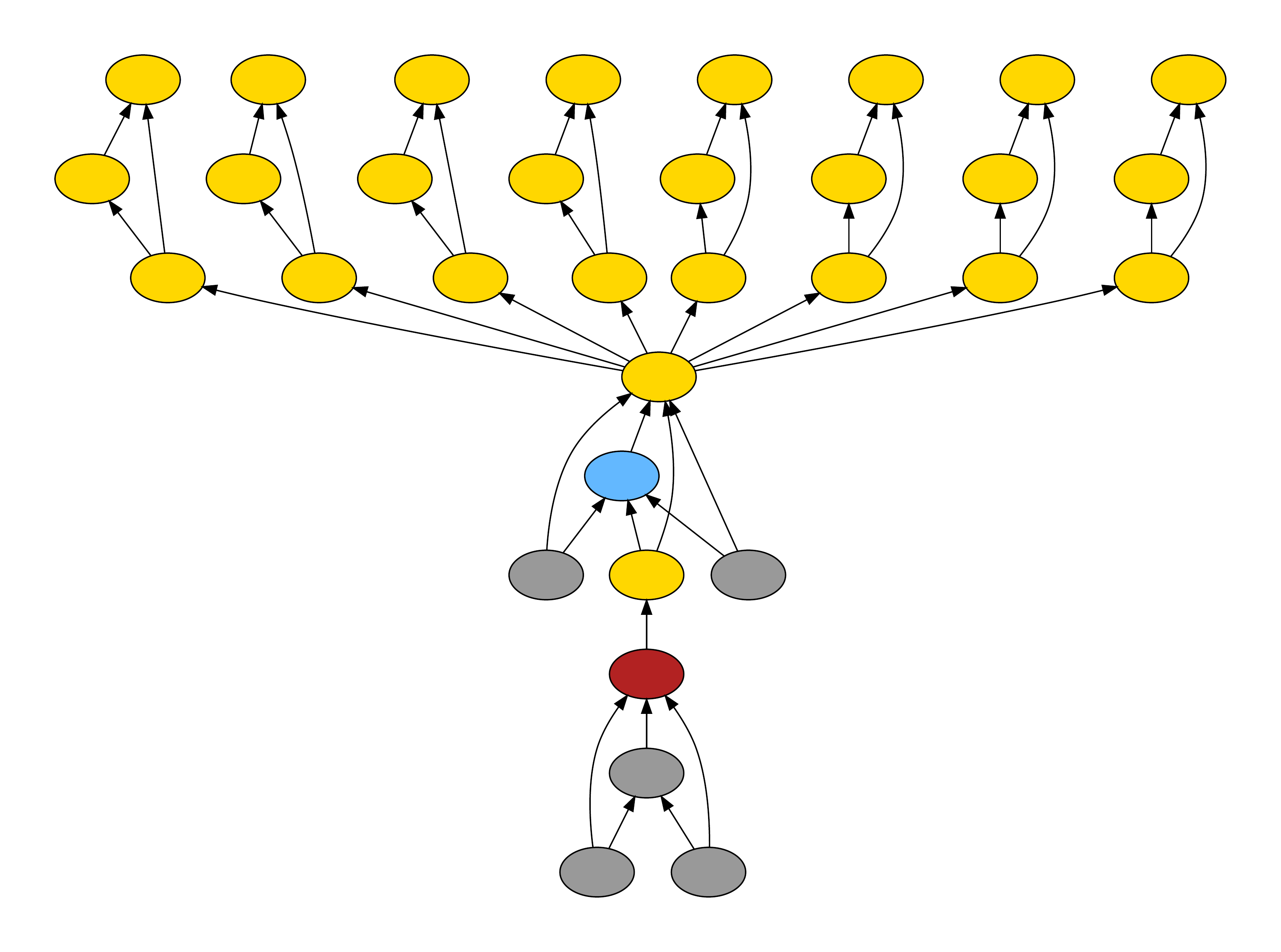}
(b)
\includegraphics[width=200pt,height=200pt]{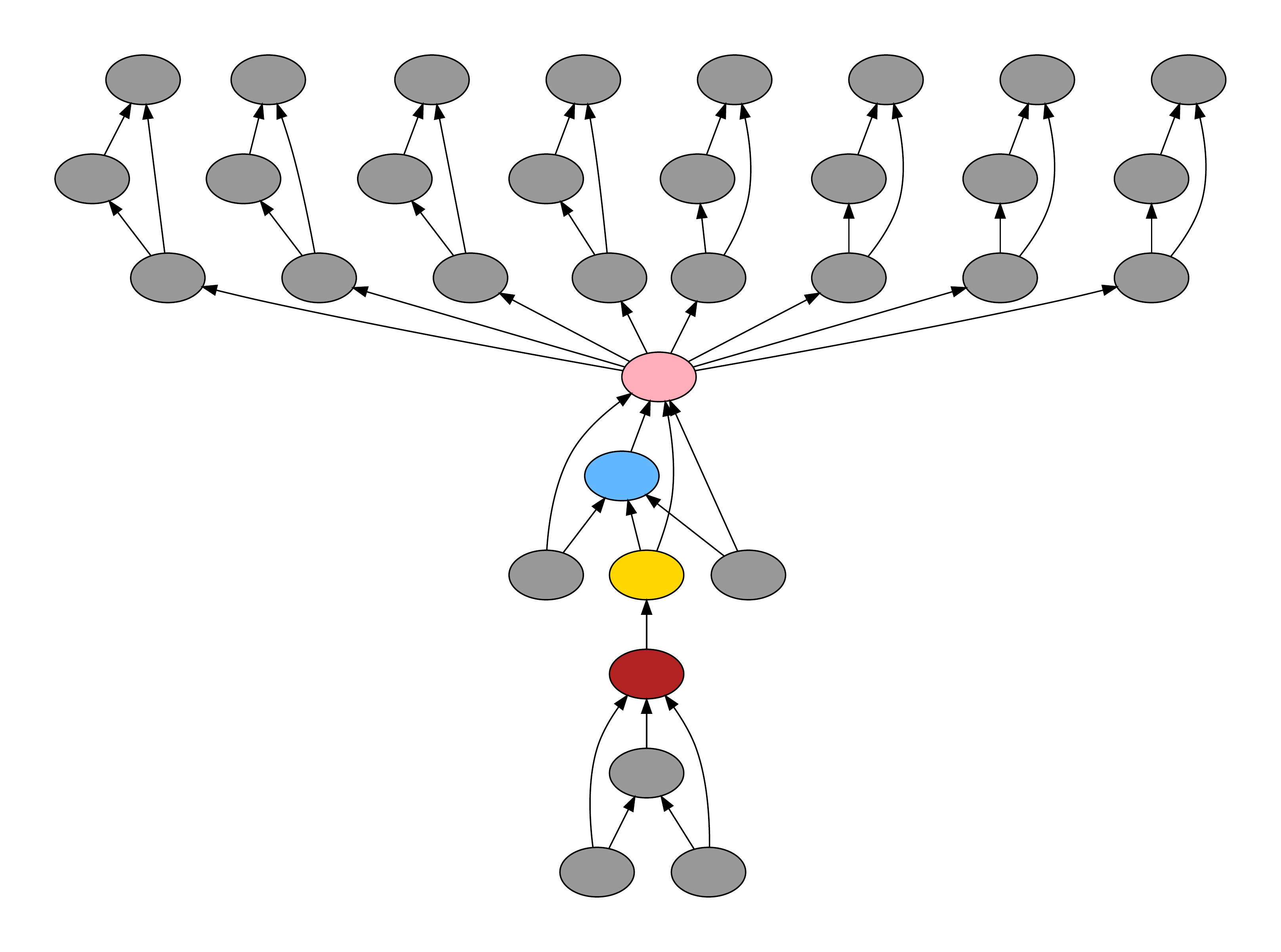}
\caption{{\bf Scaffolds with and without the ``absorbing-at-applications'' optimization.} Repeated applications of stochastic procedures are common in probabilistic programs for machine learning and statistics. (a) shows a PET fragment typical of this repetition. The principal node is in red, resampling nodes are in yellow, and the constructor SP is in blue. Note the size of the scaffold, which grows linearly with the number of applications. (b) shows the same PET fragment but where the ``constructor'' SP (whose resulting SP is repeatedly applied) can absorb at applications. The SP is now pink, and it because it absorbs at applications (implementing the weight calculations needed for inference via information in its auxiliary storage) the scaffold size no longer scales linearly with the number of applications.}
\label{fig:scaffold_AAA}
\end{figure}

\subsection{Breaking down a global inference problem into collections of local inference problems}

If the principal nodes are chosen to be all the random choices in the
current trace, then the border will correspond to the random choices
constrained by the \verb|OBSERVE|s, and the terminal resampling nodes
will be any \verb|ASSUME|s that are not subsequently referred to, as
well as all the \verb|PREDICT|s. The definite regeneration graph will
only contain random choices that are guaranteed to exist in every
execution of the probabilistic program. Conditionally simulating the
entire program can be reduced to conditionally simulating a completion
of a \(\torus\) given the absorbing nodes \(\AG\).

This equivalence is an important feature of Venture's design. Any
inference strategy that works for entire programs can be run on a
program fragment, conditioned on the remainder, and vice
versa. Venture can break down the problem of sampling from the
conditioned distribution over the entire PET, which might be extremely
high dimensional, into collections of overlapping lower-dimensional
inference subproblems. Any Venture strategy applicable to the overall
problem becomes applicable to each subproblem. Current research is exploring the use of Venture programs to construct custom proposals, where each \verb|PREDICT| directive is mapped to a principal node or a resampling node in the border, and each \verb|OBSERVE| statement is matched to an absorbing node.

\subsection{Local kernels and scaffolds}

Later sections describe techniques for building up transition operators that efficiently modify contents of scaffolds, sampling its definite regeneration graph and brush from a distribution that leaves its conditional invariant. These transition operators are built out of stochastic proposals for individual random choices that we call {\em local kernels}. These kernels must be able to sample new values for individual random choices as well as calculate the ratio of forward and reverse transition probabilities necessary for using the sampled value as a proposal for a Metropolis-Hastings scheme. Venture supports local kernels of two types: simulation and
delta kernels.

\subsubsection{Local simulation kernels: resimulation \& bottom-up proposals}

A local simulation kernel for a random choice \( x \) cannot look at previous values of \( x \) when making a proposal.  The simplest simulation kernel is the one that
samples \( x \) according to the \verb|simulate()| provided by its associated
stochastic procedure. We will sometimes refer to this as a ``resimulation''
kernel. 

\begin{enumerate}
\item Let \( \pi_\to \) denote the order in which the local kernels are called
  during sample. Then \( \sample \) returns a new value \( x \) along with the following weight:
 \[ 
\frac{P_{\pi_\to}(x)}{\K_{\pi_\to}(x)}
\]
\item Let \( \pi_\leftarrow \) denote the reverse of the order in which the local kernels
  are called during unsample. Then \( \unsample \) returns the following weight:
\[ 
\frac{\K_{\pi_\leftarrow}(x)}{P_{\pi_\leftarrow}(x)}
\]
\end{enumerate}

For the resimulation kernel, the weights from this local kernel are identically 1. Simulation kernels can also be conditioned on anything in $\torus$ without breaking asymptotic convergence. This is because the values of random choices in the torus are all available at proposal time and also unaffected by the transition. Simulation kernels can thus be used to make bottom-up proposals, using other values in the source trace as input. This provides one mechanism for augmenting the top-down processing that is typical in probabilistic programming with custom bottom-up information.

\todo[color=green]{[big] Decide on scheme for partially filled scaffolds, and relationship to inference language and sequential inference code}

\subsubsection{Local delta kernels and reuse of random choices}

Venture also supports delta kernels. When these kernels transition the
trace from \( x' \to x \), they have access to the previous state \(
x' \) in addition to the contents of the torus. A Gaussian drift
kernel, where $x \sim \mathcal{N}(x', \sigma^2)$, is a typical
example. They satisfy the following contract:

\begin{enumerate}
\item \( \sample \) returns a state $x'$ with weight:
\[ 
\frac{P(x) \K(x \to x')}{P(x') \K(x' \to x)}
\]
\item \( \unsample \) returns 0.
\end{enumerate}

These kernels cannot be applied to stochastic procedures that exhibit exchangeable coupling between applications, nor can they be used in particle methods. Since they receive the source state as an input, they also provide a mechanism for re-using the old value of a random choice if its parent choices have not changed. Delta kernels can thus reduce unnecessary or undesirable resampling.

\todo[color=green]{[big] Schematic of trace, with bottom-up proposals clearly described as correct via formalism; separate section on this?}

\section{Stochastic Regeneration Algorithms for Scaffolds}

We now show how to coherently modify a trace, given a partition and a
valid scaffold, using an algorithm we introduce called {\em stochastic
  regeneration}. Variations on stochastic regeneration, some of which
can be expressed as simple parameterizations, can be used to implement
a wide range of stochastic inference strategies.

The simplest use of stochastic regeneration is to implement a
Metropolis-Hastings transition operator for PETs as follows. First,
the border of the scaffold is ``detached'' from the trace in an
arbitrary order, and the random choices within the trace fragment that
has been detached are stored in an ``extract''. We will sometimes abbreviate this process as $\detach$. Next, the border of the
scaffold is ``regenerated and attached'' in reverse order, yielding a
new trace. We will sometimes abbreviate this process as $\regen$. This new trace can be accepted or rejected. If it is accepted, the algorithm has finished. If it is rejected, it is
detached again, and the original extract is fed to the regeneration
algorithm. This restores the trace to its original state.

Assuming simulation and density evaluation of the constituent PSPs is
constant, the runtime of this process scales with the size of the
scaffold and the brush, not the size of the entire PET. Nodes that
cannot influence or be influenced by the transition are never
visited. PET nodes are regenerated in the opposite order that they
were visited during detach; see Figure~\ref{fig:order-reversal} for a
trace that illustrates the need for this symmetry.

\todo[color=green]{[medium] Signpost figure for the idea of stochastic regeneration, plus introduction of key terminology; based on that, place subtleties figures; introduce idea of the DB (though schematic is puntable for JMLR)}

\begin{figure}
\includegraphics[width=\textwidth]{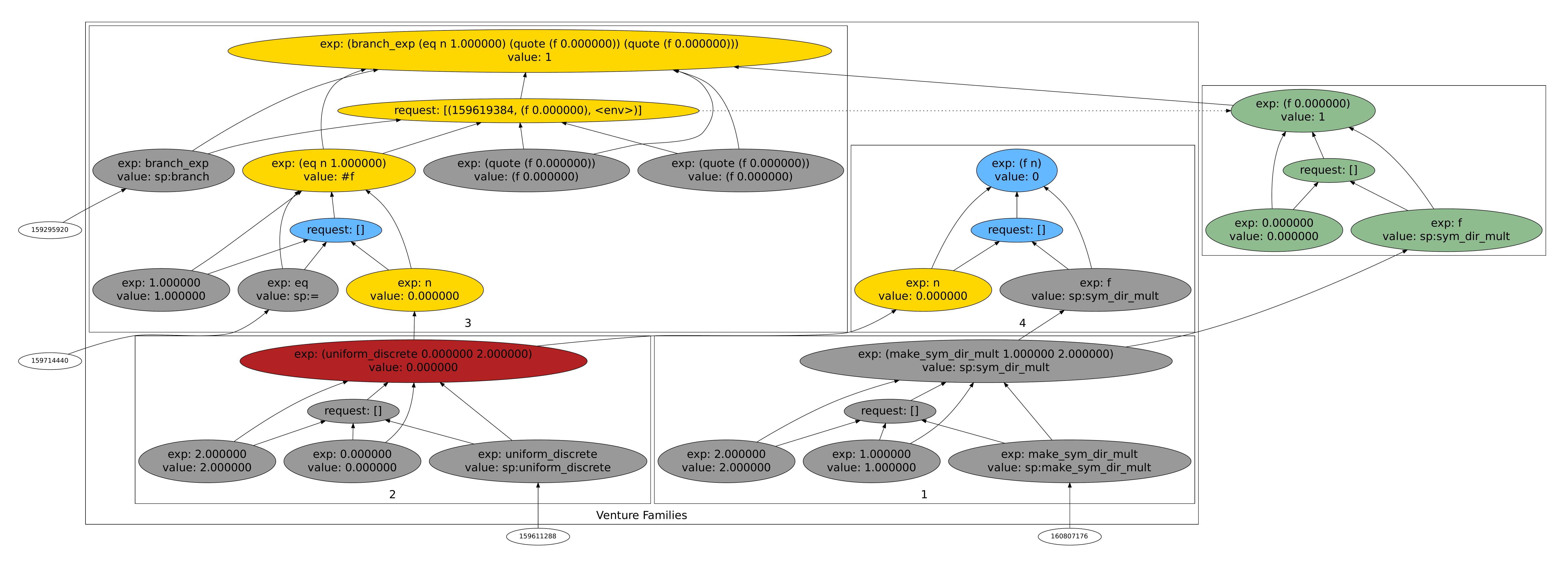}
\caption{{\bf Due to partially exchangeable SPs, $\detach$ and $\regen$ must visit nodes in opposite orders.}
A procedure with partial-exchangeable coupling is applied in both the absorbing border and the brush. The probability of the absorbing nodes depends on the order in which regeneration takes place. As Venture is detaching, it needs to compute the probability of the absorbing nodes that it would have calculated if it had proposed this trace by regenerating along the scaffold. In particular, if it would have visited the absorbing application first during $\regen$, then it must visit the absorbing application last during $\detach$. Venture solves this problem by having $\detach$ always detach all nodes in the opposite order from how $\regen$ would have generated them.}
\label{fig:order-reversal}
\end{figure}

\todo[color=yellow]{[medium] Fix AAA challenge figure by bringing up to date; rewrite caption, coloring, ...}


\subsection{Detaching along a scaffold with DETACH-AND-EXTRACT}

Detach-and-Extract takes a scaffold and a trace, and converts the trace into a torus and a full omegaDB. At that point, the torus is ready for a new proposal, while the omegaDB then contains all the information needed to restore the original trace. The first step is to disconnect (detach) the trace fragment from the border of the scaffold. The second step is to extract all the random choices from the trace fragment and store them in the omegaDB.

\begin{codebox}
\Procname{\(\proc{DetachAndExtract}(\id{trace},\id{border},\id{scaffold})\)}
\li \( \id{weight} \gets 0 \)
\li \( \id{db} \gets \func{DB}() \)
\li \For \( \id{node} \) \In \( \func{reversed}(\id{border}) \)
\li \Do
\If \( \attrib{\id{scaffold}}{\func{isAbsorbing}}(\id{node}) \)
\li \Then
\( \id{weight} \gets \id{weight} + \proc{Detach}(\id{trace},\id{node},\id{scaffold},\id{db}) \)
\li \ElseNoIf
\li \If \( \attrib{node}{\func{isObservation}} \)
\li \Then \( \id{weight} \gets \id{weight} + \proc{Unconstrain}(\id{trace},\id{node}) \) 
\End
\li \( \id{weight} \gets \id{weight} + \proc{Extract}(\id{trace},\id{node},\id{scaffold},\id{db}) \)
\End
\End
\li \Return \( \id{weight}, \id{db} \)
\end{codebox}

To detach a node from the border, Venture must unincorpoate the value from the PSP that was previously responsible for generating it, adjust the weights, and then extract its parents:

\begin{codebox}
\Procname{\(\proc{Detach}(\id{trace},\id{node},\id{scaffold},\id{db})\)}
\li \( \id{psp},\id{args} \gets \attrib{trace}{\func{getPSP}}(\id{node}), \attrib{trace}{\func{getArgs}}(\id{node}) \)
\li \( \id{groundValue} \gets \attrib{trace}{\func{getGroundValue}}(\id{node}) \)
\li \( \attrib{psp}{\func{unincorporate}}(groundValue,args) \)
\li \( \id{weight} \gets \attrib{psp}{\func{logDensity}}(groundValue,args) \)
\li \( \id{weight} \gets \proc{ExtractParents}(\id{trace},\id{node},\id{scaffold},\id{db}) \)
\li \Return \( \id{weight} \)
\end{codebox}

To extract the parents of a node, Venture loops over the parents in the correct order:

\begin{codebox}
\Procname{\(\proc{ExtractParents}(\id{trace},\id{node},\id{scaffold},\id{db})\)}
\li \( \id{weight} \gets \proc{ExtractESRParents}(\id{trace},\id{node},\id{scaffold},\id{db})\)
\li \For \( \id{parent} \) \In \( \func{reversed}(\attrib{trace}{\func{getDefiniteParents}}(\id{node})) \)
\li \Do
\( \id{weight} \gets \id{weight} + \proc{Extract}(\id{trace},\id{parent},\id{scaffold},\id{db}) \)
\End
\li \Return \id{weight}
\end{codebox}

\begin{codebox}
\Procname{\(\proc{ExtractESRParents}(\id{trace},\id{node},\id{scaffold},\id{db})\)}
\li \( \id{weight} \gets 0 \)
\li \For \( \id{parent} \) \In \( \func{reversed}(\attrib{trace}{\func{getESRParents}}(\id{node})) \)
\li \Do
\( \id{weight} \gets \id{weight} + \proc{Extract}(\id{trace},\id{parent},\id{scaffold},\id{db}) \)
\End
\li \Return \id{weight}
\end{codebox}

The key idea in this process is that when a node is no longer referenced by any others --- that is, its \verb|regenCount| is 0 --- it can be unevaluated. Also, if a node is a request node, before unevaluating it all requests it generates should be unevaluated, which may trigger the trace fragments corresponding to the request to be unevaluated if they are no longer referenced.

\begin{codebox}
\Procname{\(\proc{Extract}(\id{trace},\id{node},\id{scaffold},\id{db})\)}
\li \( \id{weight} \gets 0 \)

\li \( \id{value} \gets \attrib{trace}{\func{getValue}}(\id{node}) \)
\li \If \( \attrib{value}{isSPRef?} \) \AndNot \( \attrib{value}{makerNode} \isequal \id{node} \) \Andd \( \attrib{\id{scaffold}}{\func{isAAA}}(\attrib{value}{makerNode}) \)
\li \Then  \( \id{weight} \gets \id{weight} + \proc{Extract}(\id{trace},\attrib{value}{makerNode},\id{scaffold},\id{db}) \)
\End 

\li \If \( \attrib{\id{scaffold}}{\func{isResampling}}(\id{node}) \)
\li \Then \( \attrib{\id{scaffold}}{\func{decrementRegenCount}}(\id{node}) \)
\li \If \( \attrib{\id{scaffold}}{\func{getRegenCount}}(\id{node}) \isequal 0 \)
\li \Then \If \( \attrib{node}{isLookup?} \)
\li \Then \( \attrib{trace}{\func{clearValue}}(\id{node}) \)
\li \ElseNoIf
\li \If \( \attrib{node}{isRequestNode?} \)
\li \Then  \( \id{weight} \gets \id{weight} + \proc{UnevalRequests}(\id{trace},\id{node},\id{scaffold},\id{db}) \)
 \End
\li \( \id{weight} \gets \id{weight} + \proc{UnapplyPSP}(\id{trace},\id{node},\id{scaffold},\id{db}) \)
 \End \End
\li \Then \( \id{weight} \gets \id{weight} + \proc{ExtractParents}(\id{trace},\id{node},\id{scaffold},\id{db}) \)
\End
\End
\li \Return \( \id{weight} \)
\end{codebox}

\subsection{Regenerating a new trace along a scaffold with REGENERATE-AND-ATTACH}

Regenerate-and-Attach, often abbreviated as regen, takes a scaffold, a torus, and an omegaDB, as
well as a control parameter that indicates whether or not it is
restoring (i.e. replaying random choices from the omegaDB)\footnote{It
  also takes a map from nodes to the gradient of their log densities,
  for use by variational kernels; we will discuss this later in the
  advanced inference section.}. The key idea in regen is that it ensures the PET is constructed via sequence of self-supporting PETs, that is PETs where all parents of a node --- including all choices on which a given random choice directly depends --- have been regenerated before the node itself is.
  
\begin{codebox}
\Procname{$\proc{RegenerateAndAttach}(\id{trace},\id{border},\id{scaffold},\id{restore?},\id{db})$}
\li \( \id{weight} \gets 0 \)
\li \For \( \id{node} \) \In \( \id{border} \)
\li \Do
\If \( \attrib{\id{scaffold}}{\func{isAbsorbing?}}(\id{node}) \)
\li \Then
\( \id{weight} \gets \id{weight} + \proc{Attach}(\id{trace},\id{node},\id{scaffold},\id{restore?},\id{db}) \)
\li \ElseNoIf
\li \( \id{weight} \gets \id{weight} + \proc{Regenerate}(\id{trace},\id{node},\id{scaffold},\id{restore?},\id{db}) \)
\li \If \( \attrib{node}{\func{isObservation?}} \)
\li \Then \( \id{weight} \gets \id{weight} + \proc{Constrain}(\id{trace},\id{node},\attrib{node}{observedValue}) \)
\End
\End
\End
\li \Return \( \id{weight} \)
\end{codebox}

To attach a node from the border, one first regenerates its parents, then ``attaches'' the trace fragment to the current trace, making it responsible for generating the current node in the border. This involves updating the weight and also incorporating the value into the stochastic procedure that is newly responsible for generating it.

\begin{codebox}
\Procname{\(\proc{Attach}(\id{trace},\id{node},\id{scaffold},\id{restore?},\id{db})\)}
\li \( \id{psp},\id{args} \gets \attrib{trace}{\func{getPSP}}(\id{node}), \attrib{trace}{\func{getArgs}}(\id{node}) \)
\li \( \id{groundValue} \gets \attrib{trace}{\func{getGroundValue}}(\id{node}) \)
\li \( \id{weight} \gets \proc{RegenerateParents}(\id{trace},\id{node},\id{scaffold},\id{restore?},\id{db}) \)
\li \( \id{weight} \gets \id{weight} + \attrib{psp}{\func{logDensity}}(groundValue,args) \)
\li \( \attrib{psp}{\func{incorporate}}(groundValue,args) \)
\li \Return \( \id{weight} \)
\end{codebox}

\begin{codebox}
\Procname{\(\proc{RegenerateParents}(\id{trace},\id{node},\id{scaffold},\id{restore?},\id{db})\)}
\li \( \id{weight} \gets 0 \)
\li \For \( \id{parent} \) \In \( \attrib{trace}{\func{getDefiniteParents}}(\id{node}) \)
\li \Do
\( \id{weight} \gets \id{weight} + \proc{Regenerate}(\id{trace},\id{parent},\id{scaffold},\id{restore?},\id{db}) \)
\End
\li \( \id{weight} \gets \id{weight} + \proc{RegenerateESRParents}(\id{trace},\id{node},\id{scaffold},\id{restore?},\id{db})\)
\li \Return \id{weight}
\end{codebox}

\begin{codebox}
\Procname{\(\proc{RegenerateESRParents}(\id{trace},\id{node},\id{scaffold},\id{restore?},\id{db})\)}
\li \( \id{weight} \gets 0 \)
\li \For \( \id{parent} \) \In \( \attrib{trace}{\func{getESRParents}}(\id{node}) \)
\li \Do
\( \id{weight} \gets \id{weight} + \proc{Regenerate}(\id{trace},\id{parent},\id{scaffold},\id{restore?},\id{db}) \)
\End
\li \Return \id{weight}
\end{codebox}

Regenerating a node involves updating its regeneration counts, to ensure that after regeneration, detach can be correctly called on the new trace:

\begin{codebox}
\Procname{\(\proc{Regenerate}(\id{trace},\id{node},\id{scaffold},\id{restore?},\id{db})\)}
\li \( \id{weight} \gets 0 \)
\li \If \( \attrib{\id{scaffold}}{\func{isResampling?}}(\id{node}) \)
\li \Then \If \( \attrib{\id{scaffold}}{\func{getRegenCount}}(\id{node}) \isequal 0 \)
\li \Then \( \id{weight} \gets \id{weight} + \proc{RegenerateParents}(\id{trace},\id{node},\id{scaffold},\id{restore?},\id{db}) \)
\li \If \( \attrib{node}{isLookup?} \)
\li \Then \( \attrib{trace}{\func{setValue}}(\id{node},\attrib{trace}{\func{getValue}}(\id{sourceNode})) \)
\li \ElseNoIf
\li \( \id{weight} \gets \id{weight} + \proc{ApplyPSP}(\id{trace},\id{node},\id{scaffold},\id{restore?},\id{db}) \)
\li \If \( \attrib{node}{isRequestNode?} \)
\li \Then  \( \id{weight} \gets \id{weight} + \proc{EvalRequests}(\id{trace},\id{node},\id{scaffold},\id{restore?},\id{db}) \)
 \End \End \End
\li \( \attrib{\id{scaffold}}{\func{incrementRegenCount}}(\id{node}) \)
\End
\li \( \id{value} \gets \attrib{trace}{\func{getValue}}(\id{node}) \)
\li \If \( \attrib{value}{\id{isSPRef?}} \) \AndNot \( \attrib{value}{makerNode} \isequal \id{node} \) \Andd \( \attrib{\id{scaffold}}{\func{isAAA?}}(\attrib{value}{makerNode}) \)
\li \Then  \( \id{weight} \gets \id{weight} + \proc{Regenerate}(\id{trace},\attrib{value}{makerNode},\id{scaffold},\id{restore?},\id{db}) \)
\End 
\li \Return \( \id{weight} \)
\end{codebox}

\subsection{Building invariant transition operators using stochastic regeneration}

Stochastic regeneration can be used to implement a variety of inference schemes over probabilistic execution traces. In addition to modifying the trace in an undoable fashion as described above, stochastic regeneration provides numerical weights that are needed for inference schemes such as Metropolis-Hastings.

\subsubsection{Weights assuming simulation kernels}


Assume that the local kernels used for each stochastic procedure application is a simulation kernel --- that is, the value it proposes is independent of the value of the random choice in \( \rho \). Let \( \pi \) be the (after the fact) \( \regen \) order for \( \xi
\), which is the order in which the local kernels are called during \(
\regen \), and which is also by construction the reverse of the order
in which the local kernels would be called during \( \detach \). Then
\begin{enumerate}
\item \( \regen \) returns
\[ 
\frac
{P_\pi(\RG(\xi),\AG)}
{\K_{\regen}(\RG(\xi))}
\label{eq:regenSimulation}
\]
\item \( \detach \) returns:
\[ 
\frac
{\K_{\regen}(\RG(\xi))}
{P_{\pi}(\RG(\xi),\AG)}
\label{eq:detachSimulation}
\]
\end{enumerate}
where \( \K_\regen \) is the kernel that proposes to \( \RG(\cdot) \) by calling local kernels in the \( \regen \) order, and \( \detach \) computes the reverse probability correctly since each local kernel is unapplied in the reverse of the order in which it would have been applied during \( \regen \). These weights can be used to implement Metropolis-Hastings transitions between \( \rho \) and \( \xi \).

Note that if all local kernels are resimulation kernels, only the ``likelihood ratio'' induced by the absorbing nodes is involved in the Metropolis-Hastings acceptance ratio.

\subsubsection{Weights assuming delta kernels}

Now assume that at least one of the local kernels is a delta kernel --- that is, the value it proposes depends on the value of the random choice being modified from \( \rho \). In this case, \( \regen \) and \( \detach \) do not return useful quantities individually, but their product (\( \detachAndRegen \)) is the same
as in the case of simulation kernels, and is suitable for implementing Metropolis-Hastings transitions:

\begin{align} \label{eq:detachAndRegen}
\frac
{P_{\pi_\xi}(\RG(\xi), \AG) \K_{\regen}(\RG(\xi) \to \RG(\rho))}
{P_{\pi_\rho}(\RG(\rho), \AG) \K_{\regen}(\RG(\rho) \to \RG(\xi))}
\end{align}

\todo[color=green]{[small] Discuss delta kernel trickery and potential extensions to the inference language}

\todo[color=green]{[medium] Add section discussing asymptotic complexity, touching on inference scopes}

\todo[color=yellow]{[medium] Add short section on bottom-up inference, with schematics; point is to defend turf for an inference programming language}

\subsection{Context-independent versus context-specific inference schemes}

The contents of the scaffold and brush do not depend on the values of the principal nodes for a given transition. For example, every random choice whose existence might be affected by a value chosen for a principal node (or a descendant of a principal node that must be resampled) must be included in the brush, and regenerated during a transition, even if the particular value(s) sampled for the principal node were consistent with the old execution trace.  It could be asymptotically more efficient in some circumstances to exploit context-specific independencies, both for the values of random choices and for their existence.

Unfortunately, context-specificity comes at the cost of significant additional complexity. Other Turing-complete probabilistic programming systems --- for example, the original implementation in \citep{Mansinghka:2009}, as well as MIT-Church, and an earlier version of Venture --- were context sensitive. Monte was context-independent but significantly more limited than Venture. Of these, we believe only the earlier version of Venture achieved the correct order of growth for all scale parameters in Latent Dirichlet Allocation, and none supported reprogrammable inference.

It seems likely that scaffold construction and stochastic regeneration can be interleaved to retain the flexibility of our current architecture while achieving the efficiency of an update scheme sensitive to context-specific independencies. We leave this for future work.

\section{General-purpose Inference Strategies via Stochastic Regeneration}

Venture's inference programming language provides a number of primitive inference mechanisms \todo[color=yellow]{[small] Consider exposing SMC idea in the inference language.} \todo[color=green]{Support a cloud of weighted traces; sort out multiripl extensions (two "resample" instructions, one to pick a trace from a weighted cloud, and one to implement e.g. multinomial or stratified resampling)}. The effect of each strategy is the evolution of a PET according to a transition operator that leaves the posterior distribution on PETs invariant, with modifications constrained to lie within the scaffold. Each of these strategies can be applied to any scaffold, including the global scaffold. Each is implemented using stochastic regeneration, with \( \regen \) and \( \detach \) performing mutation of the trace in place and returning weights that can be used for any rejection steps necessary to maintain invariance.

Before describing specific inference strategies, we establish some preliminary results that simplify reasoning about transition operators on PETs.

\subsection{Factorization of the acceptance ratio}

\todo[color=yellow]{[small] Reintroduce products and subsequences for SPs}

Let \( \rho \) be a PET representing an execution and \( \pi \) be an ordering of its nodes. Suppose we visit each node in this order, for each SP application (i.e. random choice) computing the probability density of its output given its input, and then passing that input-output pair to the incorporate method associated with the SP, to handle exchangeable coupling. Define \( P_\pi(\rho) \) to be the total probability density of the sequence of random choices in \( \rho \), i.e. the product of the probability densities calculated above. Recall that every SP satisfies the property
that the probability of any sequence of input-output pairs is invariant under permutation. Thus we have that \( P_\pi(\rho) = P(\rho) \). 

For an arbitrary subset \( Z \) of the nodes in \( \rho \), we do not
necessarily have \( P_{\pi_1}(Z) = P_{\pi_2}(Z) \). The equivalence only holds
over the probability of all of \( \rho \), not over arbitrary fragments of
it. For example, consider the following program:

\vspace{0.05in}\begin{lstlisting}[frame=single,showstringspaces=false]
(assume f (make-ccoin 1 1))
(predict (f))
(predict (f))
\end{lstlisting}

\noindent Now let \( \rho \) be an execution with (say) both predicts \( \true \). Then the
probability of the first predict depends on whether or not it is weighted before
or after the second predict, but the total probability will be the same no
matter which order we proceed in. 

However, the probabilities will be equal if \( Z \) is a prefix in both
orderings, because we can view \( Z \) as the complete trace associated with some modified probabilistic program $\rho'$.

\begin{prop} \label{prop:factorRatio}
Let \( \rho ,\xi  \in \Xi[\torus,\DG] \). Let \( \pi_{\rho} \) be a permutation
on \( \rho \) such that \( \pi_\rho(\PG \cup \IG) < \pi_\rho(\RG(\rho) \cup \AG) \), and likewise for \( \pi_\xi \). Then
\begin{align}
\frac{P(\rho)}{P(\xi)} 
&= 
\frac
{P_{\pi_\rho}(\RG(\rho), \AG)}
{P_{\pi_\xi}(\RG(\xi), \AG)}
\end{align}
\end{prop}

\begin{proof}
We can factor $\rho$ and $\xi$ into disjoint subsets:
\begin{align}
\rho &=  (\PG \cup \IG) \dot{\cup} (\RG(\rho) \cup \AG) \\
\xi &=  (\PG \cup \IG) \dot{\cup} (\RG(\xi) \cup \AG)
\end{align}
We have \( P_{\pi_\rho}(\PG \cup \IG) = P_{\pi_\xi}(\PG \cup \IG) \) since the
set is a prefix in both orderings. The result follows immediately.
\end{proof}

\todo[color=yellow]{[small] ensure the factorization result is referred to later in the text, by number}

\todo[color=green]{[big] Add a careful treatment of kernel composition w a review of standard stuff (or [medium] for a sloppier one) as a subsection}

\subsection{Auxiliary variables for state-dependent stochastic selection of kernels}

We would like to be able to use Metropolis-Hastings to construct valid transition operators that leave a conditional distribution $P()$ on PETs invariant. In some cases, we also would like these transition operators to be individually ergodic. For any random choice guaranteed to be in all executions --- that is, the SP application is in the DRG of the global scaffold --- this is straightforward. Standard cycle or mixture hybrid kernels (see e.g. \citep{bonawitz2008composable, andrieu2003introduction, Mansinghka:2009}) can be used to sequence Metropolis-Hastings proposals on individual variables. Alternately, if we are only interested in inference strategies that choose random choices at uniformly at random, we can make a compound Metropolis-Hastings proposal that first selects a random variable to change (using it to seed a scaffold) and then makes a within that scaffold. Because the choice of random choices involves only a simple state dependence, and it is (nearly) trivial to integrate out the number of paths from some trace \( \rho \) to another trace \( \xi \) under this sequence of choices, it is straightforward to accommodate it in a single Metropolis-Hastings step.

Venture supports a broader range of transition operators: Metropolis-Hastings proposals where the proposal kernel is randomly chosen from a state-dependent distribution. Let \( I \) be a (possibly infinite) set, \( f \) a stochastic function from traces to \( I \) with induced density $P_f()$. For every \( i \in I \), let \( Q_i \) be a proposal kernel on PETs, and let $\rho$ be the current PET. Then the following cycle
kernel on the extended state space \( \Omega \times I \) preserves \( \pi \): 

\begin{enumerate}
\item Sample \( i | \rho  \sim f(\rho) \).

\item 
\begin{enumerate}
\item Propose \( \xi \sim Q_i(\rho \to \xi) \).

\item Accept if
 \[ U \sim \mathbf{Uniform}(0,1) < 
\frac{P_f(f(\xi) = i)}{P_f(f(\rho) = i)} 
\cdot
\frac
{P(\xi) Q_i(\xi \to \rho)}
{P(\rho) Q_i(\rho \to \xi)}
\]
\end{enumerate}
\end{enumerate}

This scheme treats the choice of kernel as a transient auxiliary variable \( I \) that is created solely for the duration of the ordinary Metropolis-Hastings proposal that depends on it. Invariance is ensured by accounting for the effect of \( Q_i \) on \( P_f(f(\xi) = i) \). This construction can be used to straightforwardly justify sophisticated custom proposal strategies such as Algorithm 8 from \citep{Neal:1998wz} without need for any approximation.

Stochastic index selection makes it straightforward to compositionally analyze richer transition operators on PETs. For example, one can select scaffolds by choosing a single principal node uniformly at random, as in typical Church implementations. Alternately, one could estimate the entropy of the conditional distributions for each random choice given its containing scaffold, and choose scaffolds with probability proportional to this estimate. One could also select principal nodes by starting with some query variable of interest --- e.g. a random choice representing an important prediction or action variable --- and walk the PET, favoring random choices that are ``close'' to the query variable. 

We anticipate that the evolution of the Venture inference programming language will depend on the expansion and refinement of a library of techniques for converting proposals into invariant transition operators.

\subsection{Metropolis-Hastings via Stochastic Regeneration}

Here we describe the basic Metropolis-Hastings algorithm for incremental
inference over a PET. With all of the tools we have developed so far, a generic version of this transition operator is straightforward to specify:

\begin{enumerate}
\item Sample a set of principal nodes \( \pns  \sim f(\rho) \), and record the probability \( P(f(\rho) = \pns) \) of sampling it.

\item Construct \( \DG \) with \( \pns \) as the set of principal nodes.

\item Set \( \alpha = \detachAndRegen(\DG) \)

\item Calculate \( P(f(\xi) = \pns) \).

\item Accept if
\[ U < 
\frac{P(f(\xi) = \pns)}{P(f(\rho) = \pns)}
\cdot
\alpha \]
\end{enumerate}

Correctness follows immediately from \ref{prop:factorRatio} and \eqref{eq:detachAndRegen}, and applies to a broad class of customized kernels. Because stochastic regeneration does not visit portions of the PET that are conditionally independent of the random choices affected by the transition, it is more scalable than approaches based on transformational compilation. In particular, for typical machine learning problems on datasets of size $N$, the number of random choices scales linearly with $N$ but the connections in the PET are sparse. Thus an inference sweep of $N$ single-site Metropolis-Hastings transitions --- enough to consider each random choice roughly once --- requires $O(N)$ time, as opposed to $O(N^2)$ time for transformational compilers.

\todo[color=green]{[small] CLRS pseudocode for MH}

\todo[color=green]{[medium] Clearly explain and illustrate idea of approximating the conditional distribution on a scaffold}

\subsection{Approximating optimal proposals via Stochastic Variational Inference}

We can use the same auxiliary variable technique to lift any local
proposal into one that preserves global stationarity. Here we show how
to learn a local variational approximation to the posterior
distribution on \( \RG(\xi) \) given \( \AG \) using \( \regen \) and
\( \detach \), and then wrap it in Metropolis-Hastings as discussed
above.  

\subsubsection{Posing the optimization problem}

\todo[color=green]{[big] Add variational inference pedagogy}

\todo[color=yellow]{[small] Discuss q() as an arbitrary program}

Let \( \rho \) be a trace and \( \DG \) a definite regeneration graph inducing
\( \torus \). Let \( \mathcal{Q} \) be some family of
distributions on \( \Xi[\torus,\DG] \) where every \( Q_\theta \in
\mathcal{Q}\) is defined by its parameters \( \theta \in \Theta \), which
control the local kernels along \( \DG \), and where each \( Q_\theta \) reverts
to the resimulation kernel on the brush. Our goal is to I-Project \( P(\RG) \) onto the
space \( \mathcal{Q} \), which involves solving the following optimization
problem:

\begin{equation*}
\begin{aligned}
&\text{minimize}
& & f(\theta) = \mathbb{E}_{\xi \sim Q_\theta} \left[ \log \frac{
    Q_\theta(\RG(\xi)) }{P(\RG(\xi) , \AG)}
\right] \\
& \text{subject to}
& & \theta \in \Theta \\
\end{aligned}
\end{equation*}

\subsubsection{Stochastic gradient descent}

One generic approach is stochastic gradient descent. Under assumptions of
differentiability that may not hold in arbitrary programs, the gradient of the
objective function has the following nice form (as originally shown in
\citep{wingateweber2013} for entire traces):

\todo[color=green]{[small] Check alignment of SGD math before submission}
\begin{align}
\nabla_{\theta} f(\theta)
&=
\nabla_{\theta} \mathbb{E}_{\xi \sim Q_\theta} \left[ \log \frac{
    Q_\theta(\RG(\xi)) }{P(\RG(\xi), \AG)}
\right] \\
&=
\nabla_{\theta} \int_{\xi \in \Xi} Q_{\theta}(\RG(\xi))  \log \frac{
  Q_\theta(\RG(\xi)) }{P(\RG(\xi) , \AG)}
 d \xi \\
&=
\int_{\xi \in \Xi} \nabla_{\theta} \left\{Q_{\theta}(\RG(\xi))  \log \frac{
    Q_\theta(\RG(\xi)) }{P(\RG(\xi) ,\AG)}
 \right\} d \xi \\
&=
\int_{\xi \in \Xi} \nabla_{\theta} Q_{\theta}(\RG(\xi)) \log \frac{
  Q_\theta(\RG(\xi)) }{P(\RG(\xi) ,\AG)}
  d \xi
+
\int_{\xi \in \Xi} Q_{\theta}(\RG(\xi)) \left( \nabla_{\theta} \log \frac{
    Q_\theta(\RG(\xi)) }{P(\RG(\xi) ,\AG)}
 \right) d \xi \\
\end{align}
\begin{align}
&=
\int_{\xi \in \Xi} \nabla_{\theta} Q_{\theta}(\RG(\xi)) \log \frac{
  Q_\theta(\RG(\xi)) }{P(\RG(\xi) ,\AG)}
  d \xi \\
&=
\int_{\xi \in \Xi} Q_{\theta}(\RG(\xi)) \nabla_{\theta} \log Q_{\theta}(\RG(\xi)) \log \frac{ Q_\theta(\RG(\xi)) }{P(\RG(\xi),\AG)}
  d \RG(\xi) \\
&=
\mathbb{E}_{\xi \sim Q_{\theta}} \left[\nabla_{\theta} \log
  Q_{\theta}(\RG(\xi)) \log \frac{ Q_\theta(\RG(\xi)) }{P(\RG(\xi),\AG)} \right]
\end{align}
where we used that
\begin{align}
\int_{\xi \in \Xi} Q_{\theta}(\RG(\xi)) \left( \nabla_{\theta} \log \frac{
    Q_\theta(\RG(\xi)) }{P(\RG(\xi) ,\AG)}
 \right) d \xi 
&=
\int_{\xi \in \Xi} Q_{\theta}(\RG(\xi)) \left( \nabla_{\theta} \log Q_\theta(\RG(\xi))
 \right) d \xi \\
&=
\int_{\xi \in \Xi} \nabla_{\theta} Q_\theta(\RG(\xi))
d \xi \\
&=
\nabla_{\theta} \int_{\xi \in \Xi} Q_\theta(\RG(\xi))
d \xi \\
&= 0
\end{align}
Thus for a given setting of \( \theta^{(t)} \), we can approximate the gradient
\( \nabla_{\theta} f(\theta) |_{\theta^{(t)}} \) using Monte Carlo, by forward-sampling
  trajectories from \( Q_{\theta^{(t)}} \). This is the scheme
  proposed in \citep{wingateweber2013}: 
\begin{enumerate}
\item
Sample a trajectory from \( Q_{\theta^{(t)}} \).
\item
Calculate the noisy gradient estimate \( \delta^{(t)} \).
\item
Take a noisy gradient descent step \( \theta^{(t+1)} := \theta^{(t)} -
\alpha^{(t)} \delta^{(t)} \).
\end{enumerate}

\subsubsection{As a proposal for Metropolis-Hastings}

We can solve this optimization problem as above to construct a proposal distribution \(
Q_{\theta^\star} \), and then sample \( \xi \) from it and accept or reject it
as a single M-H proposal. Note that while solving the optimization problem, we
did not inspect any nodes that were in \( \rho \) but that might not be in \(
\xi \); therefore we can use the same \( Q_{\theta^\star} \) to compute the
probability of the reverse transition.
\todo[color=yellow]{[small] Carefully clarify contribution beyond wingateweber and blei above and below (generality; actual algorithm)}

\subsubsection{Using regen and detach}

The stochastic regeneration recursions can be easily extended to carry additional metadata that enables them to evaluate local gradients for each random choice, and store this information as a local kernel that modifies additional state capturing variational parameters. Once this is done, it turns out that both the stochastic gradient optimization of the variational approximation and the use of it as a proposal can be expressed using stochastic regeneration:

\begin{enumerate}
\item Construct \( \DG \).

\item Detach \( \rho \).

\item For every application node in \( \DG \) whose operator cannot change and that can compute
  gradients, initialize variational parameters to the current arguments.

\item Loop until satisfied with the quality of the variational approximation:

\begin{enumerate}
\item Attach variational kernels to \( \DG \) to define \( Q_{\theta^{(t)}} \).

\item Call \textbf{regen} to sample \( \xi \sim Q_{\theta^{(t)}} \), and to compute 
\[\log \left( \frac{P(\xi,\AG)}{Q_{\theta^{(t)}}(\RG(\xi))} \right) \]
as well as the local gradients \( \{ \nabla_\theta \log Q_{\theta^{(t)}}(\RG(\xi))
\} \).

\item Take a gradient step.

\item Detach.
\end{enumerate}

\item \textbf{restore} \( \rho \) to compute \( \alpha_\rho := \frac{P(\RG(\rho),\AG)}{Q_{\theta^\star}(\RG(\rho))} \).

\item \textbf{detach} 
\item \textbf{regen} to sample \( \xi \) and compute \( \alpha_\xi := \frac{P(\RG(\xi),\AG)}{Q_{\theta^\star}(\RG(\xi))} \).

\item Accept if 
 \[ U \sim \mathbf{Uniform}(0,1) < 
\frac{P(f(\xi) = i)}{P(f(\rho) = i)} 
\cdot
\frac
{\alpha_\xi}{\alpha_\rho}
\]

\end{enumerate}


\todo[color=yellow]{[small] Add CLRS for variational}
\todo[color=yellow]{[small] Elaborate a bit and cite Nando de Freitas and Wingate and Weber
  here and Dave blei}
\todo[color=green]{[medium] Add variational CLRS}
\todo[color=yellow]{[small] Brief note on stochastic approximations and austerity/etc}
\todo[color=green]{[medium] Write reasonable section on embedding stochastic approximations based on TEX comment}




\section{Particle-based Inference: Enumerative Gibbs and Particle Markov chain Monte Carlo}

Many inference strategies involve weighing multiple alternative states against one another as a way of approximating optimal Gibbs-type transition operators over portions of the hypothesis space. Gibbs sampling on a discrete variable in a Bayesian network can be implemented by enumerating all possible values for the variable, substituting each into the network, evaluating the joint probability of each network configuration, then renormalizing and sampling. Normalized importance sampling for approximating Gibbs involves proposing multiple values for a variable from its prior, weighting by its likelihood, and then normalizing over and sampling from the sampled set. Each of these techniques can be applied iteratively to larger subsets of variables. For example, discrete Gibbs can be iterated over a sequence of variables to generate a proposal from a single-particle particle filter, or extended via enumeration (and perhaps dynamic programming) to larger subsets of random variables. Normalized importance sampling can be also iterated to yield sequential importance sampling with resampling, and embedded within a Markov chain to yield particle Markov Chain Monte Carlo techniques. 

It can be useful to think of these inference strategies as instantiations of the idea of weighted speculative execution technique within inference programming. An ensemble of executions is represented, evolved and stochastically weighted and selected via the absorbing nodes, although ultimately only one of these executions will be used.

We refer to each alternative state as a particle, and implement Gibbs sampling, particle Markov chain Monte Carlo, and other techniques using general machinery for particle-based inference methods. This allows us to develop common techniques for handling the dependent random choices in the brush --- both in the proposal and in the analysis --- that apply to all these inference strategies.

\subsection{In-place mutation versus simultaneous particles in memory}

Particle methods also provide new opportunities for time-space tradeoffs in inference. Although the semantics of particle methods involve multiple alternative traces, it may not always be desirable, or even possible, to represent alternate traces in memory simultaneously. Consider scaffolds whose definite regeneration graph or brush invokes a complex external simulator. If we cannot modify the code of this simulator, or it relies on an expensive external compute resource that cannot be multiplexed, it will be necessary to implement particle methods via mutation of a single PET in place. Otherwise, it is likely to be more efficient to represent all particles in memory simultaneously, sharing a common PET and scaffold as their source, and cloning the auxiliary states of stochastic procedures that are referred to from multiple particles. This choice also enables the parallel evolution of collections of particles via multiple independent threads, synchronizing whenever resampling occurs. 

Venture supports both in-place mutation and simultaneous representation of particles. This is necessary to recover asymptotic scaling that is competitive with custom sequential Monte Carlo techniques. In this paper, we focus on the transition operators associated with particle methods, and give pseudocode for implementing them in terms of a generic API for simultaneous particles. The data structures needed to implement this API efficiently, and the algorithm analysis for efficient implementations, involve novel applications of techniques from functional programming and persistent data structures and are beyond the scope of this paper.

\todo[color=green]{[medium] Section for scaffold + ordering control}

\subsection{The Mix-MH operator for constructing particle-based transition operators}

Although the auxiliary variable technique introduced previously is sufficient for lifting a local proposal into a global Metropolis-Hastings proposal, it will be helpful to introduce a variant of the technique that is closed under application --- i.e. the kernel that results from applying the technique is a valid proposal kernel that can be parameterized, randomly selected, and used as a proposal in a subsequent application. This will give us significant additional flexibility in explaining and justifying our particle-based inference schemes.

\subsubsection{Metropolis-Hastings}

Suppose we have some distribution \( \pi \) of interest on \( \Omega \).

Let \( \K \) be a proposal kernel on \( \Omega \). We can create a new kernel \( \MH(\K,\pi) \) based on it that satisfies detailed balance with respect to \( \pi \) by the Metropolis-Hastings rule. Instead of representing it as a black-box kernel, we will keep around its constituent elements, so we have
\begin{align}
\MH(\K,\pi) &= (\K, \lambda)
\end{align}
where 
\[ \lambda(\rho \to \xi) = \pi(\xi) \K(\xi \to \rho) \]
and where applying \( (\K, \lambda) \) involves:

\begin{enumerate}
\item Sample \( \xi \) from \( \K(\rho,\xi) \).

\item Let $U \sim [0,1]$. Accept the new trace $\xi$ iff
\[ U < \frac{\lambda(\rho \to \xi)}{\lambda(\xi \to \rho)} \]
\end{enumerate}

It is straightforward to show that this transition operator preserves detailed balance with respect to $\pi$, and therefore leaves $\pi$ invariant. Ergodic convergence can be proved under fairly mild assumptions \citep{andrieu2003introduction}.

\todo[color=green]{[small] Write out full detailed balance proof for MH, with careful treatment of ergodicity, as a reference to non-ML readers; include discussion of non-ergodicity of boosted acceptance rates.}

\subsubsection{State-dependent Mixtures of Kernels}

We would like to be able to construct compound kernels that are themselves stochastically chosen, and certify detailed balance of the compound kernel given easily verified conditions on the component kernel(s) and the stochastic kernel selection rule.


Let $P()$ be a density of interest, such as the conditioned density on resampling nodes given the absorbing nodes in a scaffold. Let \( I \) be a (possibly infinite) set, and \( f(\rho) = i \) be a stochastic function from traces to \( I \), with $P_f(i)$ the density of $i$ induced by $f$. We will use $f(\rho)$ to select a kernel $\Q_i = (\K_i, \alpha_i)$ at random given the trace $\rho$ (and also evaluate $P_f(f(\xi) = i)$ for reversing the transition). Let $\K_i(\rho \to \xi)$ be a transition operator on traces that can be simulated from, and let $\alpha_i(\rho \to \xi)$ be a procedure that evaluates the factors out of which the Metropolis-Hastings acceptance ratio is constructed:
$$
\alpha_i(\rho \to \xi) = P(\xi) \K_i(\xi \to \rho)
$$
Consider the following kernel, parameterized by scalars $\lambda_i(\rho \to \xi)$ to be determined:
\begin{enumerate}
\item Sample $i \sim f(\rho)$. This will be used to choose a kernel $\K_i$.

\item Sample $\xi \sim \K_i(\rho \to \xi)$.

\item Let $U \sim [0,1]$. Accept $\xi$ iff
\[ U < \frac{\lambda_i(\rho \to \xi)}{\lambda_i(\xi \to \rho)} \]
\end{enumerate}

Our goal is to define $\lambda_i(\rho \to \xi)$ such that this kernel preserves detailed balance, analogously to the Metropolis-Hastings kernels described above. We assume that the stochastic function $f$ satisfies the following symmetry condition:
\[ P_f(f(\rho) = i) \K_i(\rho \to \xi) > 0 \iff P_f(f(\xi) = i)  \K_i(\xi \to \rho) > 0 \]
This expresses the constraint that if you can reach $\xi$ from $\rho$ by selecting $\Q_i$ according to $f(\rho)$ then applying it, then with nonzero probability, reaching $\rho$ from $\xi$ is possible by selecting the same $\Q_i$ according to $f(\xi)$ then applying it.
Let \( I(\rho) = \{ i \in I : P_f(f(\rho) = i) > 0 \} \) be the set of indices (kernels) reachable from $\rho$, and \( I(\rho,\xi) = I(\rho) \cap I(\xi) \) be the set of indices (kernels) reachable from both $\rho$ and $\xi$. To establish detailed balance for this kernel, it suffices to show that
\begin{align}
P(\rho) \sum_{i \in I(\rho)} P_f(f(\rho) = i) \K_i(\rho \to \xi) \lambda_i(\rho \to \xi)
& = 
P(\xi) \sum_{i \in I(\xi)} P_f(f(\xi) = i) \K_i(\xi \to \rho) \lambda_i(\xi \to \rho)
\end{align}
Due to our symmetry condition, $i \in I(\rho)$ but $i \notin I(\xi)$ implies that $\K_i(\rho \to \xi) = 0$. Therefore detailed balance is satisfied if and only if the following holds:
\begin{align}
P(\rho) \sum_{i \in I(\rho,\xi)} P_f(f(\rho) = i) \K_i(\rho \to \xi) 
 \lambda_i(\rho \to \xi)
& = 
P(\xi) \sum_{i \in I(\rho,\xi)} P_f(f(\xi) = i) \K_i(\xi \to \rho) 
\label{eqn:mixmhdb}
 \lambda_i(\xi \to \rho)
\end{align}
By hypothesis, we have that each \( \Q_i = (\K_i, \alpha_i) \) satisfies detailed balance with respect to $P$:
\begin{align}
P(\rho)  \K_i(\rho \to \xi) \alpha_i(\rho \to
\xi) 
& = 
P(\xi) \K_i(\xi \to \rho) \alpha_i(\xi \to
\rho) 
\qquad \text{for all \( i \)}
\label{eqn:mixmhdb2}
\end{align}
Let $\lambda_i(\rho \to \xi) = \alpha_i(\rho \to \xi) P_f(f(\xi) = i)$. If we multiply both sides of (\ref{eqn:mixmhdb2}) by $P_f(f(\rho) = i)P_f(f(\xi) = i)$ then sum both sides over $i \in I(\rho, \xi)$, we establish (\ref{eqn:mixmhdb}).


\subsubsection{The MixMH operator}


We define the $\Mixmh$ operator as follows. $\Mixmh$ takes as input stochastic index sampler $f$ and a set of base kernels $\{K_i, \alpha_i\}_{i \in I}$ parameterized by the index $i$ sampled from the index sampler:
$$
\Mixmh(f , \left \{ (\K_i, \alpha_i) \right \}_{i \in I})$$
Now let  
$\lambda_i(\rho \to \xi) = \alpha_i(\rho \to \xi) P(f(\xi) =i) $. To apply $\Mixmh(f , \left \{ (\K_i, \alpha_i) \right \}_{i \in I})$ we:

\begin{enumerate}
\item Sample \( i \) from \( f(\rho) \).

\item Apply $(\K_i, \lambda_i)$.
\end{enumerate}


\subsection{Particle-based Kernels}

\todo[color=red]{[medium] Incorporate Cameron feedback on particle section and argument structure}

Our goal here is to define and sketch correctness arguments for particle-based transition operators. We first describe particle sets obtained from repeated applications of a single kernel. Our formulation incorporates a boosted acceptance ratio that is designed to recover Metropolis-Hastings in the special case where there is only one new particle generated.\footnote{This boosting comes at the cost of subtle ergodicity violations if this is the only transition operator used and the hypothesis space contains the right symmetries. Simpler particle-based analyses --- where all particles, including the current trace, are resampled from --- are possible and indeed straightforward. These yield Boltzmann-style acceptance rules that self-transition more frequently and therefore explore the space less efficiently, but avoid ergodicity issues.}

\subsubsection{Generating particles by repeatedly applying a seed kernel}


\todo[color=yellow]{[small] Recall definitions from scaffold section here}
Suppose we have a \( (\torus, \DG) \) pair, and we attach local simulation
kernels to the nodes in \( \DG \) to define a distribution \( \K \) on \(
\Xi[\torus,\DG] \). By \eqref{eq:regenSimulation}, \( \regenAndAttach \) will return
\[ w_\xi = \frac{P( \RG(\xi) , \AG)}{\K(\xi)} \]
Let \( \PS = \{ (\xi_i,w_i) : i = 1 , \dotsc, n \} \) be a multiset of weighted
particles in \( \Xi \). We can define the kernel \( \K_\PS(\cdot \to \cdot) \)
on \( \PS \) as follows:
\[ \K_\PS(\xi_i \to \xi_j) = \frac{\nds(\xi_j,i) w_j}{w_{- i}} \]
where \(\nds(\xi_j, i) \) is the number of duplicates of \( \xi_j \) in \( \PS
\setminus \xi_i \).The \( \alpha \)-factor for \( \K_\PS \) in isolation
is:

\begin{align}
\frac
{\alpha(\xi_i \to \xi_j)}
{\alpha(\xi_j \to \xi_i)}
&=
\frac
{P(\xi_j) \K_\PS(\xi_j \to \xi_i)}
{P(\xi_i) \K_\PS(\xi_i \to \xi_j)} 
\\
&=
\frac
{P(\RG(\xi_j),\AG) \frac{ \nds(\xi_i,j) w_i}{w_{- j}}}
{P(\RG(\xi_i),\AG) \frac{ \nds(\xi_j,i) w_j}{w_{- i}}}
\end{align}

We can sample the kernel \( \PS \) starting from \( \xi_i
\) by generating \( n - 1 \) additional particles from \( \K_\regen \). The
probability of generating \( \PS \) is the probability of sampling the \( n -1 \) other
particles, multiplied by the number of distinct orders in which they could have
been sampled:
\begin{align}
P(\PS | \xi_i)
&=
\frac
{(n - 1)!}
{\prod_{\gamma \in \unique(\PS \setminus \xi_i)} \nds(\gamma,i)!}
\prod_{k \neq i} \K_\regen(\xi_k)
\end{align}

\begin{claim}
\[ \frac{P(\PS | \xi_j)}{P(\PS | \xi_i)} = 
\frac{\nds(\xi_j,i)}{\nds(\xi_i,j)} 
\frac
{\prod_{k \neq j} \K_\regen(\xi_k)}
{\prod_{k \neq i} \K_\regen(\xi_k)}
\]
\end{claim}

\begin{proof}
\begin{align}
\frac
{P(\PS | \xi_j)}
{P(\PS | \xi_i)}
&=
\frac
{
\frac
{(n - 1)!}
{\prod_{\gamma \in \unique(\PS \setminus \xi_j)} \nds(\gamma,j)!}
\prod_{k \neq j} \K_\regen(\xi_k)
}
{
\frac
{(n - 1)!}
{\prod_{\gamma \in \unique(\PS \setminus \xi_i)} \nds(\gamma,i)!}
\prod_{k \neq i} \K_\regen(\xi_k)
}
\\
&=
\frac
{
\prod_{\gamma \in \unique(\PS \setminus \xi_i)} \nds(\gamma,i)!
\prod_{k \neq j} \K_\regen(\xi_k)
}
{
\prod_{\gamma \in \unique(\PS \setminus \xi_j)} \nds(\gamma,j)!
\prod_{k \neq i} \K_\regen(\xi_k)
}
\\
&=
\frac
{
\nds(\xi_j,i)! \nds(\xi_i,i)!
\prod_{k \neq j} \K_\regen(\xi_k)
}
{
\nds(\xi_i,j)! \nds(\xi_j,j)!
\prod_{k \neq i} \K_\regen(\xi_k)
}
\\
&=
\frac{\nds(\xi_j,i)}{\nds(\xi_i,j)} 
\frac
{\prod_{k \neq j} \K_\regen(\xi_k)}
{\prod_{k \neq i} \K_\regen(\xi_k)}
\end{align}
\end{proof}

Thus once we apply \( \Mixmh \), the \( \alpha \)-factor becomes:

\begin{align}
\frac
{\alpha(\xi_i \to \xi_j)}
{\alpha(\xi_j \to \xi_i)}
&=
\frac
{P(\RG(\xi_j),\AG) \frac{ \nds(\xi_i,j) w_i}{w_{- j}} P(\PS | \xi_j)}
{P(\RG(\xi_i),\AG) \frac{ \nds(\xi_j,i) w_j}{w_{- i}} P(\PS | \xi_i)}
\\
&=
\frac
{P(\RG(\xi_j),\AG) \frac{ \nds(\xi_i,j) w_i}{w_{- j}} \nds(\xi_j, i) \prod_{k \neq j} \K(\xi_k)}
{P(\RG(\xi_i),\AG) \frac{ \nds(\xi_j,i) w_j}{w_{- i}} \nds(\xi_i,j) \prod_{k \neq i} \K(\xi_k)}
\\
&=
\frac
{P(\RG(\xi_j),\AG) \frac{w_i}{w_{- j}} \prod_{k \neq j} \K(\xi_k)}
{P(\RG(\xi_i),\AG) \frac{w_j}{w_{- i}} \prod_{k \neq i} \K(\xi_k)}
\\
&=
\frac
{P(\RG(\xi_j),\AG) \frac{\left( \frac{P( \RG(\xi_i) , \AG)}{\K(\xi_i)} \right) }{w_{- j}} \prod_{k \neq j} \K(\xi_k)}
{P(\RG(\xi_i),\AG) \frac{\left( \frac{P( \RG(\xi_j) , \AG)}{\K(\xi_j)} \right) }{w_{- i}} \prod_{k \neq i} \K(\xi_k)}
\\
&=
\frac
{w_{-i}}
{w_{-j}}
\end{align}

\subsubsection{The MH$_n$ operator}

This pattern, where a multiset of weighted particles is generated from a ``seed'' kernel $\K$ and then sampled from, is a sufficiently common operation that we give it a name:
\[ \MHn(\K) = \Mixmh(\PS \sim \K, \K_\PS) \]
We will use variations on it throughout this section. The multiset functions as the index for $\Mixmh$, and the sampling step where a particle is chosen functions as the base proposal kernel.

\subsubsection{Metropolis-Hastings as special case of particle methods}

Metropolis-Hastings with a simulation kernel is actually just the special case \( \MH_2 \), since in this case \( w_{-\rho} = w_\xi \).  Parallelizable extensions of this kind of ``locally independent'' Metropolis-Hastings scheme are also natural. For example, one could make multiple proposals and weigh them against against one another. This could be viewed as an importance sampling approximation to an optimal Gibbs proposal over a scaffold, where the proposal can be any simulation kernel, including one that conditions on downstream information. That said, the restriction to simulation kernels is significant. Gaussian drift kernels, for example, are not permitted.

\subsubsection{Enumerative Gibbs as a special case: using different kernels for each particle}

There are many ways to generate weighted particle sets \( \PS \). For instance,
we may want to allow a different kernel \( \K_i \) for each particle \( \xi_i \). We can
represent enumeration over some set of variables this way: each distinct tuple in the Cartesian product of the domains of the variables will correspond to a different kernel.

Suppose there are some variables in \( \DG \) that we want to enumerate over,
yielding \( n \) total combinations. Assume that \( \rho \) has the first
combination. We can sample \( \PS \) from \( \rho \) as
follows: call \( \regen \) \( n - 1 \) times as before to generate \( n - 1 \)
new particles, but passing \( \K_i \) for \( \xi_i \), where \( \K_i \) is the
same as \( \K \) except with deterministic kernels replacing the old kernels on
the nodes that we are enumerating.

The analysis is even simpler in this case because the particles are distinct. Once we apply \( \Mixmh \), the \( \alpha \)-factor becomes:

\begin{align}
\frac
{\alpha(\xi_i \to \xi_j)}
{\alpha(\xi_j \to \xi_i)}
&=
\frac
{P(\xi_j) \K_\PS(\xi_j \to \xi_i)  \prod_{k \neq j} \K_k(\xi_k)}
{P(\xi_i) \K_\PS(\xi_i \to \xi_j) \prod_{k \neq i} \K_k(\xi_k)}
\\
&=
\frac
{P(\RG(\xi_j), \AG) \frac{w_i}{w_{- j}}  \prod_{k \neq j} \K_k(\xi_k)}
{P(\RG(\xi_i), \AG) \frac{w_j}{w_{- i}}  \prod_{k \neq i} \K_k(\xi_k)} \\
&=
\frac
{P(\RG(\xi_j), \AG)  \frac{\left( \frac{P( \RG(\xi_i) , \AG)}{\K(\xi_i)} \right)}{w_{- j}} \prod_{k \neq j} \K_k(\xi_k)}
{P(\RG(\xi_j), \AG) \frac{\left( \frac{P( \RG(\xi_j) , \AG)}{\K(\xi_j)} \right)}{w_{- i}} \prod_{k \neq i} \K_k(\xi_k)} \\
&=
\frac{w_{-i}}{w_{-j}}
\end{align}

as above.

\subsubsection{Particle Markov chain Monte Carlo: adding iteration and resampling}

It is often useful to iterate particle-based methods, and to resample collections of particles based on their weights, stochastically allocating particles to regions of the execution space that appear locally promising. This insight is at the heart of particle filtering, i.e. sequential importance sampling with resampling \citep{Doucet01}, as well as more sophisticated sequential Monte Carlo techniques. For example, \citep{andrieu2010particle} introduces particle Markov chain Monte Carlo methods, a family of techniques for using sequential Monte Carlo to make sensible proposals over subsets of variables as part of larger MCMC schemes.

\todo[color=red]{[small] Clarify notation here and fix exposition}
We can also use conditional SMC to generate \( \PS \). As above, suppose we have
\( \rho \) and \( \DG \), with local simulation kernels attached to the nodes in
\( \DG \), which determines a simulation kernel \( \K \) that generates samples
by propagating along \( \DG \) via \( \regen \). Suppose we group the sinks into
\( T \) groups, which in turn partitions all of \( \DG \) into \( T \)
groups \( \DG_1, \dotsc, \DG_T \) according to the \( \regen \) recursion. We
will refer to \( \K \) on \( \RG_t \) as \( \K_t \).

To establish notation, first let us consider a simpler case: basic sequential Monte Carlo, instead of conditional SMC. We
can propose along \( \RG_1 \) from \( \K_1 \), and then (inductively) 
propagate to time \( t \) by sampling independently from a proposal kernel,
\[ M_t(a^{(m)}_t,\RG_t(\xi_t^{(m)})) = \frac{w_{t-1}^{(m)}}{w_{t-1}}
\K_t(\RG_t(\xi_t^{(m)})) \]
where 
\[ w_{t-1}^{(m)} = \frac{P(\RG_{t-1}(\xi_{t-1}^{(m)}),
  \AG_{t-1})}{\K_{t-1}(\xi_{t-1}^m)} \]
is the return value from \( \regen \) on the \( t -1 \)st group of sinks, and where \( \RG_t(\xi_t^{(m)}) \) is assumed to be conditioned on \( \xi_t^{(m)} \)'s parent particle \( \xi_{t-1}^{(a_t^{(m)})} \).

It is instructive to view this sampling procedure as a way of generating a
single sample of all auxiliary variables \( \Gamma \) from the density
\begin{align}
\psi(\Gamma)
&=
\prod_{m=1}^n \K_1(\RG_1(\xi_1^{(m)})) \prod_{t=2}^T \prod_{m=1}^n M_t(a_t^{(m)}, \xi_t^{(m)})
\\
&= 
\prod_{m=1}^n \K_1(\RG_1(\xi_1^{(m)})) \prod_{t=2}^T \prod_{m=1}^n \frac{w_{t-1}^{(m)}}{w_{t-1}} \K_t(\RG_t(\xi_t^{(m)}))
\\
&= 
\prod_{m=1}^n \K_1(\RG_1(\xi_1^{(m)})) \prod_{t=2}^T \prod_{m=1}^n \frac{P(\RG_{t-1}(\xi_{t-1}^{(m)}),\AG_{t-1}) \K_t(\RG_t(\xi_t^{(m)}))}{\K_{t-1}(\RG_{t-1}(\xi_{t-1}^{(m)})) w_{t-1}}
\\
&=
\left( \prod_{t=2}^T \prod_{m=1}^n \frac{1}{w_{t-1}} \right)
\prod_{m=1}^n \K_1(\RG_1(\xi_1^{(m)})) \prod_{t=2}^T \prod_{m=1}^n
\frac{P(\RG_{t-1}(\xi_{t-1}^{(m)}),\AG_{t-1})
  \K_t(\RG_t(\xi_t^{(m)}))}{\K_{t-1}(\RG_{t-1}(\xi_{t-1}^{(m)}))}
\end{align}

Now suppose we perform the conditional SMC sweep to sample \(
\Gamma \). First we sample an index \( b_t \in \{ 1, \dotsc, n \} \) for each time step \( t =
1, \dotsc, T-1 \) for the index of the forced resampling of the source particle \(
\xi_i \). Then we sample the other particles as in standard CSMC. 

The probability of \( \Gamma \) starting from \( \xi_i \) is precisely:

\begin{align}
P(\Gamma ; \xi_i)
&=
\left( \prod_{t=2}^T w_{t-1} \right)
\frac
{\psi(\Gamma)}
{n^{T-1} \K_1(\RG_1(\xi_1^{(b_1)})) \prod_{t=2}^T w_{t-1}^{(b_t)} \K_t(\RG_t(\xi_t^{(b_t)}))}
\\
&=
\left( \prod_{t=2}^T w_{t-1} \right)
\frac
{w_T^{(i)} \psi(\Gamma) }
{n^{T-1} P(\RG(\xi^{(i)}),\AG)}
\end{align}
 
However, we want to index our kernel with multiset semantics for the complete
particles as we did above, and so we index by the equivalence class \(
\overline{\Gamma} \) where \( \Gamma \equiv \Gamma' \) if they agree on
everything up until the last time step, and then agree on the same multiset of
complete particles. We have

\[ P(\overline{\Gamma} ; \xi_i) = P(\Gamma ; \xi_i ) \frac
{(n - 1)!}
{\prod_{\gamma \in \unique(\Gamma^{(T)} \setminus \xi_i)} \nds(\gamma,i)!}
 \]

and thus when we apply \( \Mixmh \), the \( \alpha \)-factor becomes\footnote{When it is clear we are referring to the final weight, we will refer to \( w_T^{(i)} \) as \( w_i \) to be consistent with the previous sections.}

\begin{align}
\frac
{\alpha(\xi_i \to \xi_j)}
{\alpha(\xi_j \to \xi_i)}
&=
\frac
{P(\xi_j) \K_{\overline{\Gamma}}(\xi_j \to \xi_i) P( \overline{\Gamma} | \xi_j)}
{P(\xi_i) \K_{\overline{\Gamma}}(\xi_i \to \xi_j) P( \overline{\Gamma} | \xi_i)}
\\
&=
\frac
{P(\xi_j) \K_{\overline{\Gamma}}(\xi_j \to \xi_i) \nds(\xi_j,i) P( \Gamma | \xi_j)}
{P(\xi_i) \K_{\overline{\Gamma}}(\xi_i \to \xi_j) \nds(\xi_i,j) P( \Gamma | \xi_i)}
\\
&=
\frac
{P(\xi_j) \K_{\overline{\Gamma}}(\xi_j \to \xi_i) \nds(\xi_j,i) P(\RG(\xi^{(i)}),\AG) w_j}
{P(\xi_i) \K_{\overline{\Gamma}}(\xi_i \to \xi_j) \nds(\xi_i,j) P(\RG(\xi^{(j)}),\AG) w_i}
\\
&=
\frac
{\K_{\overline{\Gamma}}(\xi_j \to \xi_i) \nds(\xi_j,i) w_j}
{\K_{\overline{\Gamma}}(\xi_i \to \xi_j) \nds(\xi_i,j) w_i}
\\
&=
\frac
{\frac{ \nds(\xi_i,j) w_i}{w_{- j}} \nds(\xi_j,i) w_j}
{\frac{ \nds(\xi_j,i) w_j}{w_{- i}} \nds(\xi_i,j) w_i}
\\
&=
\frac
{w_{-i}}
{w_{-j}}
\end{align}

We can think of this as a strict generalization of \( \MHn \) where the kernel is
split into a sequence of kernels. 

This scheme is called {\tt PGibbs} in the Venture inference programming language, as it enables the approximation of blocked Gibbs sampling over arbitrary scaffolds. Cycles and mixtures of {\tt PGibbs} with other kernels recovers a wide range of existing particle Markov chain Monte Carlo schemes as well as novel algorithms.

\subsubsection{Pseudocode for {\tt PGibbs} with simultaneous particles}

Here we give pseudocode for implementing the {\tt PGibbs} transition operator using an interface for simultaneous particles. The particle interface permits particles to be constructed from source PETs or from one another, sharing the maximum amount of state possible. Stochastic regeneration is used to prepare the source trace $\rho$ for conditional SMC, to populate the array of particles, and to calculate weights.

\begin{codebox}
\Procname{\(\proc{PGibbs}(\id{trace}, \id{border}, \id{scaffold}, \id{P})\)}
\li \( \id{T} \gets \attrib{border}{length} \)
\li \( \id{rhoWeights} \gets [T] \)
\li \( \id{rhoDBs} \gets [T] \)
\li \For \( \id{t} \gets T \) \To \( \id{1} \)
\li \Do
      \( \id{rhoWeights}[\id{t}], \id{rhoDBs}[\id{t}]
      \gets \proc{DetachAndExtract}(\id{trace},\id{border}[\id{t}],\id{scaffold}) \)
    \End
\li \( \id{particles} \gets [P] \)
\li \( \id{particleWeights} \gets [P] \)
\li \For \( \id {p} \gets 0 \) \To \( \id{P} \)
\li \Do
      \( \id{particles}[\id{p}] \gets \proc{Particle}(\id{trace}) \)
    \End
\li \( \id{particleWeights}[0] \gets \)
    \Indentmore
\zi   \( \proc{RegenerateAndAttach}(\id{particles}[0],\id{border}[1],\id{scaffold},\const{true},\id{rhoDBs}[1]) \)
    \End
\li \For \( \id{p} \gets 1 \) \To \( \id{P} \)
\li \Do
      \( \id{particleWeights}[\id{p}] \gets \)
      \Indentmore
\zi     \( \proc{RegenerateAndAttach}(\id{particles}[p],\id{border}[1],\id{scaffold},\const{false},\const{nil}) \)
      \End
    \End
\li \( \id{newParticles} \gets [P] \)
\li \( \id{newParticleWeights} \gets [P] \)
\li \For \( \id{t} \gets 2 \) \To \( \id{T} \)
\li \Do
      \( \id{newParticles}[0] \gets \proc{Particle}(\id{particles}[0]) \)
\li   \( \id{newParticleWeights}[0] \gets \)
      \Indentmore
\zi     \( \proc{RegenerateAndAttach}(\id{newParticles}[0],\id{border}[\id{t}],\id{scaffold},\const{true},\id{rhoDBs}[\id{t}]) \)
      \End
\li   \For \( \id{p} \gets 1 \) \To \( \id{P} \)
\li   \Do
        \( \id{parentIndex} \gets \proc{SampleCategorical}(\proc{MapExp}(\id{particleWeights})) \)
\li     \( \id{newParticles}[\id{p}] \gets \proc{Particle}(\id{particles}[\id{parentIndex}]) \)
\li     \( \id{newParticleWeights}[\id{p}] \gets \)
        \Indentmore
\zi       \( \proc{RegenerateAndAttach}(\id{newParticles}[p],\id{border}[\id{t}],\id{scaffold},\const{false},\const{nil}) \)
        \End
      \End
\li \( \id{particles} \gets \id{newParticles} \)
\li \( \id{particleWeights} \gets \id{newParticleWeights} \)
    \End
\li \( \id{finalIndex} \gets \proc{SampleCategorical}(\proc{MapExp}(\id{particleWeights}[1:\id{P}])) \)
\li \( \id{weightMinusXi} \gets \proc{LogSumExp}(\attrib{particleWeights}{\proc{Remove}}(\id{finalIndex})) \)
\li \( \id{weightMinusRho} \gets \proc{LogSumExp}(\attrib{particleWeights}{\proc{Remove}}(0)) \)
\li \( \id{alpha} \gets \id{weightMinusRho} - \id{weightMinusXi}\)
\li \Return \( \id{particles}[\id{finalIndex}], \id{alpha}\)
\end{codebox}

This implementation illustrates the use of versions of $\regen$ and $\detach$ that act on particles, and the initialization of particles using parent particles as well as source traces. All but the last four lines are simply constructing the randomly chosen kernel to be applied, corresponding to proposing to replace the current race with the contents of a particle.


\section{Conditional Independence and Parallelizing Transitions}

Here we briefly describe the conditional independence relationships that are easy to extract from probabilistic execution traces. This analysis clarifies the potential long-range dependencies in a PET introduced by the choice to make probabilistic closures into first-class objects in the language. These independence relationships could be used to support probabilistic program analyses and also to justify the correctness of parallelized kernel composition operators in the inference programming language. They also expose parallelism that is distinct from particle-level parallelism. 

\subsection{Markov Blankets, Envelopes, and Conditional Independence}

In Bayesian networks, the Markov Blanket of a set of nodes provides a useful
characterization of important conditional
independencies. These independencies permit parallel simulation of Markov chain
transition operators. In constrast, while the scaffold provides a useful
factorization of the logdensity of a trace, it does not permit parallel
simulation of transition operators. In this section, we show how to formulate a
notion of locality that permits parallel transitions on PETs.

We first review Markov Blankets in Bayesian networks. Let \( \DG \) and $\DG'$ be sets of
nodes in a Bayesian network, \( \AG \) the children of $\DG$, 
and \( \PG \) the parents of \( \DG \cup \AG \) excluding \( \DG \) and \( \AG
\). Then if \( \DG' \cap (\DG \cup \AG \cup \PG) = \emptyset \), proposals can
be made on the two regeneration graphs in parallel, although both proposals may
\emph{read} the same values. We refer to \( \AG \cup \PG \) as the \emph{Markov
  blanket} of \( \DG \), and it can be thought of as the set of all nodes that
one must read from (but not write to) when resampling \( \DG \) and computing
the new probabilities of its children \( \AG \). 

The situation becomes more complicated in the case of PETs for two main
reasons. First, the parents \( \PG = \PG(\rho) \cup \PG(\xi) \) are not known a
priori, since we do not know \( \PG(\xi) \) until we have simulated \( \xi \) to
see what nodes it reads from. Second, the SP auxiliary states that will be mutated are not 
known a priori, since the PSPs being applied in \( \RG(\xi) \) are not known.

\subsection{After-the-fact envelopes}

Let \( \SG(\rho) \) denote the SP auxiliary states for SPs with PSPs that are
applied in \( \RG(\rho) \). Define

\begin{align}
\MB(\rho)
&= \PG(\rho) \cup \AG
\end{align}
to be the set of all nodes that are read from but not written to, as before for Bayesian networks. Now define the scope of a proposal
\begin{align}
\Scope(\rho)
&= \RG(\rho) \cup \SG(\rho)
\end{align}
to be the set of all parts of the PET that are written to or created while generating \( \rho \) from
\( \torus \), and define the envelope 
\begin{align}
\Envelope(\rho) &= \MB(\rho) \cup \Scope(\rho)
\end{align}
to be the set of all nodes that must be read from or written to during the proposal. Note that
in general there may be some nodes in the scope that are written to but never
read from, e.g. terminal resampling nodes and the sufficient statistics for
uncollapsed SPs, but we will ignore this distinction for now.

We can make the following claim: suppose the proposal \( \rho \to
\xi \) temporally overlaps with the proposal \( \rho' \to \xi' \), and that
by some luck, 
\begin{align}
\left( \Envelope(\rho,\xi) \cap \Scope(\rho',\xi') \right) \cup \left( \Envelope(\rho',\xi')
  \cap \Scope(\rho,\xi) \right) &= \emptyset
\end{align}
then the transitions will not have clashed. In words, this says that the two
proposals will not clash as long as no node that one proposal reads differs
in the two traces of the other proposal. If the transitions have not clashed, then they could have been simulated simultaneously.

\subsection{The envelope of a scaffold}

Define:
\begin{align}
 \Envelope[\DG] &= \bigcup_{\xi \in \Xi[\DG]} \Envelope(\xi) \\
\Scope[\DG] &= \bigcup_{\xi \in \Xi[\DG]} \Scope(\xi)
\end{align}

That is, $\Envelope[\DG]$ and $\Scope[\DG]$ correspond to the union of the trace-specific envelope and scope, taken over all possible completions of $\torus$.  Then two proposals on \( \DG \) and \( \DG' \) are guaranteed not to clash if
\begin{align}
 ( \Envelope[\DG] \cap \Scope[\DG'] ) \cup (\Envelope[\DG'] \cap \Scope[\DG]) &=
  \emptyset  \label{eq:envelopeCriterion}
\end{align}

Now let $\RG$ and $\RG'$ be the random variables corresponding to the results of invocations of $\regen$ (including the weights) on scaffolds with definite regeneration graphs $D$ and $D'$. One approach to formalizing the relationship between conditional independence and parallel simulation in Venture would be to first try to show that
(\ref{eq:envelopeCriterion}) holds if and only if \( \RG \perp \RG' \), and
second to try to show that random variables can safely be simulated in parallel if and only if
they are conditionally independent. 

When probabilistic programmers can identify two scaffolds for which (\ref{eq:envelopeCriterion}) holds, they can schedule transitions on them simultaneously. Also, in some circumstances speculative simultaneous simulation may be beneficial, for example if the clash is on small, cheap procedures or an approximate transition that ignores some dependency can be used as a proposal for a serial transition.

It may also be possible to use this notion of envelope to predict beforehand which transitions we can run simultaneously. To do this, we will need to bound the Markov blanket and the scope. One option would be to statically analyze Venture program fragments, implementing a kind of ``escape analysis'' for random choices. Another would be to introduce language constructs and/or stochastic procedures that came with hints about Markov blanket composition. Future work will explore the efficacy of these and other techniques in practice.

\section{Related Work}

Venture builds on and contains complementary insights to a number of other probabilistic programming languages, implementations, and inference techniques. For example, Monte \citep{montenips} was the first commercially developed prototype interpreter for a Church-like probabilistic programming language. Its architecture is the most direct antecedent of Venture. For example, it included a scaffold-like decomposition of local transitions in terms of ``proposal blankets'' for single-variable changes. Preliminary prototypes of Venture \citep{wumeng} and of an approximate multi-core scheme for Church-like languages \citep{perovmulticore} also introduced early, ad-hoc versions of some of the ideas in this paper, all in an attempt to control the asymptotic cost of each inference transition.

The most salient inefficiency in many Church systems has been that for
typical machine learning problems with $N$ datapoints, a single sweep
of inference over all random variables requires runtime that scales as
$O(N^2)$. This quadratic scaling, combined with the absolute constant
factors, has made it impossible to apply these systems beyond hundreds
of datapoints. For ``lightweight'' implementations based on direct
transformational compilation or augmented interpretation, such as
Bher \citep{wingate2011lightweight}, this is a basic consequence of the approach: the entire program
must be re-simulated to make a change to a single latent variable.

A variety of approaches have been tried to mitigate this problem. For earlier implementations such as the first prototype implementation of Church \citep{Mansinghka:2009} atop the Blaise system \citep{bonawitz2008composable}, the MIT-Church implementation, and also Monte, ad-hoc attempts were made to control the scope of re-simulation, approximating the Venture notions of PETs and their partition into scaffolds. Similarly, the Shred implementation of Church \citep{shredchurch} uses techniques from program analysis coupled with just-in-time compilation to try to avoid the re-simulation involved in lightweight implementations of Metropolis-Hastings and also reduce constant-factor overhead.

These approaches exhibit different limitations with respect to
generality, scalability, absolute achievable efficiency and runtime
predictability. For example, the Blaise implementation of Church and
MIT-Church both exhibited quadratic scaling in some cases, and
involved substantial runtime overheads. The program analysis and
compilation techniques used in Shred incur runtime costs that can be
significant in absolute terms and difficult to predict a priori, and
additionally impose constraints on the integration of custom
primitives into the language. Finally, all of these approaches have
only applied directly to single-site Metropolis-Hastings approaches to
inference.

Outside of Church-like languages, BLOG \citep{milch20071} --- the first open-universe probabilistic programming language --- is of particular relevance. For example, the original BLOG inference engine was based on infinite contingent Bayesian networks (ICBN) \citep{milch2005approximate}. ICBNs characterize the valid factorizations of their probability distribution and represent dependencies whose existence is contingent on the values of other variables. They thus provide a data structure for BLOG that serves an analogous function to PETs in Venture.
Also, the Gibbs sampling algorithm introduced for BLOG \citep{arora2012gibbs} exploits a decomposition of a possible world that is related to the scaffolds used to define the scope of inference in Venture. However, BLOG is not higher-order --- it does not support random variables that are themselves probabilistic procedures --- nor does it provide a programmable inference mechanism. BLOG also does not support collapsed primitives that exhibit exchangeable coupling between applications, primitives that lack calculable probability densities, or primitives that can create and destroy latent variables while providing their own external inference mechanism.

The Figaro system for probabilistic programming \citep{pfeffer2009figaro} shares several goals with Venture, although the differences in language semantics and inference architecture are substantial. Venture is a stand-alone virtual machine, while Figaro is based on an embedded design, implemented as a Scala library. Figaro models are represented as data objects in
Scala, and model construction and inference both proceed via calling into a Scala API. This enables Figaro to leverage the mature toolchain for Scala and the JVM, and avoids the need to reimplement difficult but standard programming language features from scratch. It remains to be seen how to apply this machinery to models that depend on higher-order probabilistic procedures, primitives exhibiting exchangeable coupling, and likelihood-free primitives, or how to enable users to extend Figaro with custom model fragments equipped with arbitrary inference schemes. IBAL \citep{Pfeffer01ibal:a}, an ancestor of Figaro based on an embedding of stochastic choice into the functional language ML, is arguably more similar to Venture in terms of its interface. However, the inference strategies around which IBAL was designed imply restrictions on stochastic choices that are analogous to Infer.NET.

Our hybrid inference primitives embody generalizations of other proposals for global inference in probabilistic programs. The idea of global mean-field inference in probabilistic programs via stochastic optimization, implemented using repeated re-simulation of the program, was first proposed in \citep{wingateweber2013}. Similarly, a global implementation of particle Gibbs via conditional sequential Monte Carlo, implemented via repeated forward re-simulation, has been independently and concurrently developed in \citep{wjwmpgibbs2014}. Both of these approaches were implemented using the ``lightweight'' scheme proposed in \citep{wingate2011lightweight}. They thus lack dependency tracking and will exhibit unfavorable asymptotic scaling. The mean field and conditional SMC schemes developed in this paper are implemented via stochastic regeneration, building on PETs for efficiency and supporting the full SPI. The techniques from this paper can also be used on subsets of the random choices in a probabilistic program and composed with other inference techniques, yielding hybrid inference strategies that may be asymptotically more efficient than homogeneous ones.

To the best of our knowledge, Venture is the first probabilistic programming platform to support a compositional inference programming language with multiple computationally universal primitives. 
It is also the only probabilistic programming system to support higher-order probabilistic procedures as first class objects, and in particular the implementation of arbitrary higher-order probabilistic procedures with external, per-application latent variables as ordinary primitives. Venture is also the only probabilistic programming system that  integrates exact and approximate sampling techniques based on standard Markov chain, sequential Monte Carlo and variational inference.

\todo[color=green]{[medium] Write up and include slice sampling after testing}
\todo[color=green]{[small] Discuss inference strategy precedents, including Blaise, Maude, LTAC, ... (differentiating strategy and tactics languages, thm proving and rewriting logics)}

\section{Discussion}

We have described Venture, along with the key concepts, data structures and algorithms needed to implement a scalable interpreter with programmable inference. Venture includes a common stochastic procedure interface that supports higher-order probabilistic procedures, ``likelihood-free'' procedures, procedures with exchangeable coupling between their invocations, and procedures that maintain external latent variables and supply external inference code. We have seen how the probabilistic execution traces on which Venture is based generalize key ideas from Bayesian networks, and support partitioning traces into scaffolds corresponding to well-defined inference subproblems. We have defined stochastic regeneration algorithms over scaffolds, and shown how to use them to implement multiple general-purpose inference strategies. We have also given example Venture programs that illustrate aspects of the modeling and inference languages that it provides.

The coverage of Venture in terms of models, inference strategies and end-to-end problems needs to be carefully assessed. It remains to be seen whether the current set of inference strategies are truly sufficient for the full range of problems arising across the spectrum from Bayesian data analysis to large-scale machine learning to real-time robotics. However, expanding the set of inference primitives may be  straightforward. One strategy rests on the analogies between PETs and Bayesian networks. For example, it may be possible to define analogues of junction trees for fragments of PETs by conditioning on random choices with existential dependencies. Along similar lines, extensions to the {\tt enumerative\_gibbs} inference instruction could leverage the insights from \citep{koller1997effective}. Variants based on message passing techniques such as expectation propagation might also be fruitful. Other inference techniques that we could build using PETs and the SPI include slice sampling and Hamiltonian Monte Carlo. Of these, Hamiltonian Monte Carlo has already seen some use in probabilistic programming \citep{stan-manual:2013,wingate2011nonstandard}. However, these implementations do not apply to arbitrary Venture scopes, where proposals can trigger the resampling of other random choices due to the presence of ``likelihood-free'' primitives or the brush. Generalizations of Hamiltonian and slice methods that use the auxiliary variable machinery presented in this paper are straightforward and are currently included in development branches of the Venture system.

Although some real-world applications have been successfully implemented using unoptimized prototypes, the performance surface of Venture is largely uncharacterized. This is intentional. Our understanding of probabilistic programming is more limited than our understanding of functional programming when Miranda was introduced \citep{turner1985miranda}. Standard graphical models from machine learning correspond to short nested loops that usually can be fully unrolled before execution, and therefore only exercise small subsets of the capabilities of typical expressive languages. It thus seems premature to focus on optimizing runtime performance until better foundations have been established.

The theoretical principles needed to assess the tradeoffs between inference accuracy, asymptotic scaling, memory consumption, and absolute runtime are currently unclear. Point-wise comparisons of runtime are easy to make, but just as easily mislead. Even small changes in problem formulation, dataset size or accuracy can lead to large changes in runtime. The asymptotic scaling of runtime for many probabilistic programming systems is unknown, and for some problems the relevant scale parameters are hard to identify. Systematic comparisons of inference strategies that control for implementation constants and scale thresholds have yet to be performed. Although Venture's asymptotic scaling competitive with hand-optimized samplers in theory, a rigorous empirical assessment is needed. Mathematically rigorous cost models are in our view a precondition for comparative benchmarking and thus could be an important research direction. Additionally, once the relationships between the asymptotic scaling of forward simulation and the asymptotic scaling of typical inference strategies has been characterized, it will be possible to search for effective equivalences between custom inference strategies and model transformations.

Performance engineering research for Venture can build on standard techniques for higher-order languages as well as exploitation of the structure of probabilistic models and inference strategies. Immediate opportunities include runtime specialization via type hints and/or scaffold contents, compilation of the transition operators arising from specific inference instructions, and optimizations to the SPI to minimize copying of data. PETs could be used to support a memory manager that exploits locality of reference along PETs --- and conditional independence, more generally --- to improve cache efficiency. For example, one could pre-calculate the envelope of a specific inference instruction, and pre-fetch the entire PET fragment. Similar optimizations will be beneficial for high-performance multicore implementations. There are also opportunities for asymptotic savings if deterministic sub-segments of PETs can be compressed; in principle, the memory requirements of inference should scale with the amount of randomness consumed by the model, not by its execution time. Preliminary explorations of some of these ideas have yielded promising results \citep{perovmulticore, wumeng} but much work remains to be done. In practice, we believe the judicious migration of performance sensitive regions of a probabilistic program into hand-optimized inference code will be as important as runtime and compiler optimizations.

\subsection{Debugging and profiling probabilistic programs}

Venture makes interactive modeling and inference possible by providing an expressive modeling language with scalable programmable inference. By removing the need to do integrated co-design of probabilistic models, inference algorithms and their implementations, Venture removes one of the main bottlenecks in building probabilistic computing systems. This is far from the only bottleneck, however. Design, validation, debugging and optimization of probabilistic programs --- both model programs and inference programs --- still requires expertise.

The promise of probabilistic programming is that instead of having to learn multiple non-overlapping fields, future modelers will be able to learn a coherent body of probabilistic programming principles. They should be able to use standard probabilistic programming tools and workflows to navigate the design, validation, debugging and optimization process. This could evolve analogously to how programmers learn to design, test, debug and optimize traditional programs using a mix of intuitive modeling, mathematical analysis and experimentation supported by sophisticated tools.

Venture implements a stochastic semantics for inference that is designed to cohere with Bayesian reasoning via exact sampling in the limit of infinite computation. This facilitates the development of mathematically rigorous debugging strategies. First, the precise gap between exact and approximate sampling can sometimes be measured on small-scale examples; this provides a crucial control for larger-scale debugging. Second, as the amount of computation devoted to inference increases, the quality of the approximation can only increase. As a result, at potentially substantial computational cost, it is always possible to reduce the impact of approximate inference (as distinct from modeling issues or data issues). 

There is even a path forward at scales where the true distribution is effectively impossible to sample from or even roughly characterize. For example, the Bernstein-von Mises family of theorems gives general conditions under which posteriors concentrate onto hypotheses that satisfy various invariants. These can be exploited by debugging tools: a tool could increase the amount of synthetic data and inference used to test a probabilistic program until the posterior concentrates. At this point the probabilistic program begins to resemble a problem with a ``single right answer'', significantly simplifying debugging and testing as compared to regimes where the true distribution can be arbitrarily spread out or multi-modal. Many more tools and techniques will no doubt be needed to tease apart the impact of model mismatch, inadequate or inappropriate data and insufficient inference.

Traditional profilers may also have natural analogs in probabilistic programming. A profiler for a Venture program might track ``hot spots'' --- scaffolds (and their source code locations) where a disproportionate fraction of runtime is spent --- as well as ``cold spots'': scaffolds where transitions are rejected with unusual frequency. Such cold spots might serve as useful warning signs of limited, purely local convergence of inference. As Venture maintains an execution trace, it should be possible to link an execution trace location with the source code location responsible for it, and potentially help probabilistic programmers identify and modify unnecessarily deterministic or constrained model fragments.
Both together could help Venture programmers identify the portions of their program that ought to be moved to optimized foreign inference code.

\subsection{Inference programming}

The Venture inference programming language has many known limitations that would be straightforward to address. For example, it currently lacks an iteration construct, a notion of procedural abstraction (so inference strategies can be parameterized and reused), the ability to dynamically evaluate modeling expressions to stochastically generate scope and block names, or an analogue of \verb|eval| that gives inference programs the ability to programmatically execute new directives in the hosting Venture virtual machine. Additionally, its current set of inference primitives cannot themselves be implemented via compound statements in the language; to support this, an analogue of stochastic regeneration could potentially be exposed. Future work will address these issues by reintroducing the full expressiveness of Lisp into the inference expression language and determining what extensions are needed to capture the interactions between modeling and approximate inference.

More substantial inference programming extensions may also be fruitful. For example, it may be fruitful to incorporate support for {\em inference procedures}: reusable proposal schemes for scaffolds that match pre-specified patterns, where the proposal mechanism is implemented by another Venture program. The scaffold's contents constitute the formal arguments. Each random choice in the border could be mapped to an \verb|OBSERVE|, and the new values for resampling nodes could each come from a \verb|PREDICT|. If inference in the proposal program is assumed to have converged, it should be possible to construct a valid transition operator. This mechanism would enable Venture programmers to use probabilistic modeling and approximate inference to design inference strategies, turning every modeling idiom in Venture into a potential inference tool. A similar mechanism could facilitate the use of custom Venture programs as the skeleton for variational approximations. Finally, the inference programming language needs to be extended to relax the constraints of soundness. Instead of restricting programmers to transition operators that are guaranteed by construction to leave the conditioned distribution on traces invariant, the language should permit experts to introduce arbitrary transition operators on traces.

It will be interesting to develop probabilistic programs that automate aspects of inference while going beyond traditional formulations of interpretation and compilation. For example, it is theoretically possible to develop probabilistic programs that work as inference optimizers. Such programs would take a Venture program as input, including only placeholder \verb|INFER| instructions, and produce as output a new Venture program with \verb|INFER| instructions that are likely to perform better. These programs could also transform the modeling instructions and the instructions introducing the data to improve computational or inferential performance. Depending on architecture, these programs could be viewed as inference planners, compilers, or an integrated combination of the two. Probabilistic models, approximate inference and Bayesian learning could all be deployed to augment engineering of the planning algorithm, compiler transformations or even the objective function used to summarize inference performance. Machine learning schemes for algorithm selection \citep{xu2008satzilla} and reinforcement-learning-based meta-computation \citep{lagoudakis2001learning} can be viewed as natural special cases. Venture also makes it possible to integrate a ``control'' probabilistic program into an interpreter running another probabilistic program. The control program could intercept and modify the inference instructions in the running Venture program based on approximate, model-based inferences from live performance data, and perhaps influence the machine code generation process invoked by a just-in-time compiler.

Inference programming also supports the integration of ``approximate'' compilers based on probabilistic modeling and approximately Bayesian learning. Consider a probabilistic program with a specific set of \verb|OBSERVE|s and \verb|PREDICT|s. The expressions in the \verb|OBSERVE|s define a space of possible inputs, each corresponding to a set of literal values, one per observation. The expressions in the \verb|PREDICT|s define a space of possible outputs, each corresponding to an assignment of a sampled value to each \verb|PREDICT|. A compiler for Venture might generate an equivalent program that is restricted to precisely the given pattern of inputs and outputs, optimizing the implementation of any \verb|INFER| instructions in the program given this restriction. It is also possible to use modeling and inference to approximately emulate the program, by writing a probabilistic program that models $p(\{\verb|PREDICT|s\} | \{\verb|OBSERVE|s\})$, with parameters and/or structure estimated from data that is generated from the original probabilistic program. This can be viewed as approximate compilation implemented via inference, where the target language is given by the hypothesis space of the model in the emulator program. Both the creation of this kind of emulator and its use as a proposal would be natural to integrate as additional inference instructions.

Probabilistic programming systems should ultimately support the complete spectrum from black-box, truly automatic inference to highly customized inference strategies. So far Venture has focused on the extremes. Other points on the spectrum will require extensions to Venture's instruction language to encode specifications for exact and approximate inference within particular runtime and/or accuracy parameters. Developing a suitable specification language and cost model to enable precise control over the scope of automatic inference is an important challenge for future research.

\todo[color=yellow]{[small] here and in the discussion, discuss the question of a basis set for inference} 

\subsection{Conclusion}

The similarities between non-probabilistic, Turing-complete,
higher-order programming languages are striking given the enormous
design space of such languages. Lisp, Java and Python support many of
the same programming idioms and can be used to simulate one another
without changing the asymptotic order of growth of program
runtime. Each of these languages also presents similar foreign
interfaces for interoperation with external software. They have all
been used for data analysis and for system building, deployed in fields ranging from robotics to statistics. They have also all been used for machine intelligence research grounded in probabilistic modeling and inference. One measure of their flexibility and interchangeability is provided by their use in probabilistic programming research. Each of these languages has been used to implement expressive probabilistic programming languages. For example, multiple versions of Church and Venture have been written in Python, Lisp and Java, and BLOG has been implemented in both Java and Python.

We do not know if it will be possible to attain the flexibility, extensibility and efficiency of Lisp, Java or Python in any single probabilistic programming language, especially due to the complexity of inference. Thus far, most probabilistic programming languages have opted for a narrower scope. Most are not Turing-complete or higher-order, only support a subset of standard approximate inference strategies, and have no notion of inference programming. Translating between sufficiently expressive probabilistic programming languages without distorting the asymptotic scaling of forward simulation will be difficult, especially given the lack of a standard cost model. Providing faithful translations that do not distort the asymptotic scaling of inference --- in runtime, in accuracy, or both --- will be harder still.

Despite these challenges, it may be possible to develop probabilistic languages that can be used to specify and solve probabilistic modeling and approximate inference problems from many fields. Ideally such languages would be able to cover the modeling idioms and inference strategies from fields such as robotics, statistics, and machine learning, while also meeting key representational needs in cognitive science and artificial intelligence. We hope Venture, and the principles behind its design and implementation, represent a significant step towards the development of a probabilistic programming platform that is both computationally universal and suitable in practice for general-purpose use.

{\footnotesize

\bibliography{fwood_uber,dan,references}
\bibliographystyle{unsrt}

}

\end{document}